\documentclass[lettersize,journal,onecolumn,draftclsnofoot]{IEEEtran}
\usepackage{amsmath,amsfonts}
\usepackage{algorithmic}
\usepackage{algorithm}
\usepackage{array}
\usepackage[caption=false,font=normalsize,labelfont=sf,textfont=sf]{subfig}
\usepackage{textcomp}
\usepackage{stfloats}
\usepackage{url}
\usepackage{verbatim}
\usepackage{graphicx}
\usepackage{cite}
\usepackage{amsmath,amsfonts,amsthm,amssymb,bbm}
\usepackage{hyperref}

\usepackage{balance}
\usepackage{mathrsfs}
\usepackage{xcolor}
\usepackage{tikz}
\usepackage{pgfplots}
\usepackage{colortbl}
\usetikzlibrary{positioning,quotes,angles,patterns,intersections,pgfplots.fillbetween,calc}
\usepackage{tkz-euclide}
\usepackage{mathtools}
\usepackage{url}
\usetikzlibrary{decorations.pathmorphing}
\usepackage{enumitem}
\usepackage{stfloats}
\usepackage{multirow}
\usepackage[outdir=Figures/]{epstopdf}
\pgfplotsset{compat=1.17} 
\usepackage[normalem]{ulem}
\DeclareMathOperator*{\argmax}{arg\,max}

\theoremstyle{plain} 
\newtheorem{theorem}{Theorem}
\newtheorem{corollary}{Corollary}
\newtheorem{definition}{Definition}
\newtheorem{lemma}{Lemma}
\newtheorem{proposition}{Proposition}
\theoremstyle{definition} \newtheorem{remark}{Remark}
\theoremstyle{definition} 
\usetikzlibrary{calc,shapes.geometric}

\begin{document}


\title{Addressing GAN Training Instabilities via Tunable Classification Losses}

\author{Monica Welfert*, Gowtham R. Kurri*, Kyle Otstot, Lalitha Sankar
\thanks{* Equal contribution}
\thanks{This work is supported in part by NSF grants CIF-1901243, CIF-1815361, CIF-2007688, DMS-2134256, and SCH-2205080.}
\thanks{M. Welfert, K. Otstot and L. Sankar are with the School of Electrical, Computer, and Energy Engineering, Arizona State University, Tempe, AZ 85281 USA (email: \{mwelfert, lsankar, kotstot\}@asu.edu). 

Gowtham R. Kurri was with the School of Electrical, Computer and Energy Engineering at Arizona State University at the time the work was done. He is now with the Signal Processing and Communications Research Centre at International Institute of Information Technology, Hyderabad, India (e-mail: {gowtham.kurri@iiit.ac.in}).}
\thanks{Manuscript received ; revised .}}

\markboth{Journal of \LaTeX\ Class Files,~Vol.~, No.~, }%
{Shell \MakeLowercase{\textit{et al.}}: A Sample Article Using IEEEtran.cls for IEEE Journals}

\IEEEpubid{0000--0000/00\$00.00~\copyright~2023 IEEE}

\maketitle

\begin{abstract}
Generative adversarial networks (GANs), modeled as a zero-sum game between a generator (G) and a discriminator (D), allow generating synthetic data with formal guarantees. 
Noting that D is a classifier, we begin by reformulating the GAN value function using 
class probability estimation (CPE) losses. 
{We prove a two-way correspondence between CPE loss GANs and $f$-GANs which minimize $f$-divergences.} 
We also show that all symmetric $f$-divergences are equivalent in convergence. 
In the finite sample and model capacity setting, we define and obtain bounds on estimation and generalization errors. We specialize these results to $\alpha$-GANs, defined using $\alpha$-loss, a tunable CPE loss family parametrized by $\alpha\in(0,\infty]$. 
We next introduce a class of dual-objective GANs to address training instabilities of GANs by modeling each player's objective using $\alpha$-loss to obtain $(\alpha_D,\alpha_G)$-GANs. We show that the resulting non-zero sum game simplifies to minimizing an $f$-divergence under appropriate conditions on $(\alpha_D,\alpha_G)$. Generalizing this dual-objective formulation using CPE losses, we define and obtain upper bounds on an appropriately defined estimation error. Finally, we highlight the value of tuning $(\alpha_D,\alpha_G)$ in alleviating training instabilities for the synthetic 2D Gaussian mixture ring as well as the large publicly available Celeb-A and LSUN Classroom image datasets.
\end{abstract}


\section{Introduction}

\IEEEPARstart{G}{enerative} adversarial networks (GANs) have become a crucial data-driven tool for generating synthetic data. GANs are generative models trained to produce samples from an unknown (real) distribution using a finite number of training data samples. They consist of two modules, a generator G and a discriminator D, parameterized by vectors $\theta\in\Theta\subset \mathbb{R}^{n_g}$ and $\omega\in\Omega\subset\mathbb{R}^{n_d}$, respectively, which play an adversarial game with each other. The generator $G_\theta$ maps noise $Z\sim P_Z$ to a data sample in $\mathcal{X}$ via the mapping $z\mapsto G_\theta(z)$ and aims to mimic data from the real distribution $P_{r}$. The discriminator $D_\omega$ takes as input $x\in\mathcal{X}$ and classifies it as real or generated by computing a score $D_\omega(x)\in[0,1]$ which reflects the probability that $x$ comes from $P_r$ (real) as opposed to $P_{G_\theta}$ (synthetic). For a chosen value function $V(\theta,\omega)$, the adversarial game between G and D can be formulated as a zero-sum min-max problem given by 
\thinmuskip=1mu
\begin{align}\label{eqn:GANgeneral}
    \inf_{\theta\in\Theta}\sup_{\omega\in\Omega} \,V(\theta,\omega). 
\end{align}
Goodfellow \emph{et al.}~\cite{Goodfellow14} introduce the vanilla GAN for which 
\thickmuskip=2mu
\medmuskip=0mu
\begin{align}
V_\text{VG}(\theta,\omega) \nonumber =\mathbb{E}_{X\sim P_r}[\log{D_\omega(X)}]+\mathbb{E}_{X\sim P_{G_\theta}}[\log{(1-D_\omega(X))}].
\end{align}
For this $V_\text{VG}$, they show that when the discriminator class $\{D_\omega\}_{\omega\in\Omega}$ is rich enough, \eqref{eqn:GANgeneral} simplifies to minimizing
the Jensen-Shannon divergence~\cite{Lin91} between $P_r$ and $P_{G_\theta}$. 

Various other GANs have been studied in the literature using different value functions, including
$f$-divergence based GANs called $f$-GANs~\cite{NowozinCT16}, IPM based GANs~\cite{ArjovskyCB17,sriperumbudur2012empirical,liang2018well}, etc. 
Observing that the discriminator is a classifier, recently, in~\cite{KurriSS21}, we show that the value function in \eqref{eqn:GANgeneral} can be written using a class probability estimation (CPE) loss $\ell(y,\hat{y})$ whose inputs are the true label $y\in\{0,1\}$ and predictor $\hat{y}\in[0,1]$ (soft prediction of $y$) as
\begin{align*}
   V(\theta,\omega) =\mathbb{E}_{X\sim P_r}[-\ell(1,D_\omega(X))]+\mathbb{E}_{X\sim P_{G_\theta}}[-\ell(0,D_\omega(X))].
\end{align*}
We further introduce $\alpha$-GAN in \cite{KurriSS21}
using the tunable CPE loss $\alpha$-loss~\cite{sypherd2019tunable,sypherd2022journal}, defined for $\alpha \in(0,\infty]$ as
\begin{align} \label{eq:cpealphaloss}
\ell_\alpha(y,\hat{y})\coloneqq\frac{\alpha}{\alpha-1}\left(1-y\hat{y}^{\frac{\alpha-1}{\alpha}}-(1-y)(1-\hat{y})^{\frac{\alpha-1}{\alpha}}\right),
\end{align}
and show that this $\alpha$-GAN formulation recovers various $f$-divergence based GANs including the Hellinger GAN~\cite{NowozinCT16} ($\alpha=1/2$), the vanilla GAN~\cite{Goodfellow14} ($\alpha=1$), and the Total Variation (TV) GAN~\cite{NowozinCT16} ($\alpha=\infty$). Further, for a large enough discriminator class, we also show that the min-max optimization for $\alpha$-GAN in \eqref{eqn:GANgeneral} simplifies to minimizing the Arimoto divergence~\cite{osterreicher1996class,LieseV06}. In \cite{kurri-2022-convergence}, we also show that the resulting Arimoto divergences are equivalent in convergence.

While each of the abovementioned GANs have distinct advantages, they continue to 
suffer from one or more types of training instabilities, including vanishing/exploding gradients, mode collapse, and sensitivity to hyperparameter tuning.  
In \cite{Goodfellow14}, Goodfellow \emph{et al.} note that the generator's objective in the vanilla GAN can \emph{saturate} early in training (due to the use of the sigmoid activation) when D can easily distinguish between the real and synthetic samples, i.e., when the output of D is near zero for all synthetic samples, leading to vanishing gradients. Further, a confident D induces a steep gradient at samples close to the real data, thereby preventing G from learning such samples due to exploding gradients.  To alleviate these, \cite{Goodfellow14} proposes a \emph{non-saturating} (NS) generator objective:
\begin{align}
    V_\text{VG}^\text{NS}(\theta,\omega)=\mathbb{E}_{X\sim P_{G_\theta}}[-\log{D_\omega(X)}].
\end{align}

This NS version of the vanilla GAN 
may be viewed as involving different objective functions for the two players (in fact, with two versions of the $\alpha=1$ CPE loss, i.e., log-loss, for D and G). However, it continues to suffer from mode collapse \cite{arjovsky2017towards,wiatrak2019stabilizing} due to failure to converge and sensitivity to hyperparameter initialization (e.g. learning rate) because of large gradients. While other dual-objective GANs have also been proposed 
(e.g., Least Squares GAN (LSGAN)~\cite{Mao_2017_LSGAN}, R\'{e}nyiGAN~\cite{bhatia2021least}, NS $f$-GAN \cite{NowozinCT16}, hybrid $f$-GAN \cite{poole2016improved}), few have successfully addressed the landscape of training instabilities.


Recent results have shown that $\alpha$-loss demonstrates desirable gradient behaviors for different $\alpha$ values \cite{sypherd2022journal}. These results also assure learning robust classifiers that can reduce the confidence of D (a classifier); this, in turn, can allow G to learn without gradient issues. More broadly, by using different loss-based value functions for D and G, we can fully exploit this varying gradient behavior.  
To this end, in \cite{welfert2023alpha_d} we 
introduce a different $\alpha$-loss objective\footnote{Throughout the paper, we use the terms objective and value function interchangeably.} for each player 
and propose a tunable dual-objective $(\alpha_D,\alpha_G)$-GAN, where the value functions of D and G are written in terms of $\alpha$-loss with parameters $\alpha_D\in(0,\infty]$ and $\alpha_G\in(0,\infty]$, respectively. 

This paper ties together and significantly enhances our prior results investigating single-objective CPE loss-based GANs including $\alpha$-GAN \cite{KurriSS21,kurri-2022-convergence} and dual-objective GANs including $(\alpha_D,\alpha_G)$-GANs \cite{welfert2023alpha_d}. We list below all our contributions (while highlighting novelty relative to \cite{KurriSS21,kurri-2022-convergence,welfert2023alpha_d}) for both single- and dual-objective GANs. 


\subsection{Our Contributions}

\noindent \textbf{Single-objective GANs:}
\begin{itemize}[leftmargin=*]
\item We review CPE loss GANs and include a two-way correspondence between CPE loss GANs and $f$-divergences (Theorem~\ref{thm:correspondence}) previously published in \cite{kurri-2022-convergence}. We note that we include a more comprehensive proof of this result here. We review $\alpha$-GANs, originally proposed in \cite{KurriSS21}, and present the optimal strategies for G and D, provided they have sufficiently large capacity and infinite samples (Theorem~\ref{thm:alpha-GAN}). We also include a result from \cite{KurriSS21} showing that $\alpha$-GAN interpolates between various $f$-GANs including vanilla GAN ($\alpha=1$), Hellinger GAN~\cite{NowozinCT16} ($\alpha=1/2$), and Total Variation GAN~\cite{NowozinCT16} ($\alpha=\infty$) by tuning $\alpha$ (Theorem~\ref{thm:fgans}).
\item A novel contribution of this work is proving an equivalence between a CPE loss GAN and a corresponding $f$-GAN (Theorem \ref{theorem:equivalence-fGAN-CPEGAN}). We specialize this for 
$\alpha$-GANs and $f_\alpha$-GANs to show that one can go between the two formulations using a bijective activation function  (Theorem \ref{thm:obj-equiv-gen} and Corollary \ref{corollary:equivalence-falphaGAN-alphaGAN}).
\item We study \emph{convergence} properties of CPE loss GANs in the presence of sufficiently large number of samples and discriminator capacity. We show that all symmetric $f$-divergences are \emph{equivalent} in convergence (Theorem~\ref{thm:equivalenceinconvergence}) generalizing an equivalence proven in our prior work \cite{kurri-2022-convergence} for Arimoto divergences. We remark that the proof techniques used here give rise to a conceptually simpler proof of equivalence between Jensen-Shannon divergence and total variation distance proved earlier by Arjovsky \emph{et al.}~\cite[Theorem~2(1)]{ArjovskyCB17}.
\item In the setting of finite training samples and limited capacity for the generator and discriminator models, we extend the definition of generalization, first introduced by Arora \emph{et al.} \cite{AroraGLMZ17}, to CPE loss GANs. We do so by introducing a refined neural net divergence and prove that it indeed generalizes with 
increasing number of training samples (Theorem \ref{thm:generalizationofarora}). 
\item To conclude our results on single-objective GANs, we review the definition of estimation error for CPE loss GANs introduced in \cite{kurri-2022-convergence}, present an upper bound on the error originally proven in \cite{kurri-2022-convergence} (Theorem \ref{thm:estimationerror-upperbound}), and a matching lower bound under additional assumptions for $\alpha$-GANs previously proven in \cite{welfert2023alpha_d} (Theorem \ref{thm:est-error-lower-bound-alpha-infinity}).
\end{itemize}

\noindent \textbf{Dual-objective GANs:}

\begin{itemize}[leftmargin=*]
\item {We begin by reviewing $(\alpha_D,\alpha_G)$-GANs, originally proposed in \cite{welfert2023alpha_d}, and the corresponding optimal strategies for D and G for appropriate $(\alpha_D,\alpha_G)$ values (Theorem \ref{thm:alpha_D,alpha_G-GAN-saturating}). We also review the non-saturating version of $(\alpha_D,\alpha_G)$-GANs, also proposed in \cite{welfert2023alpha_d}, and present its Nash equilibrium strategies for D and G (Theorem \ref{thm:alpha_D,alpha_G-GAN-nonsaturating}).}
\item A novel contribution of this work is a gradient analysis highlighting the effect of tuning $(\alpha_D,\alpha_G)$ on the magnitude of the gradient of the generator's loss for both the saturating and non-saturating versions of the $(\alpha_D,\alpha_G)$-GAN formulation (Theorem \ref{thm:sat-gradient}).
\item {We introduce a dual-objective CPE loss GAN formulation generalizing our dual-objective $(\alpha_D,\alpha_G)$-GAN formulation in \cite{welfert2023alpha_d}. For this non-zero sum game, we present the optimal strategies for D and G and prove that for the optimal $D_{\omega^*}$, G minimizes an $f$-divergence under certain conditions (Proposition \ref{prop:dual-objective-CPE-loss-GAN-strategies}).}
\item 
We generalize the definition of estimation error we introduced in \cite{welfert2023alpha_d} for $(\alpha_D,\alpha_G)$-GANs to dual-objective CPE loss GANs. We present an upper bound on the error (Theorem \ref{thm:estimationerror-upperbound-double-objective}), and show that this result subsumes that for $(\alpha_D,\alpha_G)$-GANs in \cite{welfert2023alpha_d}. 
\item Focusing on $(\alpha_D,\alpha_G)$-GANs, we demonstrate empirically that tuning $\alpha_D$ and $\alpha_G$ significantly reduces vanishing and exploding gradients and alleviates mode collapse on a synthetic 2D-ring dataset (originally published in \cite{welfert2023alpha_d}). For the high-dimensional Celeb-A and LSUN Classroom datasets, we show that our tunable approach is more robust in terms of the Fréchet Inception Distance (FID) to the choice of GAN hyperparameters, including number of training epochs and learning rate, relative to both vanilla GAN and LSGAN.
\item Finally, throughout the paper, we illustrate the effect of tuning $(\alpha_D,\alpha_G)$ on training instabilities including vanishing and exploding gradients, as well as model oscillation and mode collapse.
\end{itemize}

\subsection{Related Work}

GANs face several challenges that threaten their training stability \cite{Goodfellow14,salimans2016improved,zhao2018bias,huszar2015not}, such as
vanishing/exploding gradients, mode collapse, sensitivity to hyperparameter initialization, and model oscillation, which occurs when the generated data oscillates around modes in real data due to large gradients.
Many GAN variants have been proposed to stabilize training by changing the objective optimized \cite{Goodfellow14,Mao_2017_LSGAN,NowozinCT16,li2017mmd,berthelot2017began,ArjovskyCB17,poole2016improved,GulrajaniAADC17,mroueh2017fisher,mroueh2017mcgan,bhatia2021least} or the architecture design \cite{radford2015,donahue2016adversarial,karras2017progressive,karras2019style}. Since we focus on tuning the objective, we restrict discussions and comparisons to similar approaches. 
Approaches modifying the objective can be categorized as single-objective or dual-objective variants. 
For the single objective setting, arguing that vanishing gradients are due to the sensitivity of $f$-divergences to mismatch in distribution supports, Arjovsky \emph{et al.}~\cite{ArjovskyCB17} proposed Wasserstein GAN (WGAN) using a ``weaker" Euclidean distance between distributions. However, this formulation requires a Lipschitz constraint on D, which in practice is achieved either via clipping model weights  or using a computationally expensive gradient penalty method \cite{GulrajaniAADC17}.  More generally, a broader class of GANs based on integral probability metric (IPM) distances have been proposed, including MMD GANs \cite{dziugaite2015training,li2015generative}, Sobolev GANs \cite{mroueh2017sobolev}, (surveyed in \cite{liang2018well}), and total variation GANs \cite{PACGANLin}. Our work focuses on classifier based GANs, and does not require clipping or penalty methods, thus limiting meaningful comparisons with IPM-based GANs.
Finally, for single-{objective} GANs, many theoretical approaches to GANs assume that a particular divergence is minimized and study the role of regularization methods \cite{reshetova2021entropic,mesa2019distributed}. Our work goes beyond these approaches by explicitly analyzing the value function optimizations of both D and G, thereby enabling understanding and addressing training instabilities.

Noting the benefit of using different objectives for the D and G, various dual-objective GANs, beyond the NS vanilla GAN, have been proposed. Mao \emph{et al.} \cite{Mao_2017_LSGAN} proposed Least Squares GAN (LSGAN) where the objectives for D and G use different linear combinations of squared loss-based measures.
LSGANs can be viewed as state of the art in highlighting the effect of objective in GAN performance; therefore, in addition to vanilla GAN, we contrast our results to this work, as it allows for a fair comparison when choosing the same hyperparameters including model architecture, learning rate, initialization, optimization methodology, etc. for both approaches. 
Dual objective variants including R\'{e}nyiGAN~\cite{bhatia2021least}, least $k$th-order GANs ~\cite{bhatia2021least}, NS $f$-GAN \cite{NowozinCT16}, and hybrid $f$-GAN \cite{poole2016improved} have also been proposed. Recently, \cite{veiner2023unifying} attempts to unify a variety of divergence-based GANs (including special cases of both our $(\alpha_D,\alpha_G)$-GANs and LSGANs) via $\mathcal{L}_\alpha$-GANs. However, our work is distinct in highlighting the role of GAN objectives in reducing training instabilities.
Finally, it is worth mentioning that dual objectives have been shown to be essential 
in the context of learning models robust to adversarial attacks \cite{robey2023adversarial}. 


Generalization for single-objective GANs was first introduced by Arora \emph{et al.} \cite{AroraGLMZ17}. Our work is the first to extend the definition of generalization to incorporate CPE losses. There is a growing interest in studying and constructing bounds on the estimation error in training GANs \cite{zhang2017discrimination,liang2018well,JiZL21}. Estimation error evaluates the performance of a limited fixed capacity generator (e.g., a class of neural networks) learned with finite samples relative to the best generator. The results in \cite{zhang2017discrimination,liang2018well,JiZL21} study estimation error using a specific formulation that does not take into account the loss used and also define estimation error only in the single-objective setting. In this work, we study the impact of the loss used as well as the dual-objective formulation on the estimation error guarantees. To the best of our knowledge, this is the first result of this kind for dual-objective GANs. 

The remainder of the paper is organized as follows. We review various GANs in the literature, classification loss functions, particularly $\alpha$-loss, and GAN training instabilities in Section~\ref{section:preliminaries}. In Section~\ref{sec:loss-function-perspective-single-objective}, we present and analyze the loss function perspective of GANs and introduce tunable $\alpha$-GANs. In Section~\ref{sec:dual-objective}, we propose and analyze dual-objective $(\alpha_D,\alpha_G)$-GANs and introduce a dual-objective CPE-loss GAN formulation. Finally, in Section~\ref{sec:experimental-results}, we highlight the value of tuning $(\alpha_D,\alpha_G)$ for $(\alpha_D,\alpha_G)$-GANs on several datasets. {All proofs and additional experimental results can be found in the accompanying supplementary material (Appendices A-Q).}

\section{Preliminaries: Overview of GANs and Loss Functions for Classification}\label{section:preliminaries}
\subsection{Background on GANs}
We begin by presenting an overview of GANs in the literature. Let $P_r$ be a probability distribution over $\mathcal{X}\subset\mathbb{R}^d$, which the generator wants to learn \emph{implicitly} by producing samples by playing a competitive game with a discriminator in an adversarial manner. 
We parameterize the generator G and the discriminator D by vectors $\theta\in\Theta\subset \mathbb{R}^{n_g}$ and $\omega\in\Omega\subset\mathbb{R}^{n_d}$, respectively, and write $G_\theta$ and $D_\omega$ ($\theta$ and $\omega$ are typically the weights of neural network models for the generator and the discriminator, respectively). The generator $G_\theta$ takes as input a $d^\prime(\ll d)$-dimensional latent noise $Z\sim P_Z$ and maps it to a data point in $\mathcal{X}$ via the mapping $z\mapsto G_\theta(z)$. For an input $x\in\mathcal{X}$, the discriminator outputs $D_\omega(x)\in[0,1]$, the probability that $x$ comes from $P_r$ (real) as opposed to $P_{G_\theta}$ (synthetic). The generator and the discriminator play a two-player min-max game with a value function $V(\theta,\omega)$, resulting in a saddle-point optimization problem given by
\begin{align}\label{eqn:GANgeneral-background}
    \inf_{\theta\in\Theta}\sup_{\omega\in\Omega} V(\theta,\omega). 
\end{align}
Goodfellow \emph{et al.}~\cite{Goodfellow14} introduced the vanilla GAN using
\begin{align}
    V_\text{VG}(\theta,\omega)
    &=\mathbb{E}_{X\sim P_r}[\log{D_\omega(X)}]+\mathbb{E}_{Z\sim P_{Z}}[\log{(1-D_\omega(G_\theta(Z)))}]\nonumber\\
    &=\mathbb{E}_{X\sim P_r}[\log{D_\omega(X)}]+\mathbb{E}_{X\sim P_{G_\theta}}[\log{(1-D_\omega(X))}],\label{eq:Goodfellowobj}
\end{align}
for which they showed that when the discriminator class $\{D_\omega\}$, parametrized by $\omega$, is rich enough, \eqref{eqn:GANgeneral-background} simplifies to finding $\inf_{\theta\in\Theta} 2D_{\text{JS}}(P_r||P_{G_\theta})-\log{4}$,
where $D_{\text{JS}}(P_r||P_{G_\theta})$ is the Jensen-Shannon divergence~\cite{Lin91} between $P_r$ and $P_{G_\theta}$. This simplification is achieved, for any $G_\theta$, by choosing the optimal discriminator
    \begin{align}
    D_{\omega^*}(x)=\frac{p_r(x)}{p_r(x)+p_{G_\theta}(x)}, \quad x \in \mathcal{X},
    \label{eq:opt-disc-vanilla}
    \end{align}
where $p_r$ and $p_{G_\theta}$ are the corresponding densities of the distributions $P_r$ and $P_{G_\theta}$, respectively, with respect to a base measure $dx$ (e.g., Lebesgue measure).

Generalizing this by leveraging the variational characterization of $f$-divergences~\cite{NguyenWJ10}, Nowozin \emph{et al.}~\cite{NowozinCT16} introduced $f$-GANs via the value function
\begin{align}\label{eqn:fGANobj}
    V_f(\theta,\omega)=\mathbb{E}_{X\sim P_r}[D_\omega(X)]+\mathbb{E}_{X\sim P_{G_\theta}}[-f^*(D_\omega(X))],
\end{align}
where\footnote{This is a slight abuse of notation in that $D_\omega$ is not a probability here. However, we chose this for consistency in  notation of discriminator across various GANs. 
} $D_\omega:\mathcal{X}\rightarrow \mathbb{R}$ and $f^*(t)\coloneqq \sup_u\left\{ut-f(u)\right\}$ is the Fenchel conjugate of a convex lower semicontinuous function $f$ defining
an $f$-divergence
$D_f(P_r||P_{G_\theta})\coloneqq\int_\mathcal{X}p_{G_\theta}(x)f\left(\frac{p_r(x)}{p_{G_\theta}(x)}\right)dx$~\cite{measures_renyi1961,Csiszar67,Alis66}. In particular, $\sup_{\omega\in\Omega} V_f(\theta,\omega)=D_f(P_r||P_{G_\theta})$ when there exists $\omega^*\in\Omega$ such that $D_{\omega^*}(x)=f^\prime\left(\frac{p_r(x)}{p_{G_\theta}(x)}\right)$. In order to respect the domain $\text{dom}(f^*)$ of the conjugate $f^*$,
Nowozin \emph{et al.} further decomposed \eqref{eqn:fGANobj} by assuming the discriminator $D_\omega$ can be represented in the form $D_\omega(x) = g_f(Q_\omega(x))$, yielding the value function
\begin{align}\label{eqn:fGANobj-activation}
    \widetilde{V}_f(\theta,\omega)=\mathbb{E}_{X\sim P_r}[g_f(Q_\omega(x))]+\mathbb{E}_{X\sim P_{G_\theta}}[-f^*(g_f(Q_\omega(x)))],
\end{align}
where $Q_\omega:\mathcal{X}\to\mathbb R$ and $g_f:\mathbb R \to \text{dom}(f^*)$ is an output activation function specific to the $f$-divergence used.

Highlighting the problems with the continuity of various $f$-divergences (e.g., Jensen-Shannon, KL, reverse KL, total variation) over the parameter space $\Theta$~\cite{arjovsky2017towards}, Arjovsky \emph{et al.}~\cite{ArjovskyCB17} proposed Wasserstein-GAN (WGAN) using the following Earth Mover's (also called Wasserstein-1) distance:
\begin{align}
    W(P_r,P_{G_\theta})
    =\inf_{\Gamma_{X_1X_2}\in\Pi(P_r,P_{G_\theta})}\mathbb{E}_{(X_1,X_2)\sim \Gamma_{X_1X_2}}\lVert{X_1-X_2}\rVert_2,  
\end{align}
where $\Pi(P_r,P_{G_\theta})$ is the set of all joint distributions $\Gamma_{X_1X_2}$ with marginals $P_r$ and $P_{G_\theta}$. WGAN employs the Kantorovich-Rubinstein duality \cite{villani2008optimal} using the value function
\begin{align}\label{eqn:WGANobj}
   V_\text{WGAN}(\theta,\omega)=\mathbb{E}_{X\sim P_r}[D_\omega(X)]-\mathbb{E}_{X\sim P_{G_\theta}}[D_\omega(X)],
\end{align}
where the functions $D_\omega:\mathcal{X}\rightarrow \mathbb{R}$ are all 1-Lipschitz,
to simplify $\sup_{\omega\in\Omega}V_{\text{WGAN}}(\theta,\omega)$ to $W(P_r,P_{G_\theta})$ when the class $\Omega$ is rich enough. Although various GANs have been proposed in the literature, each of them exhibits their own strengths and weaknesses in terms of convergence, vanishing/exploding gradients, mode collapse, computational complexity, etc., leaving the problem of addresing GAN training instabilities unresolved~\cite{wiatrak2019stabilizing}.

\subsection{Background on Loss Functions for Classification}

The ideal loss function for classification is the Bayes loss, also known as the 0-1 loss. However, the complexity of implementing such a non-convex loss has led to much interest in seeking surrogate loss functions for classification. Several surrogate losses with desirable properties have been proposed to train classifiers; the most oft-used and popular among them is log-loss, also referred to as cross-entropy loss. However, enhancing robustness of classifier has broadened the search for better surrogate losses or families of losses; one such family is the class probability estimator (CPE) losses that operate on a soft probability or risk estimate. Recently, it has been shown that a large class of known CPE losses can be captured by a tunable loss family called $\alpha$-loss,
which includes the well-studied exponential loss ($\alpha=1/2$), log-loss ($\alpha=1$), and soft 0-1 loss, i.e., the probability of error ($\alpha=\infty$). Formally, $\alpha$-loss is defined as follows.

\begin{definition}[Sypherd \emph{et al.}~\cite{sypherd2022journal}]
\label{def:alphaloss} For a set of distributions $\mathcal{P}(\mathcal{Y})$ over $\mathcal{Y}$, $\alpha$-loss $\ell_{\alpha}:\mathcal{Y} \times \mathcal{P}(\mathcal{Y}) \rightarrow \mathbb{R}_{+}$ for $\alpha \in (0,1) \cup (1,\infty)$ is defined as
\begin{equation} \label{eq:alphaloss_prob}
\ell_{\alpha}(y,\hat{P}) = \frac{\alpha}{\alpha - 1}\left(1 - \hat{P}(y)^{\frac{\alpha-1}{\alpha}}\right).
\end{equation} 
By continuous extension, $\ell_{1}(y,\hat{P}) = -\log{\hat{P}(y)}$, $\ell_{\infty}(y,\hat{P}) = 1 - \hat{P}(y)$, and $\ell_{0}(y,\hat{P})=\infty$.
\end{definition}
Note that $\ell_{1/2}(y,\hat{P}) = \hat{P}(y)^{-1} - 1$, which is related to the exponential loss, particularly in the margin-based form~\cite{sypherd2022journal}. Also, $\alpha$-loss is convex in the probability term $\hat{P}(y)$.
Regarding the history of~\eqref{eq:alphaloss_prob}, Arimoto first studied $\alpha$-loss in finite-parameter estimation problems~\cite{arimoto1971information}, and later Liao \textit{et al.}~\cite{liao2018tunable} independently introduced and used $\alpha$-loss to model the inferential capacity of an adversary to obtain private attributes.
Most recently, Sypherd \textit{et al.}~\cite{sypherd2022journal} studied $\alpha$-loss extensively in the classification setting, which is an impetus for this work.

{\subsection{Background on GAN Training Instabilities}}
\label{subsec:gan-training-instabilities}

GANs face several challenges during training. Imbalanced performance between the generator and discriminator often coincides with the presence of exploding and vanishing gradients. When updating the generator weights during the backward pass of the network $ G_{\theta} \circ D_{\omega}$, the gradients are computed by propagating the gradient of the value function from the output layer of $D_{\omega}$ to the input layer of $G_{\theta}$, following the chain rule of derivatives. Each layer contributes to the gradient update by multiplying the incoming gradient with the local gradient of its activation function, and passing it to the preceding layer. When the gradients become large, the successive multiplication of these gradients across the layers can result in an exponential growth, known as \emph{exploding gradients}. Conversely, small gradients can lead to an exponential decay, referred to as \emph{vanishing gradients}. In both cases, networks with multiple hidden layers are particularly susceptible to unstable weight updates, causing extremely large or small values that may overflow or underflow the numerical range of computations, respectively.

In the context of the vanilla GAN, exploding gradients can occur when the generator successfully produces samples that are severely misclassified (close to 1) by the discriminator. During training, the generator is updated using the loss function $\log\left(1 - D_{\omega}(x)\right)$, which diverges to $-\infty$ as the discriminator output $D_{\omega}(x)$ approaches 1. Consequently, the gradients for the generator weights fail to converge to non-zero values, leading to the generated data potentially overshooting the real data in any direction. This is illustrated in Fig. \ref{fig:vanilla-gan-vanishing-exploding gradients}(b), relative to an initial starting point in Fig. \ref{fig:vanilla-gan-vanishing-exploding gradients}(a). In severe cases of exploding gradients, the weight update can push the generated data towards a region far from the real data. As a result, the discriminator can easily assign probabilities close to zero to the generated data and close to one to the real data. As the discriminator output approaches zero, the generator's loss function saturates, causing the gradients of the generator weights to gradually vanish. This is shown in Fig. \ref{fig:vanilla-gan-vanishing-exploding gradients}(c). The conflation of these two phenomena can prevent the generator from effectively correcting itself and improving its performance over time.

\begin{figure}[t]
\centering
\footnotesize
\setlength{\tabcolsep}{10pt}
\begin{tabular}{@{}ccc@{}}
 \includegraphics[page=1,width=0.3\linewidth]{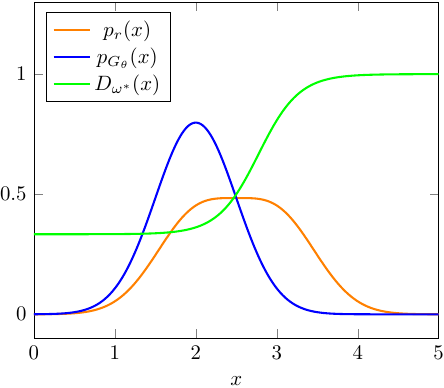}
&  \includegraphics[page=2,width=0.3\linewidth]{./Figures/vanishing_and_exploding_gradients.pdf}
& \includegraphics[page=3,width=0.3\linewidth]{./Figures/vanishing_and_exploding_gradients.pdf}\\
   (a) & (b) & (c)
\end{tabular}
\caption{A toy example of the vanilla GAN illustrating vanishing and exploding gradients, where the real distribution $P_r=0.5\mathcal{N}(2,0.5^2)+0.5\mathcal{N}(3,0.5^2)$ (orange curve) and the assumed initial generated distribution $P_{G_\theta}=\mathcal{N}(2,0.5^2)$ (blue curve). (a) A plot of the optimal discriminator output $D_{\omega^*}(x)$ in \eqref{eq:opt-disc-vanilla} (green curve). (b) A plot of the generator's saturating loss $\log(1-D_{\omega^*}(x))$ (pink curve). The rightmost generated samples receive steep gradients (exploding gradients) which causes the generated data to overshoot the real data mode toward the $D_{\omega^*}(x) \approx 1$ region. (c) For this saturating generator loss setting, following the generator's update using (b), when the discriminator updates, the generated samples now receive flat gradients (vanishing gradients), thus freezing $P_{G_\theta}$. }
\label{fig:vanilla-gan-vanishing-exploding gradients}
\end{figure}

\begin{figure}[t]
\centering
\includegraphics[width=0.4\linewidth]{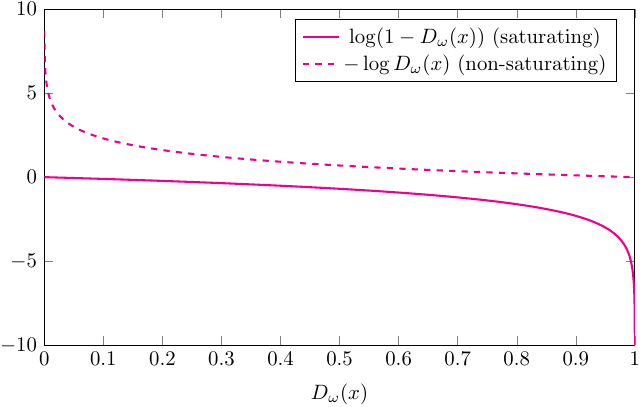}
\caption{A plot of the vanilla GAN generator's saturating loss $\log(1-D_{\omega^*}(x))$ and non-saturating loss $-\log(D_{\omega^*}(x))$.}
\label{fig:vanilla-gan-sat-and-nonsat-losses}
\end{figure}

To alleviate the issues of exploding and vanishing gradients, Goodfellow \textit{et al.} \cite{Goodfellow14} proposed a \textit{non-saturating} (NS) generator objective:
\begin{align}
    V_\text{VG}^\text{NS}(\theta,\omega)=\mathbb{E}_{X\sim P_{G_\theta}}[-\log{D_\omega(X)}].
    \label{eq:v-ns-vg}
\end{align}
The use of this non-saturating objective provides a more intuitive optimization trajectory that allows the generated distribution $P_{G_\theta}$ to converge to the real distribution $P_r$. As the discriminator output $D_\omega(x)$ for a sample $x$ approaches 1, the generator loss $-\log D_\omega(x)$ approaches zero, indicating that the generated data is closer to the real distribution Additionally, with a high-performing discriminator, the generator receives steep gradients (as opposed to vanishing gradients) during the update process; this occurs because the generator loss diverges to $+\infty$ as the discriminator output approaches zero (see Fig. \ref{fig:vanilla-gan-sat-and-nonsat-losses}). As we show in the sequel, using $\alpha$-loss based value functions allow modulating the magnitude of the gradient (and therefore, how steeply it rises), thereby improving over the vanilla GAN performance.

While the non-saturating vanilla GAN (an industry standard) incorporates two different objective functions for the generator and discriminator in order to combat vanishing and exploding gradients, it can still suffer from mode collapse and oscillations \cite{arjovsky2017towards,wiatrak2019stabilizing}. These issues often arise due to the sensitivity of the GAN to hyperparameter initialization. 
The problem of \textit{mode collapse} occurs when the generator produces samples that closely resemble only a limited subset of the real data. In such cases, the generator lacks the incentive to capture the remaining modes since the discriminator struggles to effectively differentiate between the real and generated samples. One possible explanation for this phenomenon, as depicted in Fig. \ref{fig:vanilla-gan-mode-collapse}, is that the generator and/or discriminator become trapped in a local minimum, impeding the necessary adjustments to mitigate mode collapse. In Fig. \ref{fig:vanilla-gan-mode-collapse}(a), the generated distribution approaches a single mode of the real distribution, which causes the optimal discriminator to have uniform predicted probabilities in this region; as a result, when the discriminator landscape is sufficiently flat in the mode neighborhood, the generator will get stuck and won't move out of the mode. We note that an extreme case of complete mode collapse is captured in Fig. \ref{fig:vanilla-gan-vanishing-exploding gradients}(c) where the generator is stuck in a non-mode region. As we show in the sequel, $\alpha$-loss based dual objective GANs can resolve such mode collapse issues which result from vanishing and exploding gradients.  

\begin{figure}[t]
\centering
\footnotesize
\setlength{\tabcolsep}{20pt}
\begin{tabular}{@{}cc@{}}
\includegraphics[page=2,width=0.32\linewidth]{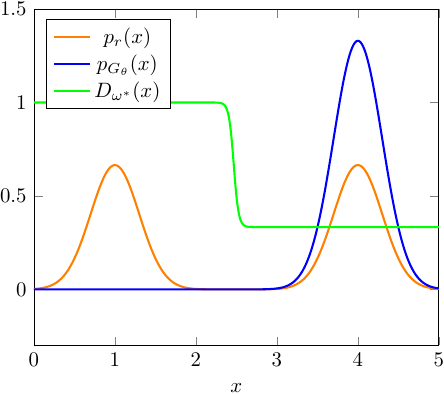}
& \includegraphics[page=3,width=0.32\linewidth]{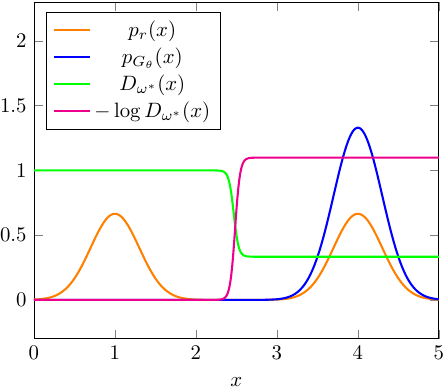}\\
   (a) & (b)
\end{tabular}
\caption{A toy example of the vanilla GAN illustrating mode collapse, where the real distribution is $P_r=0.5\mathcal{N}(1,0.3^2)+0.5\mathcal{N}(4,0.3^2)$ (orange curve) and the assumed initial generated distribution is $P_{G_\theta}=\mathcal{N}(4,0.3^2)$ (blue curve). (a) A plot of the optimal discriminator output $D_{\omega^*}(x)$ in \eqref{eq:opt-disc-vanilla}. The discriminator output is flat in the dense $p_{G_\theta}$ region. (b) A plot of the generator's non-saturating loss $-\log(D_{\omega^*}(x))$. The loss is also flat in the dense $p_{G_\theta}$ region, causing the generator to receive near-zero gradients, thus appearing to “collapse” on the real data mode. }
\label{fig:vanilla-gan-mode-collapse}
\end{figure}


Yet another potential cause of mode collapse is \emph{model oscillation}. This occurs when a generator training with the non-saturating value function $V_{\text{VG}}^{\text{NS}}$ fails to converge due to the influence of a generated outlier data sample, as illustrated in  Fig. \ref{fig:vanilla-gan-model-oscillation}. In Fig. \ref{fig:vanilla-gan-model-oscillation}(a), most of the generated data is situated at a real data mode, while some are outliers and are situated very far from the real distribution. The discriminator very confidently classifies such outlier data as fake but is less sure about the generated data that is close to the real data. As shown in Fig. \ref{fig:vanilla-gan-model-oscillation}(b), the outlier data consequently receive gradients of very large magnitude while the generated data closer to the real data receive gradients of much smaller magnitude. The generator then prioritizes directing the outlier data toward the real data over keeping the data close to the real data in place; as a result, the generator update reflects a compromise in Fig. \ref{fig:vanilla-gan-model-oscillation}(c), where the outliers are resolved at the expense of moving the other data away from the real data mode. Although the generator succeeds at bringing down the average loss by eliminating these outliers, the discriminator is now able to confidently distinguish between the distributions, leading to near-zero probabilities assigned to the generated data. In turn, as shown in Fig. \ref{fig:vanilla-gan-model-oscillation}(d), the generated samples all receive very large gradients which may result in oscillations around the real data. For this setting as well, in the sequel, we show that choosing value functions that allow modulating the role of the outliers such as via $\alpha$-loss, can be very beneficial in addressing mode oscillation. We begin our analysis by first introducing a loss function perspective of GANs.

\begin{figure}[t]
\centering
\footnotesize
\setlength{\tabcolsep}{10pt}
\begin{tabular}{@{}ccccc@{}}
\includegraphics[page=2,width=0.23\linewidth]{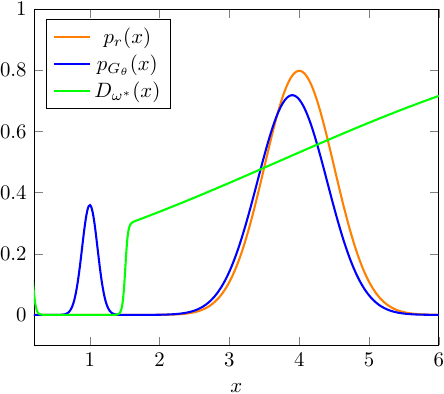}
& \includegraphics[page=3,width=0.22\linewidth]{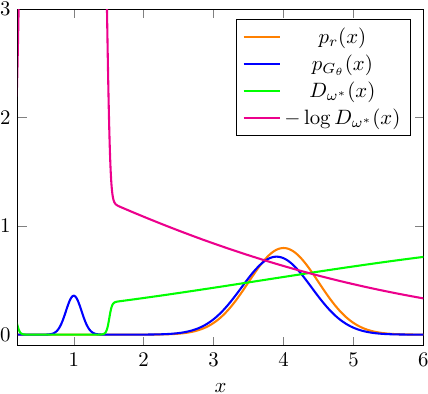}
& \includegraphics[page=4,width=0.23\linewidth]{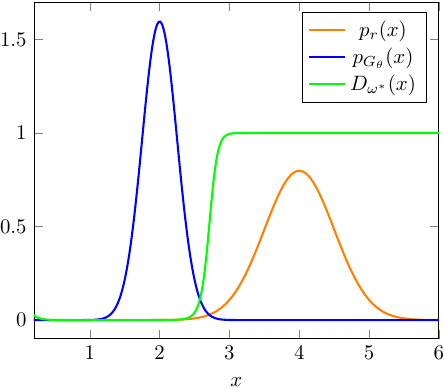}
& \includegraphics[page=5,width=0.22\linewidth]{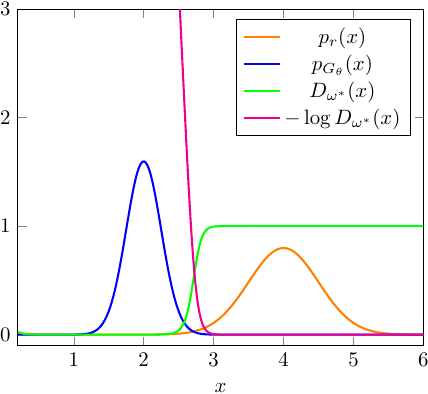}\\
   (a) & (b) & (c) & (d)
\end{tabular}
\caption{A toy example of the vanilla GAN illustrating model oscillation, where the real distribution $P_r=\mathcal{N}(4,0.5^2)$ (orange curve) and the assumed initial generated distribution $P_{G_\theta}=0.1\mathcal{N}(1,(1/9)^2) + 0.9\mathcal{N}(3.9,0.5^2)$ (blue curve). (a) A plot of the optimal discriminator output $D_{\omega^*}(x)$ in \eqref{eq:opt-disc-vanilla}. The discriminator confidently classifies “outlier” generated data and gives
cautious predictions for remaining data. (b) A plot of the generator's non-saturating loss $-\log(D_{\omega^*}(x))$. The outlier generated data receive very large gradients while generated data close to the real data receive relatively small gradients, which causes the generator to prioritize correcting the outlier data at the expense of
preserving the proximity of the generated data close to the real data. (c) A plot of the optimal discriminator output $D_{\omega^*}(x)$ in \eqref{eq:opt-disc-vanilla} after the generator and discriminator both update. The discriminator now confidently distinguishes the generated data from the real data. (d) A plot of the generator's non-saturating loss $-\log(D_{\omega^*}(x))$ after the updates in (c). The generated samples now receive very large gradients, which may lead to oscillations around the real mode.}
\label{fig:vanilla-gan-model-oscillation}
\end{figure}


\section{Loss Function Perspective on GANs}
\label{sec:loss-function-perspective-single-objective}

Noting that a GAN involves a classifier (i.e., discriminator), it is well known that the value function $V_{\text{VG}}(\theta,\omega)$ in \eqref{eq:Goodfellowobj} considered by Goodfellow \emph{et al.}~\cite{Goodfellow14} is related to binary cross-entropy loss. We first formalize this loss function perspective of GANs. In \cite{AroraGLMZ17}, Arora \emph{et al.}~observed that the $\log$ function in \eqref{eq:Goodfellowobj} can be replaced by any (monotonically increasing) concave function $\phi(x)$ (e.g., $\phi(x)=x$ for WGANs). In the context of using classification-based losses, we show that one can write $V(\theta,\omega)$ in terms of \emph{any} class probability estimation (CPE) loss $\ell(y,\hat{y})$ whose inputs are the true label $y\in\{0,1\}$ and predictor $\hat{y}\in[0,1]$ (soft prediction of $y$). For a GAN, we have $(X|y=1)\sim P_r$, $(X|y=0)\sim P_{G_\theta}$, and $\hat{y}=D_\omega(x)$. With this, we define a value function 
\begin{align}
    V(\theta,\omega)&=\mathbb{E}_{X|y=1}[-\ell(y,D_\omega(X))]+\mathbb{E}_{X|y=0}[-\ell(y,D_\omega(X))]\\
    &=\mathbb{E}_{X\sim P_r}[-\ell(1,D_\omega(X))]+\mathbb{E}_{X\sim P_{G_\theta}}[-\ell(0,D_\omega(X))]\label{eqn:lossfnps1}.
\end{align}
For binary cross-entropy loss, i.e., $\ell_{\text{CE}}(y,\hat{y})\coloneqq -y\log{\hat{y}}-(1-y)\log{(1-\hat{y})}$, notice that the expression in \eqref{eqn:lossfnps1} is equal to $V_{\text{VG}}$ in \eqref{eq:Goodfellowobj}.
For the value function in \eqref{eqn:lossfnps1}, we consider a GAN given by the min-max optimization problem:
\begin{align}\label{eqn:lossfnbasedGAN}
    \inf_{\theta\in\Theta}\sup_{\omega\in\Omega}V(\theta,\omega).
\end{align}
Let $\phi(\cdot)\coloneqq-\ell(1,\cdot)$ and $\psi(\cdot)\coloneqq-\ell(0,\cdot)$ in the sequel. The functions $\phi$ and $\psi$ are assumed to be monotonically increasing and decreasing functions, respectively, so as to retain the intuitive interpretation of the vanilla GAN (that the discriminator should output high values to real samples and low values to the generated samples). These functions should also satisfy the constraint
\begin{align}\label{eqn:condnonfnsforGAN}
    \phi(t)+\psi(t)\leq \phi\left(\frac{1}{2}\right)+\psi\left(\frac{1}{2}\right),\ \text{for all}\ t\in[0,1], 
\end{align}
so that the optimal discriminator guesses uniformly at random (i.e., outputs a constant value ${1}/{2}$ irrespective of the input) when $P_r=P_{G_\theta}$. A loss function $\ell(y,\hat{y})$ is said to be \emph{symmetric}~\cite{reid2010composite} if $\psi(t)=\phi(1-t)$, for all $t\in[0,1]$. Notice that the value function considered by Arora \emph{et al.}~\cite{AroraGLMZ17} is a special case of $\eqref{eqn:lossfnps1}$, i.e., $\eqref{eqn:lossfnps1}$ recovers the value function in \cite[Equation~(2)]{AroraGLMZ17} when the loss function $\ell(y,\hat{y})$ is symmetric. For symmetric losses, concavity of the function $\phi$ is a sufficient condition for satisfying \eqref{eqn:condnonfnsforGAN}, but not a necessary condition.

\subsection{CPE Loss GANs and $f$-divergences}


We now establish a precise correspondence between the family of GANs based on CPE loss functions and a family of $f$-divergences. 
We do this by building upon a relationship between margin-based loss functions~\cite{BartlettJM06} and $f$-divergences first demonstrated by Nguyen \emph{et al.}~\cite{NguyenWJ09} and leveraging our CPE loss function perspective of GANs given in~\eqref{eqn:lossfnps1}. 
This complements the connection established  by Nowozin \emph{et al.}~\cite{NowozinCT16} between the variational estimation approach of $f$-divergences~\cite{NguyenWJ10} and $f$-divergence based GANs.
 We call a CPE loss function $\ell(y,\hat{y})$ \emph{symmetric}~\cite{reid2010composite} if $\ell(1,\hat{y})=\ell(0,1-\hat{y})$ and an $f$-divergence $D_f(\cdot\|\cdot)$ \emph{symmetric}~\cite{liese1987convex,sason2015tight} if $D_f(P\|Q)=D_f(Q\|P)$. 
We assume GANs with sufficiently large number of samples and ample discriminator capacity.
\begin{theorem}\label{thm:correspondence}
For any symmetric CPE loss GAN with a value function  in~\eqref{eqn:lossfnps1}, the min-max optimization in \eqref{eqn:GANgeneral-background} reduces to minimizing an $f$-divergence. Conversely, for any GAN designed to minimize a symmetric $f$-divergence, there exists a (symmetric) CPE loss GAN minimizing the same $f$-divergence. 
\end{theorem}
\begin{proof}[Proof sketch]\let\qed\relax
Let $\ell$ be the symmetric CPE loss of a given CPE loss GAN; note that $\ell$ has a bivariate input $(y,\hat{y})$ ({e.g.}, in~\eqref{eq:cpealphaloss}), where $y \in \{0,1\}$ and $\hat{y} \in [0,1]$.
We define an associated margin-based loss function $\tilde{\ell}$ using a bijective link function (satisfying a mild regularity condition); note that a margin-based loss function has a univariate input $z \in \mathbb{R}$ ({e.g.}, the logistic loss $\tilde{l}^{\text{log}}(z) = \log{(1+e^{-z})}$) and the bijective link function maps $z \rightarrow \hat{y}$ (see~\cite{BartlettJM06,reid2010composite} for more details). 
We show after some manipulations that the inner optimization of the CPE loss GAN reduces to an $f$-divergence with
\begin{align}\label{eqn:thm1proofsketch1}
f(u):=-\inf_{t\in\mathbb{R}}\left(\tilde{\ell}(-t)+u\tilde{\ell}(t)\right).
\end{align}
For the converse, given a symmetric $f$-divergence, using \cite[Corollary~3 and Theorem~1(b)]{NguyenWJ09}, note that there exists a margin-based loss  $\tilde{\ell}$ such that \eqref{eqn:thm1proofsketch1} holds. The rest of the argument follows from defining a symmetric CPE loss $\ell$ from this margin-based loss $\tilde{\ell}$ via the \textit{inverse} of the same link function. See
Appendix~\ref{apndx:proof-of-thm1} 
for the detailed proof.
\end{proof}
A consequence of Theorem \ref{thm:correspondence} is that it offers an interpretable way to design GANs and connect a desired measure of divergence to a corresponding loss function, where the latter is easier to implement in practice. Moreover, CPE loss based GANs inherit the intuitive and compelling interpretation of vanilla GANs that the discriminator should assign higher likelihood values to real samples and lower ones to generated samples.

We now specialize the loss function perspective of GANs to the GAN obtained by plugging in $\alpha$-loss. We first write $\alpha$-loss in \eqref{eq:alphaloss_prob} in the form of a binary classification loss to obtain
\begin{align}
    \ell_\alpha(y,\hat{y}):=\frac{\alpha}{\alpha-1}\left(1-y\hat{y}^{\frac{\alpha-1}{\alpha}}-(1-y)(1-\hat{y})^{\frac{\alpha-1}{\alpha}}\right),\label{eqn:alphaloss}
\end{align}
 for $\alpha\in(0,1)\cup (1,\infty)$. Note that \eqref{eqn:alphaloss} recovers $\ell_{\text{CE}}$ as $\alpha\rightarrow 1$. Now consider a \emph{tunable $\alpha$-GAN} with the value function 
\begin{align}
    V_\alpha(\theta,\omega)
    &=\mathbb{E}_{X\sim P_r}[-\ell_{\alpha}(1,D_\omega(X))]+\mathbb{E}_{X\sim P_{G_\theta}}[-\ell_{\alpha}(0,D_\omega(X))]\nonumber\\
    &=\frac{\alpha}{\alpha-1}\left(\mathbb{E}_{X\sim P_r}\left[D_\omega(X)^{\frac{\alpha-1}{\alpha}}\right]+\mathbb{E}_{X\sim P_{G_\theta}}\left[\left(1-D_\omega(X)\right)^{\frac{\alpha-1}{\alpha}}\right]-2\right)\label{eqn:alphaGANobjective}.
\end{align}
We can verify that $\lim_{\alpha\rightarrow 1}V_\alpha(\theta,\omega)=V_{\text{VG}}(\theta,\omega)$, recovering the value function of the vanilla GAN. Also, notice that 
\begin{align}\label{eqn:IPM}
\lim_{\alpha\rightarrow \infty}V_\alpha(\theta,\omega)=\mathbb{E}_{X\sim P_r}\left[D_\omega(x)\right]-\mathbb{E}_{X\sim P_{G_\theta}}\left[D_\omega(x)\right]-1
\end{align}
is the value function (modulo a constant) used in Integral Probability Metric (IPM) based GANs\footnote{Note that IPMs do not restrict the function $D_\omega$ to be a probability.}, e.g., WGAN, McGan~\cite{mroueh2017mcgan}, Fisher GAN~\cite{mroueh2017fisher}, and Sobolev GAN~\cite{mroueh2017sobolev}.
The resulting min-max game in $\alpha$-GAN is given by
\begin{align} 
\inf_{\theta\in\Theta}\sup_{\omega\in\Omega}V_\alpha(\theta,\omega)\label{eqn:minimaxalphaGAN}.
\end{align}

The following theorem provides the min-max solution, i.e., Nash equilibrium, to the two-player game in \eqref{eqn:minimaxalphaGAN} for the non-parametric setting, i.e., when the discriminator set $\Omega$ is large enough.
\begin{theorem}\label{thm:alpha-GAN}
For $\alpha\in(0,1)\cup (1,\infty)$ and a generator $G_\theta$, the discriminator $D_{\omega^*}$ optimizing the $\sup$ in \eqref{eqn:minimaxalphaGAN} is 
\begin{align}\label{eqn:optimaldoisc}
    D_{\omega^*}(x)=\frac{p_r(x)^\alpha}{p_r(x)^\alpha+p_{G_\theta}(x)^\alpha}, \quad x \in \mathcal{X},
\end{align}
where $p_r$ and $p_{G_\theta}$ are the corresponding densities of the distributions $P_r$ and $P_{G_\theta}$, respectively, with respect to a base measure $dx$ (e.g., Lebesgue measure).
For this $D_{\omega^*}$, \eqref{eqn:minimaxalphaGAN} simplifies to minimizing a non-negative symmetric $f_\alpha$-divergence $D_{f_\alpha}(\cdot||\cdot)$ to obtain
\begin{align}\label{eqn:inf-obj-alpha}
    \inf_{\theta\in\Theta} D_{f_\alpha}(P_r||P_{G_\theta})+\frac{\alpha}{\alpha-1}\left(2^{\frac{1}{\alpha}}-2\right),
\end{align}
where
\begin{align}\label{eqn:falpha}
f_\alpha(u)=\frac{\alpha}{\alpha-1}\left(\left(1+u^\alpha\right)^{\frac{1}{\alpha}}-(1+u)-2^{\frac{1}{\alpha}}+2\right),
\end{align}
for $u\geq 0$ and\footnote{We note that the divergence $D_{f_\alpha}$ has been referred to as \emph{Arimoto divergence} in the literature~\cite{osterreicher1996class,osterreicher2003new,LieseV06}.}
\begin{align}\label{eqn:alpha-divergence}
D_{f_\alpha}(P||Q)=\frac{\alpha}{\alpha-1}\left(\int_\mathcal{X} \left(p(x)^\alpha+q(x)^\alpha\right)^\frac{1}{\alpha} dx-2^{\frac{1}{\alpha}}\right),
\end{align}
which is minimized iff $P_{G_\theta}=P_r$. 
\end{theorem}
A detailed proof of Theorem~\ref{thm:alpha-GAN}  is in Appendix~\ref{proofofthm1}.

\begin{remark}
As $\alpha \rightarrow 0$, note that~\eqref{eqn:optimaldoisc} implies a more cautious discriminator, i.e., if $p_{G_{\theta}}(x) \geq p_{r}(x)$, then $D_{\omega^{*}}(x)$ decays more slowly from $1/2$, and if $p_{G_{\theta}}(x) \leq p_{r}(x)$, $D_{\omega^{*}}(x)$ increases more slowly from $1/2$. 
Conversely, as $\alpha\rightarrow\infty$,~\eqref{eqn:optimaldoisc} simplifies to $D_{\omega^*}(x)=\mathbbm{1}\{p_r(x)>p_{G_\theta}(x)\}+\frac{1}{2}\mathbbm{1}\{p_r(x)=p_{G_\theta}(x)\}$, where the discriminator implements the Maximum Likelihood (ML) decision rule, i.e., a hard decision whenever $p_r(x)\neq p_{G_\theta}(x)$. In other words,~\eqref{eqn:optimaldoisc} for $\alpha \rightarrow \infty$ induces a very confident discriminator. 
Regarding the generator's perspective,~\eqref{eqn:inf-obj-alpha} implies that the generator seeks to minimize the discrepancy between $P_{r}$ and $P_{G_{\theta}}$
according to the geometry induced by $D_{f_\alpha}$.
Thus, the optimization trajectory traversed by the generator during training is strongly dependent on the practitioner's choice of $\alpha \in (0,\infty)$. 
Please refer to Fig.~\ref{fig:plotofdivergence} in Appendix \ref{proofoftheorem2} for an illustration of this observation. Figure~\ref{fig:non-overlapping_gaussians} illustrates this effect of tuning $\alpha$ on the optimal D and the corresponding loss of the generator for a toy example.
\end{remark}



\begin{figure}[t]
\centering
\footnotesize
\setlength{\tabcolsep}{20pt}
\begin{tabular}{@{}cc@{}}

\includegraphics[page=7,width=0.35\linewidth]{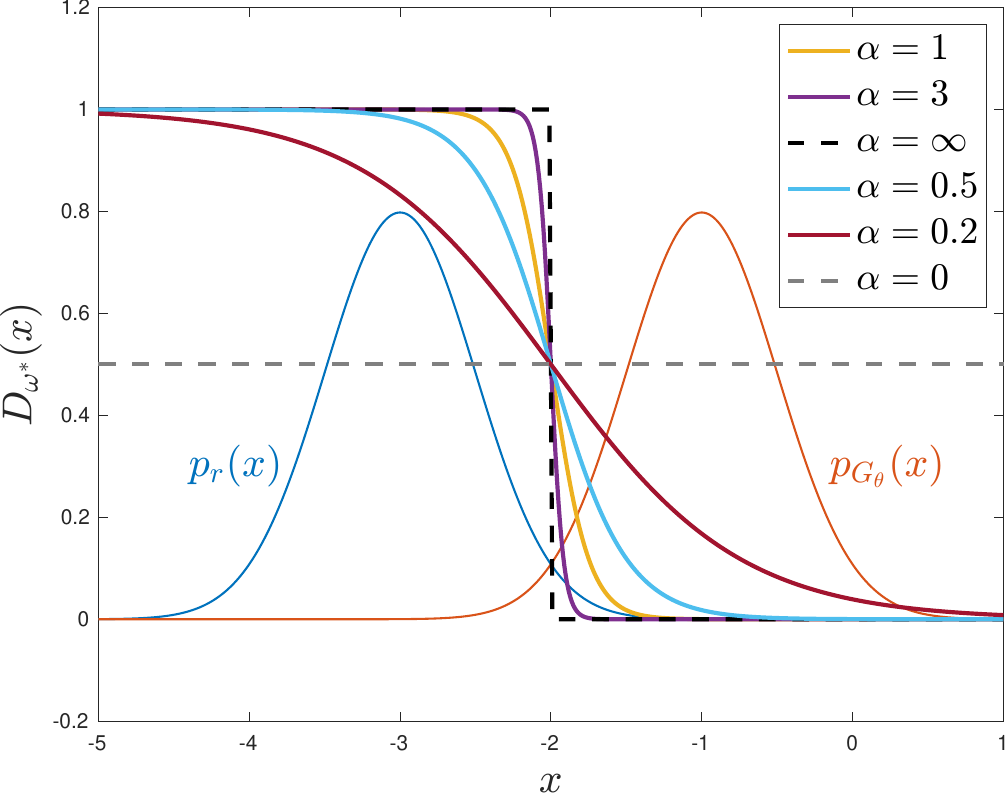}
 & \includegraphics[page=4,width=0.35\linewidth]{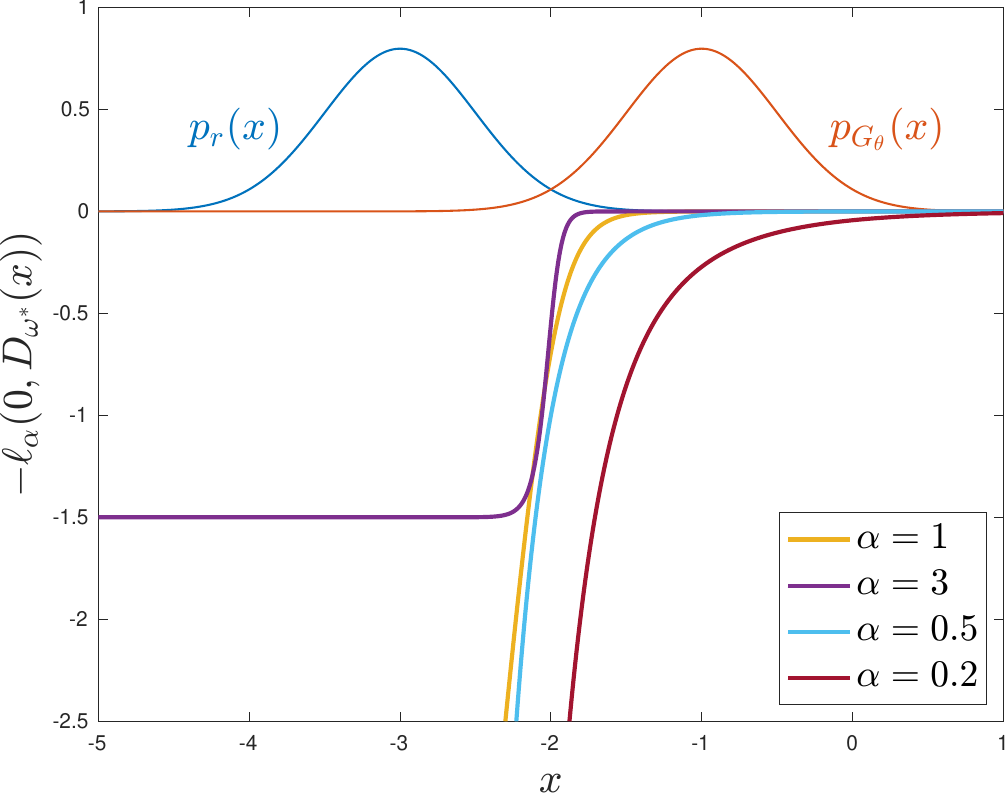}
     \\
   (a)  & (b)
\end{tabular}
\caption{A toy example of $\alpha$-GAN, where the real distribution $P_r=\mathcal{N}(-2,0.5^2)$ (blue curve) and the assumed initial generated distribution $P_{G_\theta}=\mathcal{N}(2,0.5^2)$ (orange curve). (a) A plot of the optimal discriminator output $D_{\omega^*}(x)$ for $\alpha \in \{0,0.2,0.5,1,3,\infty\}$. As $\alpha$ decreases, $D_{\omega^*}$ becomes increasingly less confident in its predictions until it outputs $1/2$ for all $x$ when $\alpha \to 0$. Conversely, as $\alpha$ increases, $D_{\omega^*}$ becomes increasingly more confident until it implements the Maximum Likelihood decision rule when $\alpha \to \infty$.   (b) A plot of the generator's corresponding loss $-\ell_\alpha(0,D_{\omega^*}(x))$ for $\alpha \in \{0.2,0.5,1,3\}$. As $\alpha$ decreases, the magnitude of the gradients of the loss increases, while increasing $\alpha$ saturates the gradients. Note that early in training, if the discriminator is very confident and outputs values close to 0 for the generated data, the generator will not have much gradient to continue learning, resulting in vanishing gradients. Decreasing $\alpha$ reduces the discriminator's confidence and provides more gradient for the generator to learn. }
\label{fig:non-overlapping_gaussians}
\end{figure}

Note that the divergence $D_{f_\alpha}(\cdot||\cdot)$ (in \eqref{eqn:alpha-divergence}) that naturally emerges from the analysis of $\alpha$-GAN was first proposed by \"{O}sterriecher~\cite{osterreicher1996class} in a statistical context of measures and was later referred to as the \emph{Arimoto divergence} by Liese and Vajda~\cite{LieseV06}. Next, we show that $\alpha$-GAN recovers various well known $f$-GANs.

\begin{theorem}\label{thm:fgans}
$\alpha$-GAN recovers vanilla GAN, Hellinger GAN (H-GAN)~\cite{NowozinCT16}, and Total Variation GAN (TV-GAN)~\cite{NowozinCT16} as $\alpha\rightarrow 1$, $\alpha=\frac{1}{2}$, and $\alpha\rightarrow \infty$, respectively.
\end{theorem}
\begin{proof}[Proof sketch]\let\qed\relax
We show the following: (i) as $\alpha\rightarrow 1$, \eqref{eqn:inf-obj-alpha} equals $\inf_{\theta\in\Theta}2D_{\text{JS}}(P_r||P_{G_\theta})-\log{4}$ recovering the vanilla GAN; (ii) for $\alpha=\frac{1}{2}$, \eqref{eqn:inf-obj-alpha} gives $2\inf_{\theta\in\Theta}D_{\text{H}^2}(P_r||P_{G_\theta})-2$ recovering Hellinger GAN (up to a constant); and (iii) as $\alpha\rightarrow\infty$, \eqref{eqn:inf-obj-alpha} equals $\inf_{\theta\in\Theta}D_{\text{TV}}(P_r||P_{G_\theta})-1$ recovering TV-GAN (modulo a constant).
A detailed proof is in Appendix~\ref{proofoftheorem2}.
\end{proof}

Next, we present an equivalence between $f_\alpha$-GAN defined using the value function in \eqref{eqn:fGANobj-activation} and $\alpha$-GAN. Define $\overline{\mathbb R} = \mathbb R \cup \{\pm \infty\}$. We first prove that there exists a mapping between the terms involved in the optimization of both GAN formulations in the following theorem. 

\begin{theorem}
For any $\alpha \in (0,1)\cup (1,\infty)$, let $\Tilde{f}_\alpha$ be a slightly modified version of \eqref{eqn:falpha} defined as
\begin{align}\label{eqn:falpha-tilde}
\Tilde{f}_\alpha(u)=\frac{\alpha}{\alpha-1}\left(\left(1+u^\alpha\right)^{\frac{1}{\alpha}}-(1+u)\right),\quad  u \ge 0,
\end{align}
with continuous extensions at $\alpha=1$ and $\alpha=\infty$. Let $\Tilde{f}^*_\alpha$ be the convex conjugate of $\Tilde{f}_\alpha$ given by
\begin{align}
        \Tilde{f}_\alpha^* (t) = \frac{\alpha}{\alpha-1}\left(1- (1-s(t))^{\frac{\alpha-1}{\alpha}} \right),
\end{align}
where
\begin{align}
s(t)=\left(1+\frac{\alpha-1}{\alpha} t\right)^{\frac{\alpha}{\alpha-1}}.
\label{eqn:s}
\end{align}
Let $g_{{f}_\alpha} : \overline{\mathbb R} \to \text{dom}( \Tilde{f}^*_\alpha)$ be a bijective output activation function. 

\begin{itemize}
    \item Given $v \in \overline{\mathbb R}$, there exists $d \in [0,1]$ such that
\begin{equation}
    g_{{f}_\alpha}(v) = -\ell_\alpha\big(1,d\big) \quad \text{and} \quad \Tilde{f}^*_\alpha(g_{{f}_\alpha}(v)) = \ell_\alpha\big(0,d\big).
    \label{eq:thm1-gen}
\end{equation}

    \item Conversely, given $d \in [0,1]$, there exists $v \in \overline{\mathbb R}$ such that \eqref{eq:thm1-gen} holds for the same function $g_{{f}_\alpha}$.
\end{itemize} 
\label{thm:obj-equiv-gen}
\end{theorem}

\begin{proof}[Proof sketch]\let\qed\relax
The result follows from comparing the corresponding terms in the $f$-GAN value function in \eqref{eqn:fGANobj-activation} (specifically for $f=\Tilde{f}_\alpha$) and the $\alpha$-GAN value function in \eqref{eqn:alphaGANobjective}.
{A detailed proof is in Appendix~\ref{appendix:obj-equiv-gen}.}
\end{proof}

Taking a closer look at the first equality in \eqref{eq:thm1-gen} and recalling that a margin-based loss is often obtained by composing a classification function (such as $\alpha$-loss) and the logistic sigmoid function, we can derive an example of such a $g_{f_\alpha}$ using the margin-based $\alpha$-loss \cite{sypherd2022journal} as
\begin{equation}
    g_{f_\alpha}(v) = \frac{\alpha}{\alpha-1}\left((1+e^{-v})^{-\frac{\alpha-1}{\alpha}}-1 \right),
    \label{eq:gfalpha-example}
\end{equation}
for $v\in \overline{\mathbb R}$ and $\alpha \ne 1$, where 
\begin{equation}
    g_{f_1}(v) = \lim_{\alpha \to 1} g_{f_\alpha}(v) = -\log(1+e^{-v})
\end{equation} 
for $v \in \overline{\mathbb R}$. The function $g_{f_\alpha}$ is monotonically increasing for any $\alpha$, with range exactly matching $\text{dom}(f_\alpha^*)$, and is therefore bijective.

The following corollary establishes the equivalence between $\Tilde{f}_\alpha$-GAN and $\alpha$-GAN. 
Two optimization problems $\sup_{v\in A}g(v)$ and $\sup_{t\in B}h(t)$ are said to be equivalent~\cite{ruderman2012tighter,belghazi2018mutual} if there exists a bijective function $k:A\rightarrow B$ such that
\begin{align}
    g(v)=h(k(v))\  \text{and}\ h(t)=g(k^{-1}(t)), \ \text{for all}\  v\in A, t\in B.
    \label{eq:equiv-opt-def}
\end{align}
In other words, two optimization problems are equivalent if a change of variable via the function $k$ can transform one into the other.

\begin{corollary}
For any $\alpha \in (0,\infty]$ and corresponding $\Tilde{f}_\alpha$ defined in \eqref{eqn:falpha-tilde}, the optimization problems involved in $\Tilde{f}_\alpha$-GAN (using \eqref{eqn:fGANobj-activation} with $f=\Tilde{f}_\alpha$) and $\alpha$-GAN (using \eqref{eqn:alphaGANobjective}) are equivalent for the choice 
\begin{equation*}
    g(Q_w)=\mathbb E_{X \sim P_r} \Big[g_{f_\alpha}\big(Q_\omega(X)\big)\Big] + \mathbb E_{X \sim P_{G_\theta}} \Big[-\Tilde{f}_\alpha^*\big(g_{f_\alpha}(Q_\omega(X))\big)\Big]
\end{equation*}
with $A = \{Q_\omega:\mathcal X \to \overline{\mathbb R} \}$ and
\begin{equation*}
    h(D_\omega)=\mathbb E_{X \sim P_r} \Big[-\ell_\alpha\big(1,D_\omega(X)\big)\Big] +\mathbb E_{X \sim P_{G_\theta}} \Big[-\ell_\alpha\big(0,D_\omega(X)\big)\Big]
\end{equation*}
with $B = \{D_\omega:\mathcal X \to [0,1] \}$ using $k:A\to B$ defined by
\[k(v) = s(g_{f_\alpha}(v))= \left(1+\left(\frac{\alpha-1}{\alpha}\right) g_{f_\alpha}(v)\right)^\frac{\alpha}{\alpha-1} , \]
 where $s$ is defined in \eqref{eqn:s} and $g_{f_\alpha}$ is a bijective output activation function mapping from $\overline{\mathbb R}$ to $\text{dom}(\Tilde{f}_\alpha^*)$.
\label{corollary:equivalence-falphaGAN-alphaGAN}
\end{corollary}

The proof of Corollary \ref{corollary:equivalence-falphaGAN-alphaGAN} follows from \eqref{eq:equiv-opt-def} and Theorem~\ref{thm:obj-equiv-gen}. The following theorem generalizes the equivalence demonstrated above between $\Tilde{f}_\alpha$-GAN and $\alpha$-GAN to an equivalence between $f$-GANs (using the original value function in \eqref{eqn:fGANobj}) and CPE loss based GANs. 
\begin{theorem}\label{theorem:equivalence-fGAN-CPEGAN}
For any given symmetric $f$-divergence, the optimization problems involved in $f$-GAN and the CPE loss based GAN minimizing the same $f$-divergence are equivalent under the following regularity conditions on $f$:
\begin{itemize}
    \item there exists a strictly convex and differentiable CPE (partial) loss function $\ell$ such that
    \begin{align}\label{eqn:f-intermsofloss}
        f(u)=\sup_{t\in[0,1]} -u\ell(t)-\ell(1-t)
    \end{align}
    (note that this condition without the requirement of strict convexity of $\ell$ is indeed guaranteed by~\cite[Theorem~2]{NguyenWJ10} for any convex function $f$ resulting in a symmetric divergence) and $-u\ell(t)-\ell(1-t)$ has a local maximum in $t$ for every $u\in\mathbb{R}_+$, and 
    \item the function mapping $u\in\mathbb{R}_+$ to unique optimizer in \eqref{eqn:f-intermsofloss} is bijective. 
\end{itemize}
\end{theorem}

\begin{proof}[Proof sketch]\let\qed\relax
Observing that the inner optimization problem in the CPE loss GAN formulation reduces to the pointwise optimization \eqref{eqn:f-intermsofloss} and that of the $f$-GAN formulation reduces to the pointwise optimization
\begin{align}\label{eqn:f-for-fGAN}
    f(u)=\sup_{v\in\text{dom}f^*} uv-f^*(v),
\end{align}
it suffices to show that the variational forms of $f$ in \eqref{eqn:f-intermsofloss} and \eqref{eqn:f-for-fGAN} are equivalent. We do this by showing that \eqref{eqn:f-intermsofloss} is equivalent to the optimization problem
\begin{align}\label{eqn:opt-mattShannon-sketch}  f(u)=\sup_{v\in\mathbb{R}_+}uf^\prime(v)-[vf^\prime(v)-f(v)],
\end{align}
which has been shown to be equivalent to \eqref{eqn:f-for-fGAN} \cite{shannon2020properties}.
A detailed proof is in Appendix~\ref{proofofthm-equv}.
\end{proof}

\begin{remark}
Since $\alpha$-loss, $\ell_\alpha(p)=\frac{\alpha}{\alpha-1}(1-p^{\frac{\alpha-1}{\alpha}})$, $p\in[0,1]$, is strictly convex for $\alpha\in(0,\infty)$, and the function mapping $u\in\mathbb{R}_+$ to unique optimizer in \eqref{eqn:f-intermsofloss} with $\alpha$-loss, i.e., $\frac{u^\alpha}{1+u^\alpha}$, is bijective, Theorem~\ref{theorem:equivalence-fGAN-CPEGAN} implies that $\alpha$-GAN is equivalent to $\Tilde{f}_\alpha$-GAN with $\Tilde{f}_\alpha$ defined in \eqref{eqn:falpha-tilde}.
\end{remark}
\begin{remark}
 Though the CPE loss GAN and $f$-GAN formulations are equivalent, the following aspects differentiate the two:
\begin{itemize}
    \item The $f$-GAN formulation focuses on the generator minimizing an $f$-divergence with no explicit emphasis on the role of the discriminator as a binary classifier in relation to the function $f$. With the CPE loss GAN formulation, we bring into the foreground the connection between the binary classification performed by the discriminator and the $f$-divergence minimization done by the generator.
    \item More importantly, the CPE loss function perspective of GANs allows us to prove convergence properties (Theorem~\ref{thm:equivalenceinconvergence}), generalization error bounds (Theorem~\ref{thm:generalizationofarora}), and estimation error bounds (Theorem~\ref{thm:estimationerror-upperbound}) as detailed in the following sections.
\end{itemize}
\end{remark}

\subsection{Convergence Guarantees for CPE Loss GANs}
Building on the above one-to-one correspondence, 
we now present \emph{convergence} results for CPE loss GANs, including $\alpha$-GAN, thereby providing a unified perspective on the convergence of a variety of $f$-divergences that arise when optimizing GANs. Here again, we assume a sufficiently large number of samples and ample discriminator capacity.
In~\cite{liu2017approximation}, Liu~\emph{et al.} address the following question in the context of convergence analysis of any GAN: 
For a sequence of generated distributions $(P_n)$, does convergence of a divergence between the generated distribution $P_{n}$ and a fixed real distribution $P$ to the global minimum lead to some standard notion of distributional convergence of $P_n$ to $P$? They answer this question in the affirmative provided the sample space $\mathcal{X}$ is a compact metric space.

Liu \emph{et al.}~\cite{liu2017approximation} formally define any divergence that results from the inner optimization of a general GAN in~\eqref{eqn:GANgeneral-background} as an \emph{adversarial divergence}~\cite[Definition~1]{liu2017approximation}, thus broadly capturing the divergences used by a number of existing GANs, including vanilla GAN~\cite{Goodfellow14}, $f$-GAN~\cite{NowozinCT16}, WGAN~\cite{ArjovskyCB17}, and MMD-GAN~\cite{dziugaite2015training}.
Indeed, the divergence that results from the inner optimization of a CPE loss GAN (including $\alpha$-GAN) in \eqref{eqn:lossfnbasedGAN} is also an adversarial divergence.
For \emph{strict adversarial divergences} (a subclass of the adversarial divergences where the minimizer of the divergence is uniquely the real distribution), 
Liu~\emph{et al.}~\cite{liu2017approximation} show that convergence of the divergence to its global minimum implies weak convergence of the generated distribution to the real distribution. Interestingly, this also leads to a structural result on the class of strict adversarial divergences~\cite[Figure~1 and Corollary~12]{liu2017approximation} based on a notion of \emph{relative strength} between adversarial divergences. 
We note that the Arimoto divergence $D_{f_{\alpha}}$ in~\eqref{eqn:alpha-divergence} is a strict adversarial divergence.
We briefly summarize the following terminology from Liu \emph{et al.}~\cite{liu2017approximation} to present our results on convergence properties of CPE loss GANs. Let $\mathcal{P}(\mathcal{X})$ be the probability simplex of distributions over $\mathcal{X}$.
\begin{definition}[Definition~11,\cite{liu2017approximation}] \label{def:equivalenceadversarialdivergence}
A {strict adversarial divergence} $\tau_1$ is said to be stronger than another strict adversarial divergence $\tau_2$ (or $\tau_2$ is said to be weaker than $\tau_1$) if for any sequence of probability distributions $(P_n)$ and target distribution $P$ (both in $\mathcal{P}(\mathcal{X})$), $\tau_1(P\|P_n)\rightarrow 0$ as $n\rightarrow \infty$ implies $\tau_2(P\|P_n)\rightarrow 0$ as $n\rightarrow \infty$. We say $\tau_1$ is equivalent to $\tau_2$ if $\tau_1$ is both stronger and weaker than $\tau_2$.
\end{definition}

Arjovsky \emph{et al.}~\cite{ArjovskyCB17} proved that the Jensen-Shannon divergence (JSD) is equivalent to the total variation distance (TVD). 
Later, Liu \emph{et al.} showed that the squared Hellinger distance is equivalent to both of these divergences, meaning that all three divergences belong to the same equivalence class (see \cite[Figure~1]{liu2017approximation}). Noticing that the squared Hellinger distance, JSD, and TVD correspond to Arimoto divergences $D_{f_\alpha}(\cdot||\cdot)$ for $\alpha=1/2$, $\alpha=1$, and $\alpha=\infty$, respectively, it is natural to ask the question: Are Arimoto divergences for all $\alpha>0$ equivalent? We answer this question in the affirmative in Theorem~\ref{thm:equivalenceinconvergence}. In fact, we prove that all symmetric $f$-divergences, including $D_{f_\alpha}$, are equivalent in convergence. 
\begin{theorem}\label{thm:equivalenceinconvergence}
Let $f_i:[0,\infty)\rightarrow \mathbb{R}$ be a convex function which is continuous at $0$ and strictly convex at $1$ such that $f_i(1)=0$, $uf_i(\frac{1}{u})=f_i(u)$, and $f_i(0)<\infty$, for $i\in\{1,2\}$. Then for a sequence of probability distributions $(P_n)_{n\in\mathbb{N}} \in \mathcal{P}(\mathcal{X})$ and a fixed distribution $P \in \mathcal{P}(\mathcal{X})$, we have $D_{f_1}(P_n||P)\rightarrow 0$ as $n\rightarrow \infty$ if and only if $D_{f_2}(P_n||P)\rightarrow 0$ as $n\rightarrow \infty$.
\end{theorem}
\begin{proof}[Proof sketch]\let\qed\relax
Note that it suffices to show that $D_{f}(\cdot\|\cdot)$ is equivalent to $D_{\text{TV}}(\cdot\|\cdot)$ for any function $f$ satisfying the conditions in the theorem. To show this, we employ an elegant result by Feldman and \"{O}sterreicher~\cite[Theorem~2]{FeldmanO89} which gives lower and upper bounds on the Arimoto divergence in terms of TVD as
\begin{align}\label{eqn:boundsonArimotomain}
    \gamma_f(D_{\text{TV}}(P||Q))\leq D_{f}(P||Q)\leq \gamma_f(1)D_{\text{TV}}(P||Q),
\end{align}
for an appropriately defined well-behaved (continuous, invertible, and bounded) function $\gamma_\alpha:[0,1]\rightarrow [0,\infty)$. 
We use the lower and upper bounds in \eqref{eqn:boundsonArimotomain} to show that $D_{f}(\cdot\|\cdot)$ is stronger than $D_{\text{TV}}(\cdot\|\cdot)$, and $D_{f}(\cdot\|\cdot)$ is weaker than $D_{\text{TV}}(\cdot\|\cdot)$, respectively. Proof details are in 
Appendix~\ref{proofoftheorem4}. 
\end{proof}

\begin{remark}
We note that the proof techniques used in proving Theorem~\ref{thm:equivalenceinconvergence} give rise to a conceptually simpler proof of equivalence between JSD ($\alpha = 1$) and TVD ($\alpha = \infty$) proved earlier by Arjovsky \emph{et al.}~\cite[Theorem~2(1)]{ArjovskyCB17}, where measure-theoretic analysis was used. In particular, our proof of equivalence relies on the fact that TVD upper bounds JSD~\cite[Theorem~3]{Lin91}. See 
Appendix \ref{appendix:simpler-equivalence-JSD-TVD}
for details.
\end{remark}

Theorems \ref{thm:correspondence} through \ref{thm:equivalenceinconvergence} hold in the ideal setting of sufficient samples and discriminator capacity. In practice, however, GAN training is limited by both the number of training samples as well as the choice of $G_\theta$ and $D_\omega$. In fact, recent results by Arora \textit{et al.} \cite{AroraGLMZ17} show that under such limitations, convergence in divergence does not imply convergence in distribution, and have led to new metrics for evaluating GANs. To address these limitations, we consider two measures to evaluate the performance of GANs, namely generation and estimation errors, as detailed below.
\subsection{Generalization and Estimation Error Bounds for CPE Loss GANs}
\label{subsec:est-and-gen-error-single-obj}
Arora \emph{et al.}~\cite{AroraGLMZ17} defined \emph{generalization} in GANs as the scenario when the divergence between the real distribution and the generated distribution is well-captured by the divergence between their empirical versions. In particular, a divergence or distance\footnote{For consistency with other works on generalization and estimation error, we refer to a semi-metric as a distance.} $d(\cdot,\cdot)$ between distributions \emph{generalizes} with $m$ training samples and error $\epsilon>0$ if, for the learned distribution $P_{G_\theta}$, the following holds with high probability:
\begin{align}
    \left|d(P_{r},P_{G_\theta})-d(\hat{P}_r,\hat{P}_{G_\theta}) \right|\leq \epsilon,
\end{align}
where $\hat{P}_r$ and $\hat{P}_{G_\theta}$ are the empirical versions of the real  (with $m$ samples) and the generated (with a polynomial number of samples) distributions, respectively. Arora \emph{et al.}~\cite[Lemma~1]{AroraGLMZ17} show that the Jensen-Shannon divergence and Wasserstein distance do not generalize with any polynomial number of samples. However, they show that generalization can be achieved for a new notion of divergence, the \emph{neural net divergence}, with a moderate number of training examples~\cite[Theorem~3.1]{AroraGLMZ17}. 
To this end, they consider
the following optimization problem
\begin{align}
    \inf_{\theta\in\Theta}d_{\mathcal{F}}(P_r,P_{G_\theta}),
\end{align}
where $d_{\mathcal{F}}(P_r,P_{G_\theta})$ is the neural net divergence defined as
\begin{align}
    d_{\mathcal{F}}(P_r,P_{G_\theta})=\sup_{\omega\in\Omega}\left(\mathbb{E}_{X\sim P_r}[\phi\left({D_\omega(X)}\right)]+\mathbb{E}_{X\sim P_{G_\theta}}[\phi\left(1-D_\omega(X)\right)]\right) -2\phi\left(\frac{1}{2}\right)
    \label{eq:nn-divergence}
\end{align}
such that the class of discriminators $\mathcal{F}=\{D_\omega:\omega\in\Omega\}$ is $L$-Lipschitz with respect to the parameters $\omega$, i.e., for every $x\in\mathcal{X}$, $|D_{\omega_1}(x)-D_{\omega_2}(x)|\leq L||\omega_1-\omega_2||$, for all $\omega_1,\omega_2\in\Omega$, and the function $\phi$ takes values in $[-\Delta,\Delta]$ and is $L_{\phi}$-Lipschitz. Let $p$ be the discriminator capacity (i.e., number of parameters) and $\epsilon>0$. For these assumptions, in \cite[Theorem~3.1]{AroraGLMZ17}, Arora \emph{et al.} prove that \eqref{eq:nn-divergence} generalizes. We summarize their result as follows: for the empirical versions $\hat{P}_r$ and $\hat{P}_{G_\theta}$ of two distributions $P_r$ and $P_{G_\theta}$, respectively, with at least $m$ random samples each, there exists a universal constant $c$ such that when $m\geq \frac{cp\Delta^2\log{\left(LL_{\phi}p/\epsilon\right)}}{\epsilon^2}$, with probability at least $1-\exp{(-p)}$ (over the randomness of samples),
\begin{align}
    \left\lvert{d}_\mathcal{F}(P_r,P_{G_\theta})-{d}_\mathcal{F}(\hat{P}_r,\hat{P}_{G_\theta})\right\rvert\leq \epsilon.
\end{align}
Our first contribution is to show that we can generalize \eqref{eq:nn-divergence} and \cite[Theorem~3.1]{AroraGLMZ17} to incorporate any partial losses $\phi$ and $\psi$ (not just those that are symmetric). To this end, we first define the \emph{refined neural net divergence} as  
\begin{align}
    \tilde{d}_{\mathcal{F}}(P_r,P_{G_\theta})=\sup_{\omega\in\Omega}\left(\mathbb{E}_{X\sim P_r}[\phi\left({D_\omega(X)}\right)]+\mathbb{E}_{X\sim P_{G_\theta}}[\psi\left(D_\omega(X)\right)]\right)-\phi\left(\frac{1}{2}\right)-\psi\left(\frac{1}{2}\right),
\end{align}
where the discriminator class is same as the above and the functions $\phi$ and $\psi$ take values in $[-\Delta,\Delta]$ and are $L_\phi$- and $L_\psi$-Lipschitz, respectively. Note that the functions $\phi$ and $\psi$ should also satisfy 
\eqref{eqn:condnonfnsforGAN} so as to respect the optimality of the uniformly random discriminator when $P_r=P_{G_\theta}$. The following theorem shows that the refined neural net divergence generalizes with a moderate number of training examples, thus extending \cite[Theorem~3.1]{AroraGLMZ17}.

\begin{theorem}\label{thm:generalizationofarora}
Let $\hat{P}_r$ and $\hat{P}_{G_\theta}$ be empirical versions of two distributions $P_r$ and $P_{G_\theta}$, respectively, with at least $m$ random samples each. For $\Delta, p, L, L_\phi,L_\psi,\epsilon>0$ defined above, there exists a universal constant $c$ such that when $m\geq \frac{cp\Delta^2\log{\left(L\max\{L_{\phi},L_{\psi}\}p/\epsilon\right)}}{\epsilon^2}$, we have that with probability at least $1-\exp{(-p)}$ (over the randomness of samples),
\begin{align}
    \left\lvert\tilde{d}_\mathcal{F}(P_r,P_{G_\theta})-\tilde{d}_\mathcal{F}(\hat{P}_r,\hat{P}_{G_\theta})\right\rvert\leq \epsilon\label{eqn:thm5}.
\end{align}
\end{theorem}

When $\phi(t)=t$ and $D_\omega=f_\omega$ can take values in $\mathbb R$ (not just in $[0,1]$), \eqref{eq:nn-divergence} yields the so-called \emph{neural net ($nn$) distance}\footnote{This term was first introduced in~\cite{AroraGLMZ17} but with a focus on a discriminator $D_\omega$ taking values in $[0,1]$. Ji \emph{et al.} \cite{ji2018minimax,JiZL21} generalized it to $D_\omega = f_\omega$ taking values in $\mathbb R$.}~\cite{AroraGLMZ17,ji2018minimax,JiZL21} given by
\begin{align}
    d_{\mathcal{F}_{nn}}(P_r,P_{G_\theta}) =\sup_{\omega\in\Omega}\left (\mathbb{E}_{X\sim P_r}\left[f_\omega(X)\right]-\mathbb{E}_{X\sim P_{G_\theta}}\left[f_\omega(X)\right] \right),
    \label{eq:nn-distance}
\end{align}
where the discriminator\footnote{In~\cite{JiZL21}, $f_{\omega}$ indicates a discriminator function that takes values in $\mathbb{R}$.} and generator $f_\omega(\cdot)$ and $G_\theta(\cdot)$, respectively, are neural networks.
Using \eqref{eq:nn-distance}, Ji \textit{et al.} \cite{JiZL21} defined and studied the notion of \textit{estimation error}, which quantifies the effectiveness of the generator (for a corresponding optimal discriminator model) in learning the real distribution with limited samples. In order to define estimation error for CPE-loss GANs (including $\alpha$-GAN), we first introduce a \emph{loss-inclusive neural net divergence}\footnote{We refer to this measure as a divergence since it may not be a semi-metric for all choices of the loss $\ell$.} $d^{(\ell)}_{\mathcal{F}_{nn}}$ to highlight the effect of the \emph{loss} on the error. For training samples $S_x=\{X_1,\dots,X_n\}$ and $S_z=\{Z_1,\dots,Z_m\}$ from $P_r$ and $P_Z$, respectively, 
we begin with the following minimization for GAN training:
\begin{align}\label{eqn:training-empirical}
    \inf_{\theta\in\Theta}d^{(\ell)}_{\mathcal{F}_{nn}}(\hat{P}_r,\hat{P}_{G_\theta}),
\end{align}
where $\hat{P}_r$ and $\hat{P}_{G_\theta}$ are the empirical real and generated distributions estimated from $S_x$ and $S_z$, respectively, and 
\begin{align}
d^{(\ell)}_{\mathcal{F}_{nn}}(\hat{P}_r,\hat{P}_{G_\theta})=\sup_{\omega\in\Omega}\left(\mathbb{E}_{X\sim \hat{P}_{r}}[\phi \big(D_\omega(X)] \big)+\mathbb{E}_{X\sim \hat{P}_{G_\theta}}[\psi \big(D_\omega(X)] \big)\right) -\phi\left(\frac{1}{2}\right)-\psi\left(\frac{1}{2}\right),
    \label{eq:loss-nn-distance}
\end{align}
where for brevity we henceforth use $\phi(\cdot)\coloneqq -\ell(1,\cdot)$ and $\psi(\cdot)\coloneqq -\ell(0,\cdot)$. As proven in Theorem~\ref{thm:fgans}, for $\ell=\ell_\alpha$ and $\alpha=\infty$, \eqref{eq:loss-nn-distance} reduces to the neural net total variation distance. 

As a step towards obtaining bounds on the estimation error, we consider the following setup, analogous to that in \cite{JiZL21}.
For $x\in\mathcal{X}\coloneqq\{x\in\mathbb{R}^d:||x||_2\leq B_x\}$ and  $z\in\mathcal{Z}\coloneqq\{z\in\mathbb{R}^p:||z||_2\leq B_z\}$, we consider  discriminators and generators as neural network models of the form:
\begin{align}
    D_\omega&:x\mapsto \sigma\left(\mathbf{w}_k^\mathsf{T}r_{k-1}(\mathbf{W}_{d-1}r_{k-2}(\dots r_1(\mathbf{W}_1(x)))\right)\,  \label{eqn:disc-model}\\
    G_\theta&:z\mapsto \mathbf{V}_ls_{l-1}(\mathbf{V}_{l-1}s_{l-2}(\dots s_1(\mathbf{V}_1z))),
\end{align}
where $\mathbf{w}_k$ is a parameter vector of the output layer; for $i\in[1:k-1]$ and $j\in[1:l]$, $\mathbf{W}_i$ and $\mathbf{V}_j$ are parameter matrices; $r_i(\cdot)$ and $s_j(\cdot)$ are entry-wise activation functions of layers $i$ and $j$, i.e., for $\mathbf{a}\in\mathbb{R}^t$, $r_i(\mathbf{a})=\left[r_i(a_1),\dots,r_i(a_t)\right]$ and $s_i(\mathbf{a})=\left[s_i(a_1),\dots,s_i(a_t)\right]$; and $\sigma(\cdot)$ is the sigmoid function given by $\sigma(p)=1/(1+\mathrm{e}^{-p})$ (note that $\sigma$ does not appear in the discriminator in \cite[Equation~(7)]{JiZL21} as the discriminator considered in the neural net distance is not a soft classifier mapping to $[0,1]$). We assume that each $r_i(\cdot)$ and $s_j(\cdot)$ are $R_i$- and $S_j$-Lipschitz, respectively, and also that they are positive homogeneous, i.e., $r_i(\lambda p)=\lambda r_i(p)$ and $s_j(\lambda p)=\lambda s_j(p)$, for any $\lambda\geq 0$ and $p\in\mathbb{R}$. Finally, as modelled in \cite{neyshabur2015norm,salimans2016weight,golowich2018size,JiZL21}, we assume that the Frobenius norms of the parameter matrices are bounded, i.e., $||\mathbf{W}_i||_F\leq M_i$, $i\in[1:k-1]$, $||\mathbf{w}_k||_2\leq M_k$, and $||\mathbf{V}_j||_F\leq N_j$, $j\in[1:l]$. 

We define the estimation error for a CPE loss GAN as 
\begin{align}\label{eqn:estimation-error-def}
    d^{(\ell)}_{\mathcal{F}_{nn}}(P_r,{P}_{G_{\hat{\theta}^*}})-\inf_{\theta\in\Theta} d^{(\ell)}_{\mathcal{F}_{nn}}(P_r,P_{G_{\theta}}),
\end{align}
where $\hat{\theta}^*$ is the minimizer of \eqref{eqn:training-empirical} and present the following upper bound on the error. We also specialize these bounds for $\alpha$-GANs, relying on the Rademacher complexity of this loss class to do so.
\begin{theorem}\label{thm:estimationerror-upperbound}
For the setting described above, additionally assume that the functions $\phi(\cdot)$ and $\psi(\cdot)$ are $L_\phi$- and $L_\psi$-Lipschitz, respectively.
Then, with probability at least $1-2\delta$ over the randomness of training samples $S_x=\{X_i\}_{i=1}^n$ and $S_z=\{Z_j\}_{j=1}^m$, we have
\begin{align}
    d^{(\ell)}_{\mathcal{F}_{nn}}(P_r,\hat{P}_{G_{\hat{\theta}^*}})-\inf_{\theta\in\Theta} d^{(\ell)}_{\mathcal{F}_{nn}}(P_r,P_{G_{\theta}}) \leq &\frac{L_\phi B_xU_\omega\sqrt{3k}}{\sqrt{n}}+\frac{L_\psi U_\omega U_\theta B_z\sqrt{3(k+l-1)}}{\sqrt{m}}\nonumber\\
    &\hspace{12pt}+U_\omega\sqrt{\log{\frac{1}{\delta}}}\left(\frac{L_\phi B_x}{\sqrt{2n}}+\frac{L_\psi B_zU_\theta}{\sqrt{2m}}\right), \label{eq:estimationboundrhs2}
\end{align}
where $U_\omega\coloneqq M_k\prod_{i=1}^{k-1}(M_iR_i)$ and $U_\theta\coloneqq N_l\prod_{j=1}^{l-1}(N_jS_j)$.

In particular, when this bound is specialized to the case of $\alpha$-GAN by letting $\phi(p)=\psi(1-p)=\frac{\alpha}{\alpha-1}\left(1-p^{\frac{\alpha-1}{\alpha}}\right)$, the resulting bound is nearly identical to the terms in the RHS of~\eqref{eq:estimationboundrhs2}, except for substitutions $L_\phi \leftarrow 4C_{Q_x}(\alpha)$ and $L_\psi \leftarrow 4C_{Q_z}(\alpha)$, where $Q_x\coloneqq U_\omega B_x$, $Q_z\coloneqq U_\omega U_\theta B_z$, and
\begin{align} \label{eq:clipalpha-single}
    C_h(\alpha)\coloneqq\begin{cases}\sigma(h)\sigma(-h)^{\frac{\alpha-1}{\alpha}}, \ &\alpha\in(0,1]\\
    \left(\frac{\alpha-1}{2\alpha-1}\right)^{\frac{\alpha-1}{\alpha}}\frac{\alpha}{2\alpha-1}, &\alpha\in(1,\infty).
    \end{cases}
\end{align}
\end{theorem}
\begin{proof}[Proof sketch]\let\qed\relax
Our proof involves the following steps:
\begin{itemize}[leftmargin=*]
    \item Building upon the proof techniques of Ji~\emph{et al.} \cite[Theorem~1]{JiZL21}, we bound the estimation error in terms of Rademacher complexities of \emph{compositional} function classes involving the CPE loss function. 
    \item We then upper bound these Rademacher complexities leveraging a contraction lemma for Lipschitz loss functions~\cite[Lemma~26.9]{shalev2014understanding}. We remark that this differs considerably from the way the bounds on Rademacher complexities in \cite[Corollary~1]{JiZL21} are obtained because of the explicit role of the loss function in our setting.  
    \item For the case of $\alpha$-GAN, we extend a result by Sypherd \emph{et al.}~\cite{sypherd2022journal} where they showed that $\alpha$-loss is Lipschitz for a logistic model with~\eqref{eq:clipalpha}. Noting that similar to the logistic model, we also have a sigmoid in the outer layer of the discriminator, we generalize the preceding observation by proving that $\alpha$-loss is Lipschitz when the input is equal to a sigmoid function acting on a \textit{neural network} model. This is the reason behind the dependence of the Lipschitz constant on the neural network model parameters (in terms of $Q_x$ and $Q_z$). Note that~\eqref{eq:clipalpha} is monotonically decreasing in $\alpha$, indicating the bound saturates. However, one is not able to make definitive statements regarding the estimation bounds for relative values of $\alpha$ because the LHS in~\eqref{eq:estimationboundrhs2} is \textit{also} a function of $\alpha$. Proof details are in
    Appendix~\ref{proofoftheorem3}.
\end{itemize}
\end{proof}

We now focus on developing lower bounds on the estimation error. Due to the fact that oft-used techniques to obtain min-max lower bounds on the quality of an estimator (e.g., LeCam's methods, Fano's methods, etc.) require a semi-metric distance measure, we restrict our attention to a particular $\alpha$-GAN, namely that for $\alpha=\infty$, to derive a matching lower bound on the estimation error. We consider the loss-inclusive neural net divergence in \eqref{eq:loss-nn-distance} with $\ell=\ell_\alpha$ for $\alpha=\infty$, which, for brevity, we henceforth denote as $d^{\ell_\infty}_{\mathcal{F}_{nn}}(\cdot,\cdot)$
As in \cite{JiZL21}, suppose the generator's class $\{G_\theta\}_{\theta \in \Theta}$ is rich enough such that the generator $G_\theta$ can learn the real distribution $P_r$ and that the number $m$ of training samples in $S_z$ scales faster than the number $n$ of samples in $S_x$\footnote{Since the noise distribution $P_Z$ is known, one can generate an arbitrarily large number $m$ of noise samples.}. Then $\inf_{\theta \in \Theta} d^{\ell_\infty}_{\mathcal{F}_{nn}}(P_r,P_{G_\theta}) = 0$, so the estimation error simplifies to the single term $d^{\ell_\infty}_{\mathcal{F}_{nn}}(P_r,P_{G_{\hat{\theta}^*}})$. Furthermore, the upper bound in \eqref{eq:estimationboundrhs2} reduces to $O(c/\sqrt{n})$ for some constant $c$ (note that, in \eqref{eq:clipalpha-single}, $C_h(\infty)=1/4$). In addition to the above assumptions, also assume the activation functions $r_i$ for $i \in [1:k-1]$ are either strictly increasing or ReLU. For the above setting, we derive a matching min-max lower bound (up to a constant multiple) on the estimation error.
\begin{theorem}
\label{thm:est-error-lower-bound-alpha-infinity}
For the setting above, let $\hat{P}_n$ be an estimator of $P_r$ learned using the training samples $S_x=\{X_i \}_{i=1}^n$. Then,
\[\inf_{\hat{P}_n} \sup_{P_r \in \mathcal{P}(\mathcal{X})} \, \mathbb P\left\{d^{\ell_\infty}_{\mathcal{F}_{nn}}(\hat{P}_n,P_r) \ge \frac{C(\mathcal{P}(\mathcal{X}))}{\sqrt{n}} \right\} > 0.24,\]
where the constant $C(\mathcal{P}(\mathcal{X}))$ is given by
\begin{align}
    C(\mathcal{P}(\mathcal{X})) = \frac{\log(2)}{20} \Big[ \sigma&(M_k r_{k-1}(\dots r_1(M_1 B_x))- \sigma(M_k r_{k-1}(\dots r_1(-M_1 B_x)) \Big].
    \label{eq:est-error-lower-bound-constant}
\end{align}
\end{theorem}

\begin{proof}[Proof sketch]\let\qed\relax
    To obtain min-max  lower bounds, we first prove that $d^{\ell_\infty}_{\mathcal{F}_{nn}}$ is a semi-metric. The remainder of the proof is similar to that of \cite[Theorem 2]{JiZL21}, replacing $d_{\mathcal{F}_{nn}}$ with $d^{\ell_\infty}_{\mathcal{F}_{nn}}$. Finally, we note that the additional sigmoid activation function after the last layer in D satisfies the monotonicity assumption as detailed in Appendix \ref{appendix:est-error-lower-bound-alpha-infinity}. A challenge that remains to be addressed is to verify if $d^{\ell_\alpha}_{\mathcal{F}_{nn}}$ is a semi-metric for $\alpha<\infty$.  
\end{proof}

\section{Dual-objective GANs}
\label{sec:dual-objective}


As illustrated in Fig. \ref{fig:non-overlapping_gaussians}, tuning $\alpha<1$ provides more gradient for the generator to learn early in training when the discriminator more confidently classifies the generated data as fake, alleviating vanishing gradients, and also creates a smooth landscape for the generated data to descend towards the real data, alleviating exploding gradients. However, tuning $\alpha < 1$  may provide too large of gradients for the generator when the generated samples approach the real samples, which can result in too much movement of the generated data, potentially repelling it from the real data. The following question therefore arises: Can we combine a less confident discriminator with a more stable generator loss? We show that we can do so by using different objectives for the discriminator and generator, resulting in $(\alpha_D,\alpha_G)$-GANs.

\subsection{$(\alpha_D,\alpha_G)$-GANs}

We propose a dual-objective $(\alpha_D,\alpha_G)$-GAN with different objective functions for the generator and discriminator in which the discriminator maximizes $V_{\alpha_D}(\theta,\omega)$ while the generator minimizes $V_{\alpha_G}(\theta,\omega)$, where
\medmuskip=0mu
\begin{align}
    &V_{\alpha}(\theta,\omega)=\mathbb{E}_{X\sim P_r}[-\ell_{\alpha}(1,D_\omega(X))]+\mathbb{E}_{X\sim P_{G_\theta}}[-\ell_{\alpha}(0,D_\omega(X))]\label{eqn:sat-gen-objective},
\end{align}
for $\alpha=\alpha_D,\alpha_G \in (0,\infty]$.
We recover the $\alpha$-GAN \cite{KurriSS21,kurri-2022-convergence} value function when $\alpha_D=\alpha_G=\alpha$. The resulting $(\alpha_D,\alpha_G)$-GAN is given by
\begin{subequations}
\begin{align} 
&\sup_{\omega\in\Omega}V_{\alpha_D}(\theta,\omega) \label{eqn:disc_obj} \\
& \inf_{\theta\in\Theta} V_{\alpha_G}(\theta,\omega)
\label{eqn:gen_obj}.
\end{align}
\label{eqn:alpha_D,alpha_G-GAN}
\end{subequations}

{We maintain the same ordering as the original min-max GAN formulation for this non-zero sum game, wherein for a set of chosen parameters for both players, the discriminator plays first, followed by the generator.} The following theorem presents the conditions under which the optimal generator learns the real distribution $P_r$ when the discriminator set $\Omega$ is large enough.

\begin{theorem}\label{thm:alpha_D,alpha_G-GAN-saturating}
For the game in \eqref{eqn:alpha_D,alpha_G-GAN} with $(\alpha_D,\alpha_G)\in (0,\infty]^2$, given a generator $G_\theta$, the discriminator optimizing \eqref{eqn:disc_obj} is 
\begin{align}
    D_{\omega^*}(x)=\frac{p_r(x)^{\alpha_D}}{p_r(x)^{\alpha_D}+p_{G_\theta}(x)^{\alpha_D}}, \quad x \in \mathcal{X}.
    \label{eqn:optimaldisc-gen-alpha-GAN}
\end{align}
For this $D_{\omega^*}$ and the function $f_{\alpha_D,\alpha_G}:\mathbb R_+ \to \mathbb R$ defined as
\begin{align}\label{eqn:f-alpha_d,alpha_g}
f_{\alpha_D,\alpha_G}(u)=\frac{\alpha_G}{\alpha_G-1}\left(\frac{u^{\alpha_D\left(1-\frac{1}{\alpha_G}\right)+1}+1}{(u^{\alpha_D}+1)^{1-\frac{1}{\alpha_G}}}-2^{\frac{1}{\alpha_G}}\right),
\end{align}
\eqref{eqn:gen_obj} simplifies to minimizing a non-negative symmetric $f_{\alpha_D,\alpha_G}$-divergence $D_{f_{\alpha_D,\alpha_G}}(\cdot||\cdot)$ as
\begin{align}\label{eqn:gen-alpha_d,alpha_g-obj}
    \inf_{\theta\in\Theta} D_{f_{\alpha_D,\alpha_G}}(P_r||P_{G_\theta})+\frac{\alpha_G}{\alpha_G-1}\left(2^{\frac{1}{\alpha_G}}-2\right),
\end{align}
which is minimized iff $P_{G_\theta}=P_r$ for $(\alpha_D,\alpha_G)$ such that  $\Big (\alpha_D \le 1,\;\alpha_G > \frac{\alpha_D}{\alpha_D+1} \Big ) \; \text{ or } \; \Big ( \alpha_D > 1,\;\frac{\alpha_D}{2}< \alpha_G \le \alpha_D \Big )$. 
\end{theorem}
\begin{proof}[Proof sketch]\let\qed\relax
    We substitute the optimal discriminator of \eqref{eqn:disc_obj} into the objective function of \eqref{eqn:gen_obj} and write the resulting expression in the form
    \begin{align}
    \int_\mathcal{X} p_{G_\theta}(x)f_{\alpha_D,\alpha_G}\left(\frac{p_r(x)}{p_{G_\theta}(x)}\right) dx + \frac{\alpha_G}{\alpha_G-1}\left(2^{\frac{1}{\alpha_G}}-2\right).
    \label{eq:gen-obj-with-opt-disc}
    \end{align}
    We then find the conditions on $\alpha_D$ and $\alpha_G$ for $f_{\alpha_D,\alpha_G}$ to be strictly convex so that the first term in \eqref{eq:gen-obj-with-opt-disc} is an $f$-divergence. Figure \ref{fig:convexity_regions}(a)  in Appendix \ref{appendix:alpha_D,alpha_G-GAN-saturating} illustrates the feasible $(\alpha_D,\alpha_G)$-region. A detailed proof can be found in Appendix \ref{appendix:alpha_D,alpha_G-GAN-saturating}. See Fig. \ref{fig:(alpha_d,alpha_g)-GAN-sat-gen-loss} for a toy example illustrating the value of tuning $\alpha_D<1$ and $\alpha_G\ge1$.
\end{proof}


\begin{figure}[t]
\centering
\footnotesize
\setlength{\tabcolsep}{10pt}
\begin{tabular}{@{}ccc@{}}
  \includegraphics[page=3,width=0.3\linewidth]{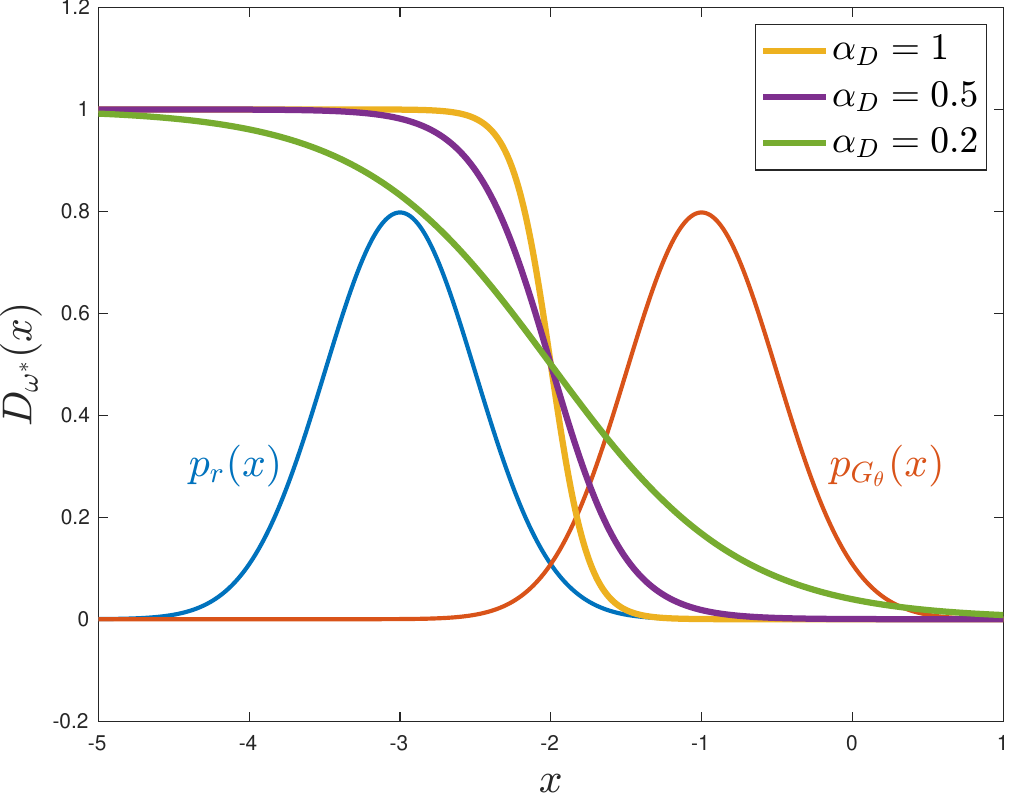}  & {\includegraphics[page=5,width=0.3\linewidth]{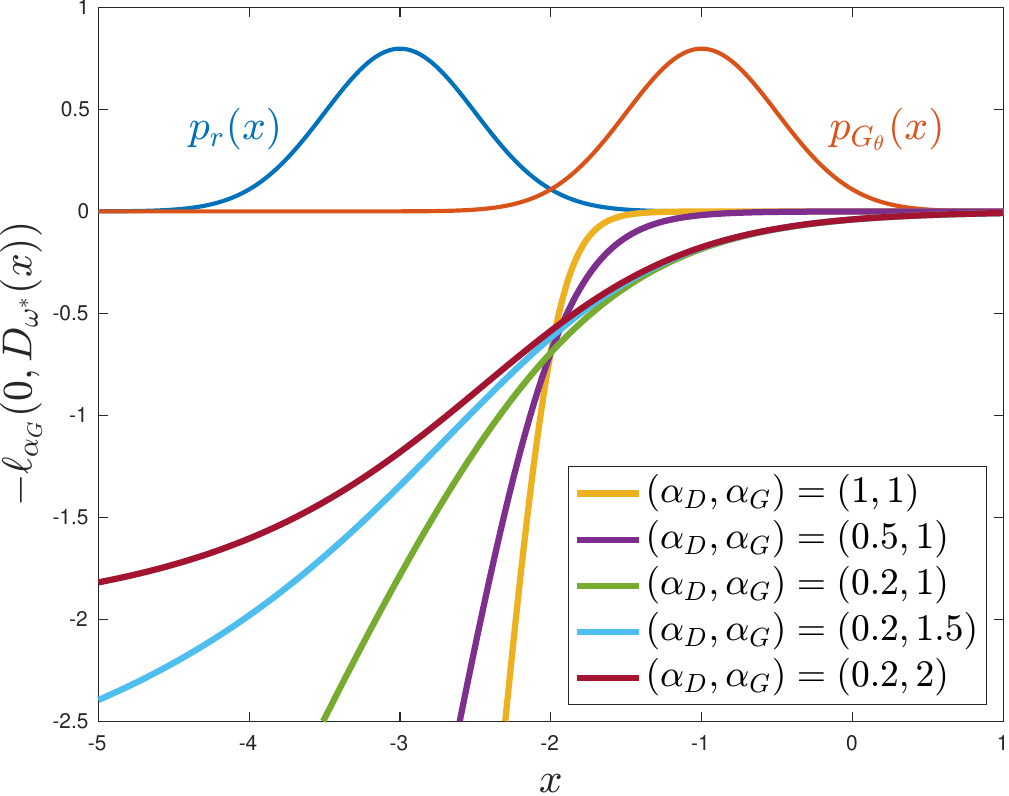}} 
  & \includegraphics[page=5,width=0.3\linewidth]{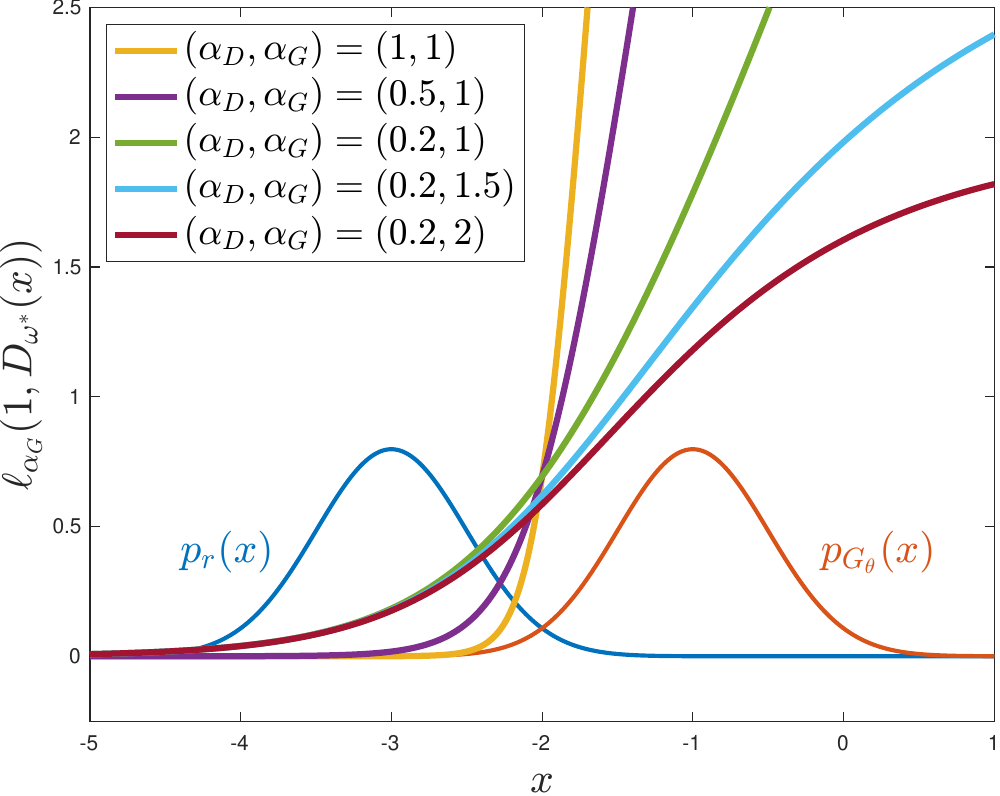}\\
   (a) & (b) & (c)
\end{tabular}
\caption{(a) A plot of the optimal discriminator output $D_{\omega^*}(x)$ in \eqref{eqn:optimaldisc-gen-alpha-GAN} for several values of $\alpha_D \le 1$ for the same toy example as in Figure \ref{fig:non-overlapping_gaussians}. Tuning $\alpha_D < 1$ reduces the confidence of the optimal discriminator $D_{\omega^*}$. (b) A plot of the generator's loss $-\ell_{\alpha_G}(0,D_{\omega^*}(x))$ for several values of $(\alpha_D\le1,\alpha_G\ge1)$. Tuning $\alpha_D<1$ and $\alpha_G=1$ provides larger gradients for the generated data far from the real data, thereby alleviating vanishing gradients, and also provides smaller gradients for generated data close to the real data, helping to combat exploding gradients. Tuning $\alpha_G \ge 1$ yields a quasiconcave objective, further reducing the magnitude of the gradients for generated data approaching the real data. (c) A plot of the generator's NS loss $\ell_{\alpha_G}(1,D_{\omega^*}(x))$ for several values of $(\alpha_D\le1,\alpha_G\ge1)$. Tuning $\alpha_D<1$ and $\alpha_G = 1$ reduces the magnitude of the gradients for generated data far from the real data, which can help stabilize training by decreasing sensitivity to hyperparameter initialization and alleviating model oscillation; tuning $\alpha_G > 1$ yields a quasiconvex generator objective, which can potentially further improve training stability.}
\label{fig:(alpha_d,alpha_g)-GAN-sat-gen-loss}
\end{figure}



Noting that $\alpha$-GAN recovers various well-known GANs, including the vanilla GAN, which is prone to saturation, the $(\alpha_D,\alpha_G)$-GAN formulation using the generator objective function in \eqref{eqn:sat-gen-objective} can similarly saturate early in training, potentially causing vanishing gradients. We propose the following NS alternative to the generator's objective in \eqref{eqn:sat-gen-objective}:
\begin{align}
    V^\text{NS}_{\alpha_G}(\theta,\omega) &= \mathbb{E}_{X\sim P_{G_\theta}}[\ell_{\alpha_G}(1,D_\omega(X))],
    \label{eqn:nonsat-gen-objective}
\end{align}
thereby replacing \eqref{eqn:gen_obj} with
\begin{align}
    \inf_{\theta\in\Theta} V^\text{NS}_{\alpha_G}(\theta,\omega).
\label{eqn:gen_obj_ns}
\end{align}

Comparing \eqref{eqn:gen_obj} and \eqref{eqn:gen_obj_ns}, note that the additional expectation term over $P_r$ in \eqref{eqn:sat-gen-objective} results in \eqref{eqn:gen_obj} simplifying to a symmetric divergence for $D_{\omega^*}$ in \eqref{eqn:optimaldisc-gen-alpha-GAN}, whereas the single term in \eqref{eqn:nonsat-gen-objective} will result in \eqref{eqn:gen_obj_ns} simplifying to an asymmetric divergence.
The optimal discriminator for this NS game remains the same as in \eqref{eqn:optimaldisc-gen-alpha-GAN}. The following theorem provides the solution to \eqref{eqn:gen_obj_ns} under the assumption that the optimal discriminator can be attained.

\begin{theorem}\label{thm:alpha_D,alpha_G-GAN-nonsaturating}
For the same $D_{\omega^*}$ in \eqref{eqn:optimaldisc-gen-alpha-GAN} and the function $f_{\alpha_D,\alpha_G}^\text{NS}:\mathbb R_+ \to \mathbb R$ defined as
\begin{align}\label{eqn:f-alpha_d,alpha_g-ns}
f^\text{NS}_{\alpha_D,\alpha_G}(u)=\frac{\alpha_G}{\alpha_G-1}\left(2^{\frac{1}{\alpha_G}-1}-\frac{u^{\alpha_D\left(1-\frac{1}{\alpha_G}\right)}}{(u^{\alpha_D}+1)^{1-\frac{1}{\alpha_G}}}\right),
\end{align}
\eqref{eqn:gen_obj} simplifies to minimizing a non-negative asymmetric $f^\text{NS}_{\alpha_D,\alpha_G}$-divergence $D_{f^{\text{NS}}_{\alpha_D,\alpha_G}}(\cdot||\cdot)$ as
\begin{align}\label{eqn:gen-alpha_d,alpha_g-obj-ns}
    \inf_{\theta\in\Theta} D_{f^\text{NS}_{\alpha_D,\alpha_G}}(P_r||P_{G_\theta})+\frac{\alpha_G}{\alpha_G-1}\left(1-2^{\frac{1}{\alpha_G}-1}\right),
\end{align}
which is minimized iff $P_{G_\theta}=P_r$ for $(\alpha_D,\alpha_G) \in (0,\infty]^2$ such that $\alpha_D + \alpha_G > \alpha_G\alpha_D.$
\end{theorem}

The proof mimics that of Theorem \ref{thm:alpha_D,alpha_G-GAN-saturating} and is detailed in Appendix \ref{appendix:alpha_D,alpha_G-GAN-nonsaturating}.  Figure \ref{fig:convexity_regions}(b) in Appendix \ref{appendix:alpha_D,alpha_G-GAN-nonsaturating} illustrates the feasible $(\alpha_D,\alpha_G)$-region; in contrast to the saturating setting of Theorem \ref{thm:alpha_D,alpha_G-GAN-saturating}, the NS setting constrains $\alpha\le 2$ when  $\alpha_D=\alpha_G=\alpha$. See Figure \ref{fig:(alpha_d,alpha_g)-GAN-sat-gen-loss}(c) for a toy example illustrating how tuning $\alpha_D<1$ and $\alpha_G\ge1$ can also alleviate training instabilities in the NS setting. 


We note that the input to the discriminator is a random variable $X$ which can be viewed as being sampled from a mixture distribution, i.e., $X\sim \delta P_r + (1-\delta)P_{G_\theta}$ where $\delta\in (0,1)$. Without loss of generality, we assume $\delta=1/2$ but the analysis that follows can be generalized for arbitrary $\delta$. We use the Bernoulli random variable $Y \in \{0,1\}$ to indicate that $X=x$ is from the real 
($Y=1$) or generated ($Y=0$) distributions. Therefore, the marginal probabilities of the two classes are $P_Y(1)=1-P_Y(0)=\delta=1/2$.
Thus, one can then compute the true posterior $P_{Y|X}(1|x)$ and its tilted version $P^{(\alpha_D)}_{Y|X}(1|x)$  as follows: 
\begin{equation}  
    \label{eq:true-posterior}
    P_{Y|X}(1|x)=\frac{ p_r(x)}{ p_r(x)+ p_{G_\theta}(x)} \quad \text{ and } \quad  P^{(\alpha_D)}_{Y|X}(1|x) = \frac{p_r(x)^{\alpha_D}}{ p_r(x)^{\alpha_D}+p_{G_\theta}(x)^{\alpha_D}},
\end{equation}
where both expressions simplify to the optimal discriminator of the vanilla GAN in \eqref{eq:opt-disc-vanilla} for $\alpha_D=1$. 

We now present a theorem to quantify precisely the effect of tuning $\alpha_D$ and $\alpha_G$. To this end, we begin by first taking a closer look at the gradients induced by the generator's loss during training. To simplify our analysis, we assume that at every step of training, the discriminator can achieve its optimum, $D_{\omega^*}$\footnote{We note that a related gradient analysis was considered by Shannon~\cite[Section~3.1]{shannon2020properties} for $f$-GANs assuming an optimal discriminator.}. For any sample $x=G_\theta(z)$ generated by G, we can write the gradient of the generator's loss for an $(\alpha_D,\alpha_G)$-GAN \textit{w.r.t.} its weight vector $\theta$ as
\begin{equation}
    -\frac{\partial \ell_{\alpha_{G}}\left(0, D_{\omega^{*}}(x)\right)}{\partial \theta} =-\frac{\partial \ell_{\alpha_{G}}\left(0, D_{\omega^{*}}(x)\right)}{\partial x}\times\frac{\partial x}{\partial \theta}  =  - \frac{\partial \ell_{\alpha_{G}}\left(0, D_{\omega^{*}}(x)\right)}{\partial D_{\omega^{*}}(x)} \times \frac{\partial D_{\omega^{*}}(x)}{\partial x} \times \frac{\partial x}{\partial \theta}.
    \label{eq:sat_grad}
\end{equation} 
We note that while we cannot explicitly analyze the term $\frac{\partial x}{\partial \theta}$ in \eqref{eq:sat_grad}, we assume that by using models satisfying boundedness and Lipschitz assumptions\footnote{These assumptions match practical settings.}, this term will not be unbounded. We thus focus on the first two terms on the right side of  \eqref{eq:sat_grad} for any $\alpha_G$.  For $\alpha_D=1$, from \eqref{eqn:optimaldisc-gen-alpha-GAN}, we see that in regions densely populated by the generated but not the real data, $D_{\omega^*}(x)\rightarrow 0$. Further, the first term in \eqref{eqn:optimaldisc-gen-alpha-GAN} is bounded thus causing the gradient in \eqref{eq:sat_grad} to vanish. On the other hand, when $\alpha_D<1$, $D_{\omega^*}$ increases (resp. decreases) in areas denser in generated (resp. real) data, thereby providing more gradients for G. This is clearly illustrated in Fig. \ref{fig:(alpha_d,alpha_g)-GAN-sat-gen-loss}(a) and \ref{fig:(alpha_d,alpha_g)-GAN-sat-gen-loss}(b) and reveals how strongly dependent the optimization trajectory traversed by G during training is on the practitioner’s choice of $(\alpha_{D}, \alpha_{G}) \in (0,\infty]^{2}$. In fact, this holds irrespective of the saturating or the NS $(\alpha_{D}, \alpha_{G})$-GAN. In the following theorem, we offer deeper insights into how such an optimization trajectory is influenced by tuning $\alpha_D$ and $\alpha_G$.  


\begin{theorem}
\label{thm:sat-gradient}
For a given $P_r$ and $P_{G_\theta}$, let $x$ be a sample generated according to $P_{G_{\theta}}$, and $D_{\omega^{*}}$ be optimal with respect to $V_{\alpha_{D}}(\theta, \omega)$. Then

    \begin{itemize}
    
        \item[\textbf{(a)}] the \textbf{saturating} and \textbf{non-saturating} gradients, $-\partial \ell_{\alpha_{G}}\left(0,D_{\omega^{*}}(x)\right) / \partial x$ and $\partial \ell_{\alpha_{G}}\left(1, D_{\omega^{*}}(x)\right) / \partial x$, respectively, demonstrate the following behavior: 
        \begin{align}
    -\frac{\partial \ell_{\alpha_{G}}\left(0, D_{\omega^{*}}(x)\right)}{\partial x} & = C_{x,\alpha_{D},\alpha_{G}} \left(\frac{1}{p_{G_{\theta}}(x)}\frac{\partial p_{G_{\theta}}}{\partial x} - \frac{1}{p_{r}(x)}\frac{\partial p_{r}}{\partial x}\right)\label{eqn:gradients64} \\
    \frac{\partial \ell_{\alpha_{G}}\left(1, D_{\omega^{*}}(x)\right)}{\partial x} & = C^{\text{NS}}_{x,\alpha_{D},\alpha_{G}} \left(\frac{1}{p_{G_{\theta}}(x)}\frac{\partial p_{G_{\theta}}}{\partial x} - \frac{1}{p_{r}(x)}\frac{\partial p_{r}}{\partial x}\right),\label{eqn:gradients65}
    \end{align}
    where using the tilted probability  $P^{(\alpha_D)}_{Y|X}(1|x) $ as written in \eqref{eq:true-posterior}, 
    \begin{align}
C_{x,\alpha_{D}, \alpha_{G}}& \coloneqq \alpha_{D} P^{(\alpha_D)}_{Y|X}(1|x) \left(1 - P^{(\alpha_D)}_{Y|X}(1|x) \right)^{1 - 1/\alpha_{G}}, \label{eq:c-sat} \\
C^{\text{NS}}_{x,\alpha_{D}, \alpha_{G}} & \coloneqq \alpha_{D} \left(1 - P^{(\alpha_D)}_{Y|X}(1|x)\right) P^{(\alpha_D)}_{Y|X}(1|x)^{1 - 1/\alpha_{G}},\ {\text{and}}
\label{eq:c-nonsat}
\end{align}    
  \item[\textbf{(b)}] the gradients in both \eqref{eqn:gradients64} and \eqref{eqn:gradients65} have directions that are independent of $\alpha_{D}$ and $\alpha_{G}$.
    \end{itemize}
\end{theorem}
\begin{remark}
    One can view the results in Theorem \ref{thm:sat-gradient} above as a one-shot (in any iteration) analysis of the gradients of the generator's loss, and thus, we fix $P_{G_\theta}$. Doing so allows us to ignore the implicit dependence on $(\alpha_D,\alpha_G)$ of the $P_{G_\theta}$ learned up to this iteration, thus allowing us to obtain tractable expressions for any iteration.
\end{remark}

A detailed proof of Theorem \ref{thm:sat-gradient} can be found in Appendix \ref{appendix:sat-gradient-proof}. 
Focusing first on saturating $(\alpha_{D}, \alpha_{G})$-GANs, in Fig. \ref{fig:p-real-vs-c}(a) in Appendix \ref{appendix:sat-gradient-proof}, we plot $C_{x,\alpha_{D}, \alpha_{G}}$ as a function of the true probability that $X \sim \frac{1}{2}P_{r} + \frac{1}{2} P_{G_{\theta}}$ is real, namely $P_{Y|X}(1|x)$, for five $(\alpha_{D}, \alpha_{G})$ combinations. In the $(1,1)$ case (i.e., vanilla GAN), $C_{x,1, 1} \approx 0$ for generated samples far away from the real data (where $P_{Y|X}(1|x) \approx 0$).
As discussed earlier, this optimization strategy is troublesome 
when the real and generated data are fully separable, since the sample gradients are essentially zeroed out by the scalar, leading to vanishing gradients. To address this issue, Fig. \ref{fig:p-real-vs-c}(a) shows that tuning $\alpha_{D}$ below 1 (e.g., 0.6) ensures that samples most likely to be ``generated'' ($P_{Y|X}(1|x) \approx 0$) receive sufficient gradient for updates that direct them closer to the real distribution.


The vanilla GAN also suffers from convergence issues since generated samples close to the real data (when $P_{Y|X}(1|x) \approx 1$) 
receive gradients large in magnitude ($C_{x,1,1} \approx 1$). Ideally, these generated samples should not be instructed to move since they convincingly pass as real to $D_{\omega^*}$. As explained in Section \ref{subsec:gan-training-instabilities}, an excessive gradient can push the generated data away from the real data, which ultimately separates the distributions and forces the GAN to restart training. Although the $(0.6,1)$-GAN in Fig. \ref{fig:p-real-vs-c}(a) appears to decrease $C_{x,\alpha_{D}, \alpha_{G}}$ for samples close to the real data ($P_{Y|X}(1|x) \approx 1$), tuning $\alpha_{G}>1$ allows this gradient to converge to zero as desired (see Fig. \ref{fig:(alpha_d,alpha_g)-GAN-sat-gen-loss}(b)). 

Although tuning the saturating $(\alpha_{D}, \alpha_{G})$-GAN formulation away from vanilla GAN promotes a more favorable optimization trajectory for G, this approach continues to suffer from the problem of providing small gradients for generated samples far from $P_r$.  
This suggests looking at the behavior of the NS $(\alpha_D,\alpha_G)$-GAN formulation. 
Figure \ref{fig:p-real-vs-c}(b) in Appendix \ref{appendix:sat-gradient-proof} illustrates the relationship between the gradient scalar $C^{\text{NS}}_{x,\alpha_{D},\alpha_{G}}$ and the probability that a sample $X \sim \frac{1}{2}P_{r} + \frac{1}{2}P_{G_{\theta}}$ is real, namely $P_{Y|X}(1|x)$, for several values of $(\alpha_{D},\alpha_{G})$. In the vanilla $(1,1)$-GAN case, we observe a negative linear relationship, i.e., the samples least likely to be real ($P_{Y|X}(1|x) \approx 0$) receive large gradients ($C^{\text{NS}}_{x,1,1} \approx 1$) while the samples most likely to be real receive minimal gradients ($C^{\text{NS}}_{x,1,1} \approx 0$). While this seems desirable, unfortunately, the vanilla GAN's optimization strategy often renders it vulnerable to model oscillation, a common GAN failure detailed in Section \ref{subsec:gan-training-instabilities}, as a result of such large gradients of the outlier (far from real) samples causing the generated data to oscillate around the real data modes.
By tuning $\alpha_{D}$ below 1, as shown in Fig. \ref{fig:p-real-vs-c}(b), one can slightly increase (resp. decrease) $C^{\text{NS}}_{x,\alpha_{D}, \alpha_{G}}$ for the generated samples close to (resp. far from) the real modes.
As a result, the generated samples are more robust to outliers and therefore more likely to converge to the real modes. Finally, tuning $\alpha_{G}$ above 1 can further improve this robustness. 
A caveat here is the fact that $C^{\text{NS}}_{x,\alpha_{D}, \alpha_{G}} \approx 0$ when $P_{Y|X}(1|x)\approx 0$ can potentially be problematic since the near-zero gradients may immobilize generated data far from the real distribution. This is borne out in our results for several large image datasets in Section \ref{sec:experimental-results} where choosing $\alpha_G=1$ yields the best results. The cumulative effects of tuning $(\alpha_D,\alpha_G)$ are further illustrated in Fig. \ref{fig:(alpha_d,alpha_g)-GAN-sat-gen-loss}(c).

\subsection{CPE Loss Based Dual-objective GANs}
Similarly to the single-objective loss function perspective in Section~\ref{sec:loss-function-perspective-single-objective}, we can generalize the $(\alpha_D,\alpha_G)$-GAN formulation to incorporate general CPE losses. To this end, we introduce a dual-objective loss function perspective of GANs
in which the discriminator maximizes $V_{\ell_D}(\theta,\omega)$ while the generator minimizes $V_{\ell_G}(\theta,\omega)$, where
\medmuskip=0mu
\begin{align}
    &V_{\ell}(\theta,\omega)=\mathbb{E}_{X\sim P_r}[-\ell(1,D_\omega(X))]+\mathbb{E}_{X\sim P_{G_\theta}}[-\ell(0,D_\omega(X))]\label{eqn:cpe-objective},
\end{align}
for any CPE losses $\ell=\ell_D,\ell_G$.
The resulting CPE loss dual-objective GAN is given by
\begin{subequations}
\begin{align} 
&\sup_{\omega\in\Omega}V_{\ell_D}(\theta,\omega) \label{eqn:cpe-disc_obj} \\
& \inf_{\theta\in\Theta} V_{\ell_G}(\theta,\omega)
\label{eqn:cpe-gen_obj}.
\end{align}
\label{eqn:cpe-loss-dual-obj-GAN}
\end{subequations}
The CPE losses $\ell_D$ and $\ell_G$ can be completely different losses, the same loss but with different parameter values, or the same loss with the same parameter values, in which case the above formulation reduces to the single-objective formulation in \eqref{eqn:lossfnbasedGAN}. For example, choosing $\ell_D = \ell_G = \ell_\alpha$, we recover the $\alpha$-GAN formulation in \eqref{eqn:minimaxalphaGAN}; choosing $\ell_D = \ell_{\alpha_D}$ and $\ell_G = \ell_{\alpha_G}$, we obtain the $(\alpha_D,\alpha_G)$-GAN formulation in \eqref{eqn:alpha_D,alpha_G-GAN}. Note that $\ell_D$ should satisfy the constraint in \eqref{eqn:condnonfnsforGAN} so that the optimal discriminator outputs ${1}/{2}$ for any input when $P_r=P_{G_\theta}$. We once again maintain the same ordering as the original min-max GAN formulation and present the conditions under which the optimal generator minimizes a symmetric $f$-divergence when the discriminator set $\Omega$ is large enough in the following proposition.

{
\begin{proposition}
\label{prop:dual-objective-CPE-loss-GAN-strategies}
Let $\ell_D$ and $\ell_G$ be symmetric CPE loss functions with $\ell_D(1,\cdot)$ also differentiable with derivative $\ell_D^\prime(1,\cdot)$ and strictly convex. Then the optimal discriminator $D_{\omega^*}$ optimizing \eqref{eqn:cpe-disc_obj} satisfies the implicit equation, provided it has a solution,
\begin{equation}
    \ell_D^\prime(1,1-D_{\omega^*}(x)) = \frac{p_r(x)}{p_{G_\theta}(x)}\ell_D^\prime(1,D_{\omega^*}(x)), \quad x \in \mathcal{X}.
    \label{eq:opt-disc-cpe-implicit-eq}
\end{equation}
If \eqref{eq:opt-disc-cpe-implicit-eq} does not have a solution for a particular $x\in\mathcal{X}$, then $D_{\omega^*}(x)=0$ or $D_{\omega^*}(x)=1$. Let $A\left(\frac{p_r(x)}{p_{G_\theta}(x)}\right) \coloneqq D_{\omega^*}(x)$. For this $D_{\omega^*}$, \eqref{eqn:cpe-gen_obj} simplifies to minimizing a symmetric $f$-divergence $D_f(P_r||P_{G_\theta})$ if the function $f:\mathbb R_+ \to \mathbb R$ is convex, where $f$ is defined as
\begin{equation}
    f(u) = -u\ell_G(1,A(u))-\ell_G(1,1-A(u))+2\ell_G(1,1/2)
    \label{eq:dual-obj-cpe-f}.
\end{equation}
\end{proposition}}

\begin{proof}[Proof sketch]\let\qed\relax
The proof involves a straightforward application of KKT conditions when optimizing \eqref{eqn:cpe-disc_obj} and substituting in  \eqref{eqn:cpe-gen_obj}. A detailed proof can be found in Appendix \ref{appendix:dual-objective-CPE-loss-GAN-strategies}.
\end{proof}

As it is difficult to come up with conditions without having the explicit forms of the losses $\ell_D$ and $\ell_G$, Proposition \ref{prop:dual-objective-CPE-loss-GAN-strategies} provides a broad outline of what the optimal strategies will look like. The assumption of the losses being symmetric can be relaxed, in which case the resulting $f$-divergence will no longer be guaranteed to be symmetric. Theorem \ref{thm:alpha_D,alpha_G-GAN-saturating} is a special case of Proposition \ref{prop:dual-objective-CPE-loss-GAN-strategies} for $\ell_D=\ell_{\alpha_D}$ and $\ell_G=\ell_{\alpha_G}$. As another example, consider the following square loss based CPE losses \footnote{Note that these losses were considered in \cite{veiner2023unifying} and were shown to result in a special case of a shifted LSGAN minimizing a certain Jensen-$f$-divergence.}:
\begin{align}
    \ell_D(y,\hat{y})&=\frac{1}{2}\left[y(\hat{y}-1)^2+(1-y)\hat{y}^2\right] \label{eq:square-loss-disc}\\  \ell_G(y,\hat{y})&=-\frac{1}{2}\left[y(\hat{y}^2-1)+(1-y)\left((1-\hat{y})^2-1\right)\right] \label{eq:square-loss-gen}.
\end{align}
Note that \eqref{eq:square-loss-disc} and \eqref{eq:square-loss-gen} are both symmetric and $\ell_D(1,\cdot)$ is both convex (and therefore $\ell_D$ satisfies \eqref{eqn:condnonfnsforGAN}) and differentiable with $\ell_D^\prime(1,\hat{y})=\hat{y}-1$. The implicit equation in \eqref{eq:opt-disc-cpe-implicit-eq} then becomes 
\[(1-D_{\omega^*}(x))-1=u(D_{\omega^*}(x)-1), \quad \text{where} \quad D_{\omega^*}
(x) = \frac{u}{u+1}=\frac{p_r(x)}{p_r(x) + p_{G_\theta}(x)}.\]
The corresponding $f$ in \eqref{eq:dual-obj-cpe-f} is 
$f(u)=[3(1-u)]/[4(u+1)]$,
which is convex. Therefore, the dual-objective CPE loss GAN using \eqref{eq:square-loss-disc} and \eqref{eq:square-loss-gen} minimizes a symmetric $f$-divergence.
\subsection{Estimation Error for CPE Loss Dual-objective GANs}
\label{subsec:est-err}
In order to analyze what occurs in practice when both the number of training samples and model capacity are usually limited, we now consider the same setting as in Section \ref{subsec:est-and-gen-error-single-obj} with finite training samples $S_x=\{X_1,\dots,X_n\}$ and $S_z=\{Z_1,\dots,Z_m\}$ from $P_r$ and $P_Z$, respectively, and with neural networks chosen as the discriminator and generator models. The sets of samples $S_x$ and $S_z$ induce the empirical real and generated distributions $\hat{P}_r$ and $\hat{P}_{G_\theta}$, respectively. A useful quantity to evaluate the performance of GANs in this setting is again that of the estimation error. In Section \ref{subsec:est-and-gen-error-single-obj}, we define estimation error for CPE loss GANs. However, such a definition requires a common value function for both discriminator and generator, and therefore, does not directly apply to the dual-objective setting we consider here.

Our definition relies on the observation that estimation error inherently captures the effectiveness of the generator (for a corresponding optimal discriminator model) in learning with limited samples. We formalize this intuition below.

Since CPE loss dual-objective GANs use different objective functions for the discriminator and generator, we start by defining the optimal discriminator ${\omega}^*$ for a generator model $G_\theta$ as
\begin{align}
    {\omega}^*(P_r,P_{G_\theta}) \coloneqq \argmax_{\omega \in \Omega} \; V_{\ell_D}(\theta,\omega)\big\rvert_{P_r,P_{G_\theta}},
    \label{eq:est-err-opt-disc}
\end{align}
where the notation $|_{\cdot,\cdot}$ allows us to make explicit the distributions used in the value function. 
In keeping with the literature where the value function being minimized is referred to as the neural net (NN) distance (since D and G are modeled as neural networks) \cite{AroraGLMZ17,JiZL21,kurri-2022-convergence}, we define the generator's NN distance $d_{\omega^*(P_r,P_{G_\theta})}$ as
\begin{align}
    d_{\omega^*(P_r,P_{G_\theta})}(P_r,{P}_{G_{{\theta}}}) \coloneqq V_{\ell_G}(\theta,\omega^*(P_r,P_{G_\theta}))\big\rvert_{P_r,P_{G_\theta}}.
    \label{eq:est-err-gen-obj}
\end{align}
The resulting minimization for training the CPE-loss dual-objective GAN using finite samples is
\begin{align}
    \inf_{\theta\in\Theta} d_{\omega^*(\hat{P}_r,\hat{P}_{G_\theta})}(\hat{P}_r,\hat{P}_{G_{{\theta}}}).
    \label{eq:training-empirical-alpha_d,alpha_g-GAN}
\end{align}
Denoting $\hat{\theta}^*$ as the minimizer of \eqref{eq:training-empirical-alpha_d,alpha_g-GAN}, we define the estimation error for CPE loss dual-objective GANs as
\begin{align}
    d_{\omega^*(P_r,P_{G_{\hat{\theta}^*}})}(P_r,{P}_{G_{\hat{\theta}^*}})-\inf_{\theta\in\Theta} d_{\omega^*({P}_r,{P}_{G_\theta})}(P_r,P_{G_{\theta}})
    \label{eq:est-error-def-alpha_d,alpha_g-GAN}.
\end{align}


We use the same notation as in Section \ref{subsec:est-and-gen-error-single-obj}, detailed again in the following for easy reference. \textcolor{red}{}For $x\in\mathcal{X}\coloneqq\{x\in\mathbb{R}^d:||x||_2\leq B_x\}$ and  $z\in\mathcal{Z}\coloneqq\{z\in\mathbb{R}^p:||z||_2\leq B_z\}$, we model the discriminator and generator as $k$- and $l$-layer neural networks, respectively, such that $D_\omega$ and $G_\theta$ can be written as:
\begin{align}
    D_\omega&:x\mapsto \sigma\left(\mathbf{w}_k^\mathsf{T}r_{k-1}(\mathbf{W}_{k-1}r_{k-2}(\dots r_1(\mathbf{W}_1(x)))\right)\,  \label{eqn:disc-model-dual}\\
    G_\theta&:z\mapsto \mathbf{V}_ls_{l-1}(\mathbf{V}_{l-1}s_{l-2}(\dots s_1(\mathbf{V}_1z))),
\end{align}
where (i) $\mathbf{w}_k$ is a parameter vector of the output layer; (ii) for $i\in[1:k-1]$ and $j\in[1:l]$, $\mathbf{W}_i$ and $\mathbf{V}_j$ are parameter matrices; (iii) $r_i(\cdot)$ and $s_j(\cdot)$ are entry-wise activation functions of layers $i$ and $j$, respectively, i.e., for $\mathbf{a}\in\mathbb{R}^t$, $r_i(\mathbf{a})=\left[r_i(a_1),\dots,r_i(a_t)\right]$ and $s_i(\mathbf{a})=\left[s_i(a_1),\dots,s_i(a_t)\right]$; and (iv) $\sigma(\cdot)$ is the sigmoid function given by $\sigma(p)=1/(1+\mathrm{e}^{-p})$. We assume that each $r_i(\cdot)$ and $s_j(\cdot)$ are $R_i$- and $S_j$-Lipschitz, respectively, and also that they are positive homogeneous, i.e., $r_i(\lambda p)=\lambda r_i(p)$ and $s_j(\lambda p)=\lambda s_j(p)$, for any $\lambda\geq 0$ and $p\in\mathbb{R}$. Finally, as is common in such analysis \cite{neyshabur2015norm,salimans2016weight,golowich2018size,JiZL21}, 
we assume that the Frobenius norms of the parameter matrices are bounded, i.e., $||\mathbf{W}_i||_F\leq M_i$, $i\in[1:k-1]$, $||\mathbf{w}_k||_2\leq M_k$, and $||\mathbf{V}_j||_F\leq N_j$, $j\in[1:l]$. We now present an upper bound on \eqref{eq:est-error-def-alpha_d,alpha_g-GAN} in the following theorem.

\begin{theorem}\label{thm:estimationerror-upperbound-double-objective}
For the setting described above, additionally assume that the functions $\phi(\cdot) \coloneqq -\ell_G(1,\cdot)$ and $\psi(\cdot) \coloneqq -\ell_G(0,\cdot)$ are $L_\phi$- and $L_\psi$-Lipschitz, respectively.
Then, with probability at least $1-2\delta$ over the randomness of training samples $S_x=\{X_i\}_{i=1}^n$ and $S_z=\{Z_j\}_{j=1}^m$, we have
\begin{align}
    d_{\omega^*(P_r,P_{G_{\hat{\theta}^*}})}(P_r,{P}_{G_{\hat{\theta}^*}})-\inf_{\theta\in\Theta} d_{\omega^*({P}_r,{P}_{G_\theta})}(P_r,P_{G_{\theta}})
    \leq & \frac{L_\phi B_xU_\omega\sqrt{3k}}{\sqrt{n}}+\frac{L_\psi U_\omega U_\theta B_z\sqrt{3(k+l-1)}}{\sqrt{m}}\nonumber\\
    &\hspace{12pt}+U_\omega\sqrt{\log{\frac{1}{\delta}}}\left(\frac{L_\phi B_x}{\sqrt{2n}}+\frac{L_\psi B_zU_\theta}{\sqrt{2m}}\right), \label{eq:estimationbound-dual-objective}
\end{align}
where $U_\omega\coloneqq M_k\prod_{i=1}^{k-1}(M_iR_i)$ and $U_\theta\coloneqq N_l\prod_{j=1}^{l-1}(N_jS_j)$.

In particular, when specialized to the case of $(\alpha_D,\alpha_G)$-GANs by letting $\phi(p)=\psi(1-p)=\frac{\alpha_G}{\alpha_G-1}\left(1-p^{\frac{\alpha_G-1}{\alpha_G}}\right)$, the resulting bound is nearly identical to the terms in the RHS of~\eqref{eq:estimationbound-dual-objective}, except for substitutions $L_\phi \leftarrow 4C_{Q_x}(\alpha_G)$ and $L_\psi \leftarrow 4C_{Q_z}(\alpha_G)$, where $Q_x\coloneqq U_\omega B_x$, $Q_z\coloneqq U_\omega U_\theta B_z$, and
\begin{align} \label{eq:clipalpha}
    C_h(\alpha)\coloneqq\begin{cases}\sigma(h)\sigma(-h)^{\frac{\alpha-1}{\alpha}}, \ &\alpha\in(0,1]\\
    \left(\frac{\alpha-1}{2\alpha-1}\right)^{\frac{\alpha-1}{\alpha}}\frac{\alpha}{2\alpha-1}, &\alpha\in(1,\infty).
    \end{cases}
\end{align}
\end{theorem}

The proof is similar to that of Theorem \ref{thm:estimationerror-upperbound} (and also \cite[Theorem 1]{JiZL21}). We observe that  
\eqref{eq:estimationbound-dual-objective} does not depend on $\ell_D$, an artifact of the proof techniques used, and is therefore most likely not the tightest bound possible. See Appendix \ref{appendix:estimationerror-upperbound-alpha_d,alpha_g-GAN} for proof details.




\section{Illustration of Results}
\label{sec:experimental-results}
\begin{figure*}[t]
    \centering
    \footnotesize
\setlength{\tabcolsep}{1pt}
\begin{tabular}{@{}cc@{}}
  \includegraphics[height=4.3cm]{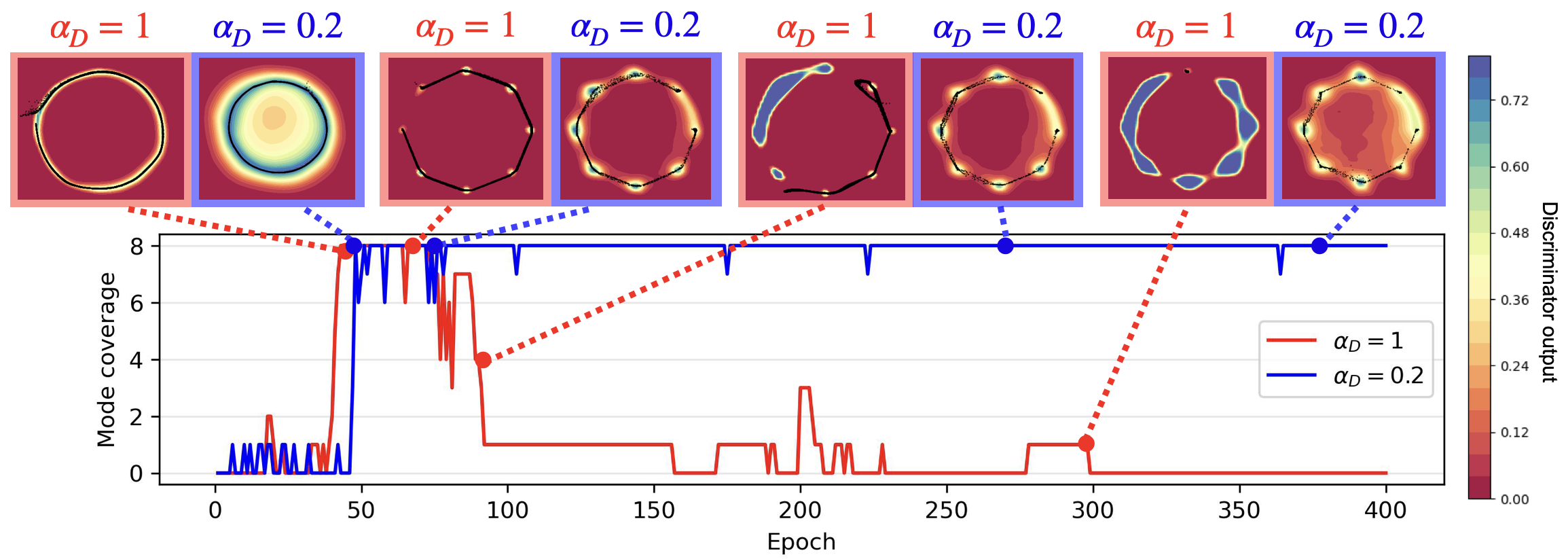}  & {\includegraphics[height=4cm]{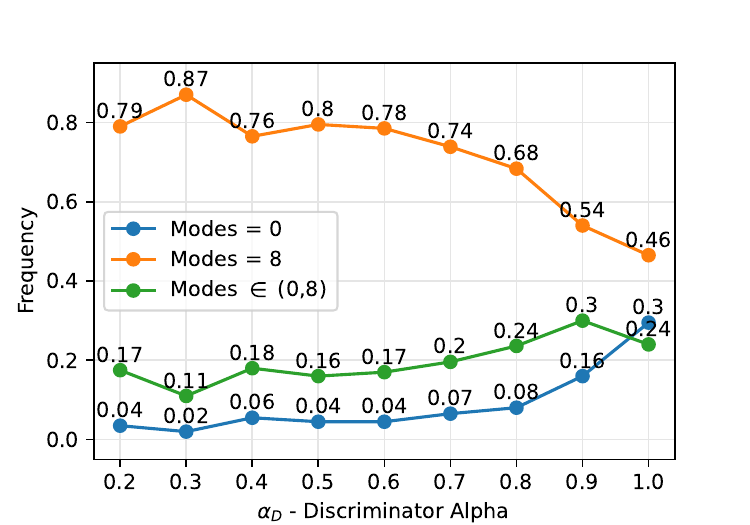}} \\
   (a)  & (b)
\end{tabular}
    \caption{(a) Plot of mode coverage over epochs for $(\alpha_{D}, \alpha_{G})$-GAN training with the \textbf{saturating} objectives in \eqref{eqn:alpha_D,alpha_G-GAN}. Fixing $\alpha_{G}=1$, we compare $\alpha_{D} = 1$ (vanilla GAN) with $\alpha_{D} = 0.2$. Placed above this plot are 2D visuals of the generated samples (in black) at different epochs; these show that both GANs successfully capture the ring-like structure, but the vanilla GAN fails to maintain the ring over time. We illustrate the discriminator output in the same visual as a heat map to show that the $\alpha_{D} = 1$ discriminator exhibits more confident predictions (tending to 0 or 1), which in turn subjects G to vanishing 
    and exploding gradients 
    when its objective $\log(1-D)$ saturates as $D\rightarrow 0$ and diverges as $D\rightarrow 1$, respectively. This combination tends to repel the generated data when it approaches the real data, thus freezing any significant weight update in the future. In contrast, the less confident predictions of the $(0.2,1)$-GAN create a smooth landscape for the generated output to descend towards the real data. (b) Plot of success and failure rates over 200 seeds vs. $\alpha_{D}$ with $\alpha_{G} = 1$ for the \textbf{saturating} $(\alpha_{D}, \alpha_{G})$-GAN on the 2D-ring, which underscores the stability of $(\alpha_{D} < 1,\alpha_G)$-GANs relative to vanilla GAN.
    }

    \label{fig:sat-figure}
\end{figure*}

We illustrate the value of $(\alpha_{D}, \alpha_{G})$-GAN as compared to the vanilla GAN (i.e., the $(1,1)$-GAN). {Focusing on DCGAN architectures \cite{radford2015}, we compare against LSGANs \cite{Mao_2017_LSGAN}, the current state-of-the-art (SOTA) dual-objective approach}. While WGANs \cite{ArjovskyCB17} have also been proposed to address the training instabilities, their training methodology is distinctly different and uses a different optimizer (RMSprop), requires gradient clipping or penalty, and does not leverage batch normalization, all of which make meaningful comparisons difficult. 

We evaluate our approach on three datasets: (i) a synthetic dataset generated by a two-dimensional, ring-shaped Gaussian mixture distribution (2D-ring) \cite{srivastava2017veegan}; (ii) the $64 \times 64$ Celeb-A image dataset \cite{liu2015}; and (iii) the $112 \times 112$ LSUN Classroom dataset~\cite{yu2015}. For each dataset and pair of GAN objectives, we report several metrics that encapsulate the stability of GAN training over hundreds of random seeds. This allows us to clearly 
showcase the potential for tuning $(\alpha_{D}, \alpha_{G})$ to obtain stable and robust solutions for image generation.
\vspace{-5pt}
\subsection{2D Gaussian Mixture Ring}

The 2D-ring is an oft-used synthetic dataset for evaluating GANs. We draw samples from a mixture of 8 equal-prior Gaussian distributions, indexed $i \in \{1,2,\hdots ,  8 \}$, with a mean of $(\cos(2\pi i / 8), \text{ } \sin(2\pi i / 8))$ and variance $10^{-4}$. We generate 50,000 training and 25,000 testing samples and the same number of 2D latent Gaussian noise vectors, where each entry is a standard Gaussian.

Both the D and G networks have 4 fully-connected layers with 200 and 400 units, respectively. 
We train for 400 epochs with a batch size of 128, and optimize with Adam \cite{kingma2014adam} and a learning rate of $10^{-4}$ for both models. We consider three distinct settings that differ in the objective functions as: \textbf{(i)} $(\alpha_{D}, \alpha_{G})$-GAN in \eqref{eqn:alpha_D,alpha_G-GAN}; \textbf{(ii)} NS $(\alpha_{D}, \alpha_{G})$-GAN's 
in \eqref{eqn:disc_obj}, \eqref{eqn:gen_obj_ns}; \textbf{(iii)} LSGAN with the 0-1 binary coding scheme (see Appendix \ref{appendix:experimental-details-results} for details).

For every setting listed above, we train our models on the 2D-ring dataset for 200 random state seeds, where each seed contains different weight initializations for D and G. Ideally, a stable method will reflect similar performance across randomized initializations and also over training epochs; thus, we explore how GAN training performance for each setting varies across seeds and epochs. Our primary performance metric is \textit{mode coverage}, defined as the number of Gaussians (0-8) that contain a generated sample within 3 standard deviations of its mean. A score of 8 conveys successful training, while a score of 0 conveys a significant GAN failure; on the other hand, a score in between 0 and 8 may be indicative of common GAN issues, such as mode collapse or failure to converge. 

For the saturating setting, the improvement in stability of the $(0.2,1)$-GAN relative to the vanilla GAN is illustrated in Figure \ref{fig:sat-figure} as detailed in the caption. 
Vanilla GAN fails to converge to the true distribution 30\% of the time while succeeding only 46\% of the time. In contrast, the $(\alpha_{D}, \alpha_{G})$-GAN with $\alpha_{D} < 1$ learns a more stable G due to a less confident D (see also Figure~\ref{fig:sat-figure}(a)). For example, the $(0.3,1)$-GAN success and failure rates improve to 87\% and 2\%, respectively. 
For the NS setting in Figure \ref{fig:NS}, we find that tuning $\alpha_D$ and $\alpha_G$ yields more consistently stable outcomes than vanilla and LSGANs. Mode coverage rates over 200 seeds for saturating (Tables \ref{table:2d-ring-sat-success-rates} and \ref{table:2d-ring-sat-failure-rates}) and NS (Table \ref{table:2d-ring-ns-success-rates}) are in Appendix \ref{appendix:experimental-details-results}.

\begin{figure}[t]
    \centering
    \includegraphics[width=0.5\linewidth]{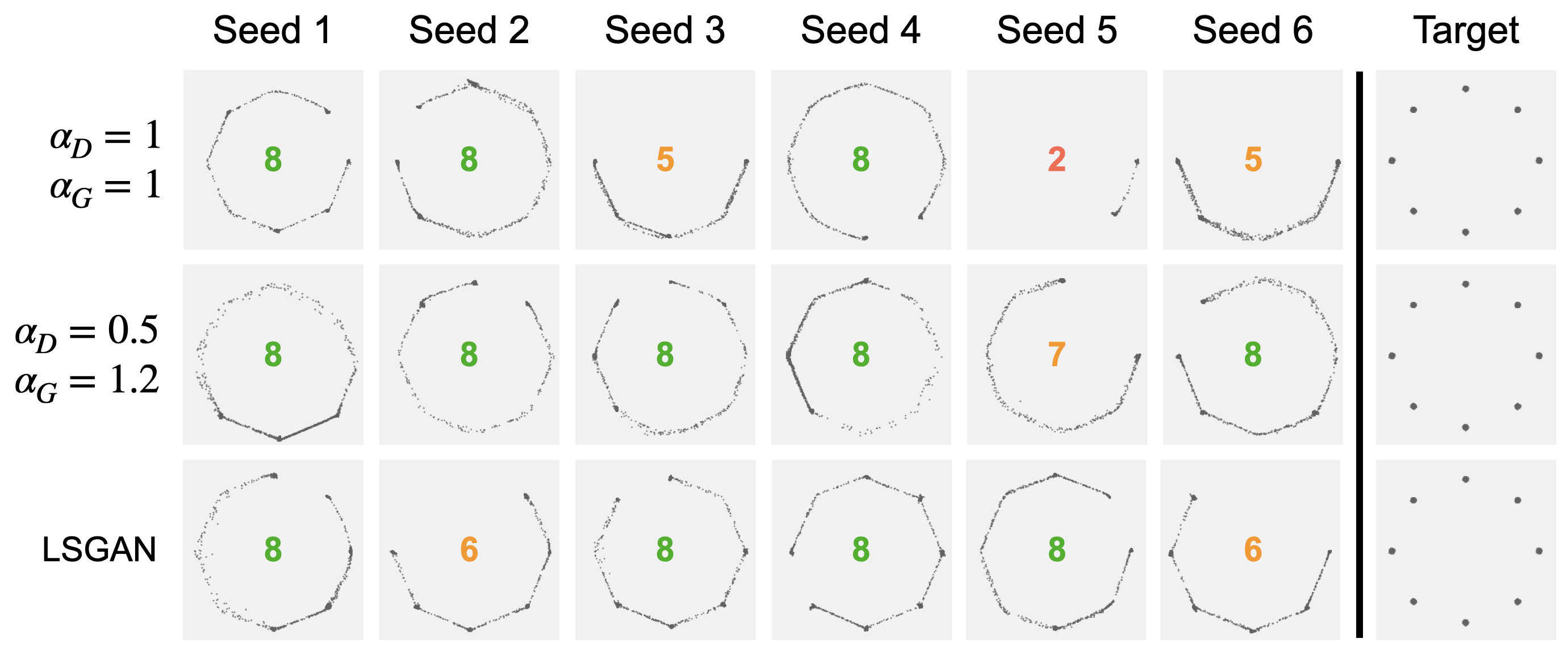}
    \caption{Generated samples from two $(\alpha_{D}, \alpha_{G})$-GANs trained with the \textbf{NS} objectives in \eqref{eqn:disc_obj}, \eqref{eqn:gen_obj_ns}, as well as LSGAN. We provide 6 seeds to illustrate the stability in performance for each GAN across multiple runs.}
    \label{fig:NS}
    \vspace{-0.075in}
\end{figure}

\vspace{-0.075in}

\subsection{Celeb-A \& LSUN Classroom}
\label{subsec:celeb-a&LSUN-experiments}
The Celeb-A dataset \cite{liu2015} is a widely recognized large-scale collection of over 200,000 celebrity headshots, encompassing images with diverse aspect ratios, camera angles, backgrounds, lighting conditions, and other variations. Similarly, the LSUN Classroom dataset \cite{yu2015} is a subset of the comprehensive Large-scale Scene Understanding (LSUN) dataset; it contains over 150,000 classroom images captured under diverse conditions and with varying aspect ratios. To ensure consistent input for the discriminator, we follow the standard practice of resizing the images to $64 \times 64$ for Celeb-A and $112 \times 112$ for LSUN Classroom. For both experiments, we randomly select 80\% of the images for training and leave the remaining 20\% for validation (evaluation of goodness metrics). Finally, for the generator, for each dataset, we generate a similar 80\%-20\% training-validation split of 100-dimensional latent Gaussian noise vectors, where each entry is a standard Gaussian, for a total matching the size of the true data. 

For training, we employ the DCGAN architecture \cite{radford2015} that leverages deep convolutional neural networks (CNNs) for both D and G. In Appendix \ref{appendix:experimental-details-results}, detailed descriptions of the D and G architectures can be found in Tables \ref{tab:arch-celeba} and \ref{tab:arch-lsun} for the Celeb-A and LSUN Classroom datasets, respectively. Following SOTA methods, we focus on the non-saturating setting, utilizing appropriate objectives for vanilla GAN, $(\alpha_{D}, \alpha_{G})$-GAN, and LSGAN. We consider a variety of learning rates, ranging from $10^{-4}$ to $10^{-3}$, for Adam optimization. We evaluate our models every 10 epochs up to a total of 100 epochs and report the Fréchet Inception Distance (FID), an unsupervised similarity metric between the real and generated feature distributions extracted by InceptionNet-V3~\cite{heusel2017fid}. For both datasets, we train each combination of objective function, number of epochs, and learning rate for 50 seeds. In the following subsections, we empirically demonstrate the dependence of the FID on learning rate and number of epochs for the vanilla GAN, $(\alpha_{D}, \alpha_{G})$-GAN, and LSGAN. Achieving robustness to hyperparameter initialization is especially desirable in the unsupervised GAN setting as the choices that facilitate steady model convergence are not easily determined \emph{a priori}. 


\begin{figure}[t]
    \centering
    \footnotesize
\setlength{\tabcolsep}{20pt}
\begin{tabular}{@{}cc@{}}
  \includegraphics[width=0.4\linewidth]{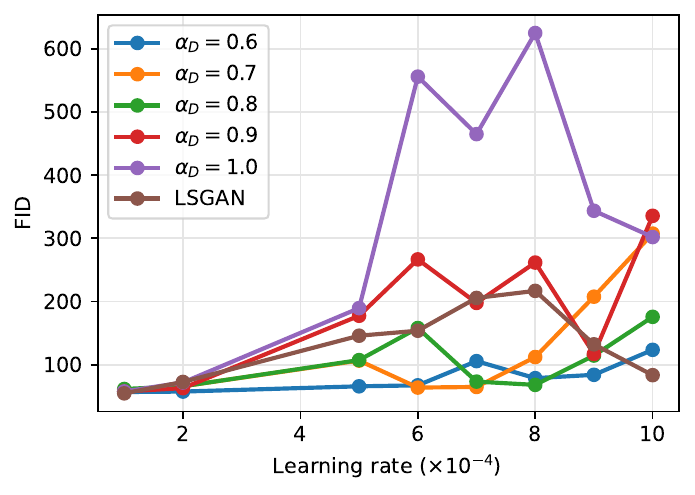}  & \includegraphics[width=0.4\linewidth]{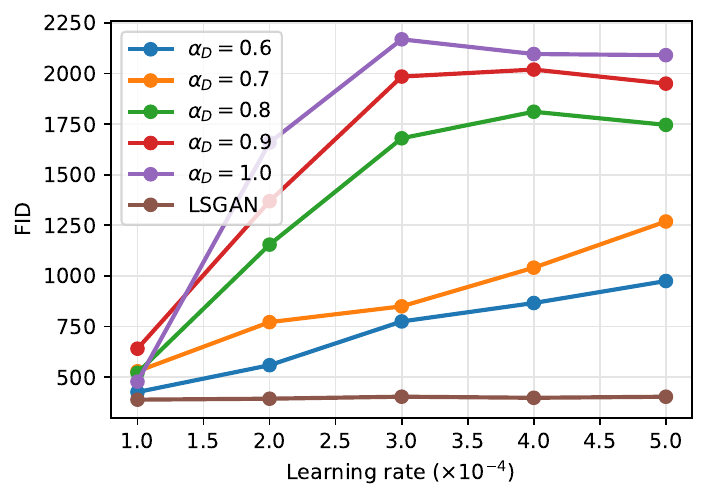} \\
   (a)  & (b)
\end{tabular}
\caption{(a) Plot of \textbf{Celeb-A} FID scores averaged over 50 seeds vs. learning rates for 6 different GANs, trained for 100 epochs. (b) Plot of \textbf{LSUN Classroom} FID scores averaged over 50 seeds vs. learning rates for 6 different GANs, trained for 100 epochs. } 
\label{fig:celeba-lsun-fid-lr}
\end{figure}

\begin{figure*}[t]
    \centering
    \footnotesize
\setlength{\tabcolsep}{5pt}
\begin{tabular}{@{}cc@{}}
  \includegraphics[height=3.6cm]{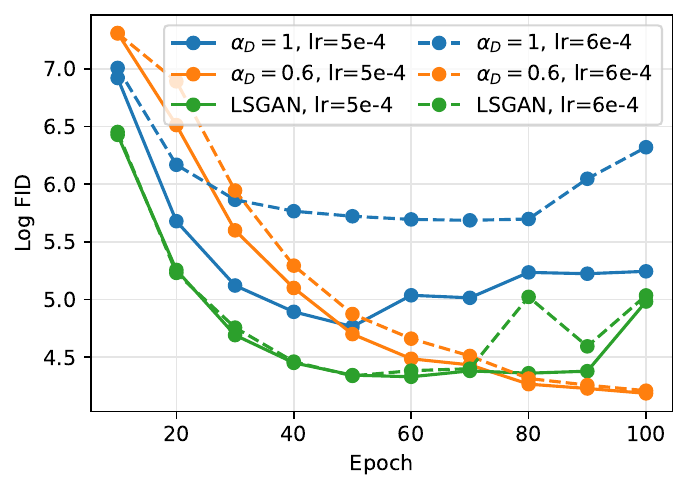}  & \includegraphics[height=3.6cm]{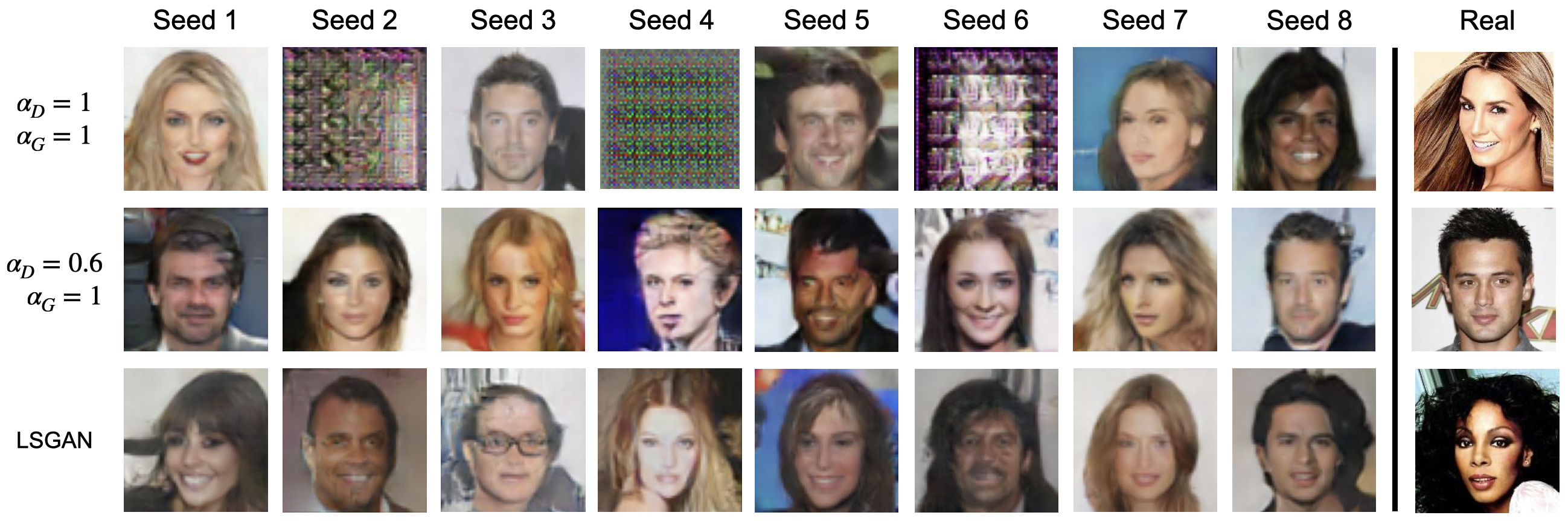} \\
   (a)  & (b)
\end{tabular}
\caption{(a) Log-scale plot of \textbf{Celeb-A} FID scores over training epochs in steps of 10 up to 100 total, for three noteworthy GANs-- $(1,1)$-GAN (vanilla), $(0.6,1)$-GAN, and LSGAN-- and for two similar learning rates-- $5 \times 10^{-4}$ and $6 \times 10^{-4}$. Results show that the vanilla GAN performance is sensitive to learning rate choice, while the other two GANs achieve consistently low FIDs. (b) Generated Celeb-A faces from the same three GANs over 8 seeds when trained for 100 epochs with a learning rate of $5 \times 10^{-4}$. These samples show that the vanilla $(1,1)$-GAN training is sensitive to random model weight initializations, while the other two GANs demonstrate both robustness to random weight initializations as well as realistic face generation.}
\label{fig:celeba-epochs-fid-faces}
\end{figure*}

\subsubsection{Celeb-A Results}

In Figure~\ref{fig:celeba-lsun-fid-lr}(a), we examine the relationship between learning rate and FID for each GAN trained for 100 epochs on the Celeb-A dataset. When using learning rates of $1 \times 10^{-4}$ and $2 \times 10^{-4}$, all GANs consistently perform well. However, when the learning rate increases,
 the vanilla $(1,1)$-GAN begins to exhibit instability across the 50 seeds. As the learning rate surpasses $5 \times 10^{-4}$, the performance of the vanilla GAN becomes even more erratic, underscoring the importance of GANs being robust to the choice of learning rate. Figure \ref{fig:celeba-lsun-fid-lr}(a) also demonstrates that the GANs with $\alpha_{D} < 1$ perform on par with, if not better than, the SOTA LSGAN. For instance, the $(0.6,1)$-GAN consistently achieves low FIDs across all tested learning rates. 


In Figure \ref{fig:celeba-epochs-fid-faces}(a), for different learning rates, we compare the dependence on the number of training epochs (hyperparameter)  of the vanilla $(1,1)$-GAN, $(0.6,1)$-GAN, and LSGAN by plotting their FIDs every 10 epochs, up to 100 epochs, for two similar learning rates: $5 \times 10^{-4}$ and $6 \times 10^{-4}$. We discover that the vanilla $(1,1)$-GAN performs significantly worse for the higher learning rate and deteriorates over time for both learning rates. Conversely, both the $(0.6,1)$-GAN and LSGAN consistently exhibit favorable FID performance for both learning rates. However, the $(0.6,1)$-GAN converges to a low FID, while the FID of the LSGAN slightly increases as training approaches 100 epochs. Finally, Fig. \ref{fig:celeba-epochs-fid-faces}(b) displays a grid of generated Celeb-A faces, randomly sampled over 8 seeds for three GANs trained for 100 epochs with a learning rate of $5 \times 10^{-4}$. Here, we observe that the faces generated by the $(0.6,1)$-GAN and LSGAN exhibit a comparable level of quality to the rightmost column images, which are randomly sampled from the real Celeb-A dataset. On the other hand, the vanilla $(1,1)$-GAN shows clear signs of performance instability, as some seeds yield high-quality images while others do not.

\subsubsection{LSUN Classroom Results}

In Figure \ref{fig:celeba-lsun-fid-lr}(b), we illustrate the relationship between learning rate and FID for GANs trained on the LSUN dataset for 100 epochs. In fact, when all GANs are trained with a learning rate of $1 \times 10^{-4}$, they consistently deliver satisfactory performance. 
However, increasing it to $2 \times 10^{-4}$ leads to instability in the vanilla $(1,1)$-GAN across 50 seeds. 

On the other hand, we observe that $\alpha_{D}<1$ contributes to stabilizing the FID across the 50 seeds even when trained with slightly higher learning rates. In Figure \ref{fig:celeba-lsun-fid-lr}(b), we see that as $\alpha_{D}$ is tuned down to 0.6, the mean FIDs consistently decrease across all tested learning rates. These lower FIDs can be attributed to the increased stability of the network. 
Despite the gains in GAN stability achieved by tuning down $\alpha_{D}$, Figure \ref{fig:celeba-lsun-fid-lr} demonstrates a noticeable disparity between the best $(\alpha_{D}, \alpha_{G})$-GAN and the SOTA LSGAN. This suggests that there is still room for improvement in generating high-dimensional images with $(\alpha_{D}, \alpha_{G})$-GANs. 

In Appendix \ref{appendix:experimental-details-results}, Figure \ref{fig:lsun_epochs_fids_images}(a), we illustrate the average FID throughout the training process for three GANs: $(1,1)$-GAN, $(0.6,1)$-GAN, and LSGAN, using two different learning rates: $1 \times 10^{-4}$ and $2 \times 10^{-4}$. These findings validate that the vanilla $(1,1)$-GAN performs well when trained with the lower learning rate, but struggles significantly with the higher learning rate. In contrast, the $(0.6,1)$-GAN exhibits less sensitivity to learning rate, while the LSGAN achieves nearly identical scores for both learning rates. 
In Figure \ref{fig:lsun_epochs_fids_images}(b), we showcase the image quality generated by each GAN at epoch 100 with the higher learning rate. This plot highlights that the vanilla $(1,1)$-GAN frequently fails during training, whereas the $(0.6,1)$-GAN and LSGAN produce images that are more consistent in mimicking the real distribution.
Finally, we present the FID vs. learning rate results for both datasets in Table \ref{table:celeba-lsun-stability} in Appendix \ref{appendix:experimental-details-results}. This allows yet another way to evaluate performance by comparing the percentage (out of 50 seeds) of FID scores below a desired threshold for each dataset, as detailed in the appendix.


\section{Conclusion}
Building on our prior work introducing CPE loss GANs and $\alpha$-GANs, we have introduced new results on the equivalence of CPE loss GANs and $f$-GANs, convergence properties of the symmetric $f$-divergences induced by CPE loss GANs under certain conditions, and the generalization and estimation error for CPE loss GANs including $\alpha$-GANs. We have introduced a dual-objective GAN formulation, focusing in particular on using $\alpha$-loss with potentially different $\alpha$ values for both players' objectives. GANs offer an alternative to diffusion models in being faster to train but training instabilities stymie such advantages. In this context, our results are very promising and highlight how tuning $\alpha$ can not only alleviate training instabilities but also enhance robustness to learning rates and training epochs, hyperparameters whose optimal values are generally not known \emph{a priori}. A natural extension to our work is to define and study generalization of dual-objective GANs. An equally important problem is to evaluate if our observations hold more broadly, including, when the training data is noisy \cite{nietert2022outlier}. 

While different $f$-divergence based GANs have been introduced, no principled reasons have been proposed thus far for choosing a specific $f$-divergence measure and corresponding loss functions to optimize. Even in the more practical finite sample and model capacity settings, different choices of objectives, as shown earlier, lead to different neural network divergence measures.
Using tunable losses, our work has the advantage of motivating the choice of appropriate loss functions and the resulting $f$-divergence/neural network divergence from the crucial viewpoint of avoiding training instabilities. This connection between loss functions and divergences to identify the appropriate measure of goodness can be of broader interest both to the IT and ML communities.


\bibliographystyle{IEEEtran}
\bibliography{Bibliography}

\newpage
\appendices
\nobalance

\section{Proof of Theorem~\ref{thm:correspondence}}\label{apndx:proof-of-thm1}
Consider a symmetric CPE loss $\ell(y,\hat{y})$, i.e., $\ell(1,\hat{y})=\ell(0,1-\hat{y})$. We may define an associated margin-based loss using an increasing bijective link function $l:\mathbb{R}\rightarrow [0,1]$ as
\begin{align}\label{eqn:margin-from-CPE}
    \tilde{\ell}(t):=\ell(1,l(t)),
\end{align}
where the link $l$ satisfies the following mild regularity conditions: 
\begin{align}
l(-t)=1-l(t)\label{eqn:regularityonlink},\\
l(0)=\frac{1}{2}\label{eqn:regularityonlink1},\\
l^{-1}(t)+l^{-1}(1-t)=0\label{eqn:regularityonlink2}
\end{align}
(e.g., sigmoid function, $\sigma(t)=1/(1+\mathrm{e}^{-t})$ satisfies this condition). Consider the inner optimization problem in \eqref{eq:Goodfellowobj} with the value function in \eqref{eqn:lossfnps1} for this CPE loss $\ell$.
\begin{align}
    &\sup_\omega\int_{\mathcal{X}}(-p_r(x)\ell(1,D_\omega(x))-p_{G_\theta}(x)\ell(0,D_\omega(x)))\ dx\nonumber\\
    &= \int_{\mathcal{X}}\sup_{p_x\in[0,1]}(-p_r(x)\ell(1,p_x)-p_{G_\theta}(x)\ell(0,p_x))\ dx\label{eqn:thm1proof6}\\
   &=\int_{\mathcal{X}}\sup_{p_x\in[0,1]}(-p_r(x)\ell(1,p_x)-p_{G_\theta}(x)\ell(1,1-p_x))\ dx\label{eqn:thm1proof1}\\
    &=\int_{\mathcal{X}}\sup_{t_x\in\mathbb{R}}(-p_r(x)\ell(1,l(t_x))-p_{G_\theta}(x)\ell(1,1-l(t_x)))dx\\
        &=\int_{\mathcal{X}}\sup_{t_x\in\mathbb{R}}(-p_r(x)\ell(1,l(t_x))-p_{G_\theta}(x)\ell(1,l(-t_x)))\ dx\label{eqn:thm1proof2}\\
    &=\int_{\mathcal{X}}\sup_{t_x\in\mathbb{R}}(-p_r(x)\tilde{\ell}(t_x)-p_{G_\theta}(x)\tilde{\ell}(-t_x)\ dx\label{eqn:thm1proof3}\\
   & =\int_{\mathcal{X}}p_{G_\theta}(x)\left(-\inf_{t_x\in\mathbb{R}}\left(\tilde{l}(-t_x)+\frac{p_r(x)}{p_{G_\theta}(x)}\tilde{l}(t_x)\right)\right) dx\label{eqn:eqn:tm1proof4}
\end{align}
where \eqref{eqn:thm1proof1} follows because the CPE loss $\ell(y,\hat{y})$ is symmetric, \eqref{eqn:thm1proof2} follows from \eqref{eqn:regularityonlink}, and \eqref{eqn:thm1proof3} follows from the definition of the margin-based loss $\tilde{\ell}$ in \eqref{eqn:margin-from-CPE}. Now note that the function $f$ defined as
\begin{align}\label{eqn:tm1proof5}
    f(u)=-\inf_{t\in\mathbb{R}}\left(\tilde{\ell}(-t)+u\tilde{\ell}(t)\right), \quad u \ge 0
\end{align}
is convex since the infimum of affine functions is concave (observed earlier in \cite{NguyenWJ09} in a correspondence between margin-based loss functions and $f$-divergences). So, from \eqref{eqn:eqn:tm1proof4}, we get
\begin{align}
    \sup_\omega\int_{\mathcal{X}}&(-p_r(x)\ell(1,D_\omega(x))-p_{G_\theta}(x)\ell(0,D_\omega(x)))\ dx\nonumber\\
    &=\int_{\mathcal{X}}p_{G_\theta}(x)f\left(\frac{p_r(x)}{p_{G_\theta}(x)}\right)\ dx\\
    &=D_f(P_r\|P_{G_\theta}).
\end{align}
Thus, the resulting min-max optimization in \eqref{eqn:GANgeneral-background} reduces to minimizing the $f$-divergence, $D_f(P_r\|P_{G_\theta})$ with $f$ as given in \eqref{eqn:tm1proof5}.

For the converse statement, first note that given a symmetric $f$-divergence, it follows from \cite[Theorem~1(b) and Corollary~3]{NguyenWJ09} that there exists a decreasing and convex margin-based loss function $\tilde{\ell}$ such that $f$ can be expressed in the form
\eqref{eqn:tm1proof5}. We may define an associated symmetric CPE loss $\ell(y,\hat{y})$ with
\begin{align}\label{eqn:margintocpe}
\ell(1,\hat{y}):=\tilde{\ell}(l^{-1}(\hat{y})),
\end{align}
where $l^{-1}$ is the inverse of the same link function. Now repeating the steps as in $\eqref{eqn:thm1proof6}-\eqref{eqn:eqn:tm1proof4}$, it is clear that the GAN based on this (symmetric) CPE loss results in minimizing the same symmetric $f$-divergence. It remains to verify that the symmetric CPE loss defined in \eqref{eqn:margintocpe} is such that $\ell(1,\hat{y})$ is decreasing so that the intuitive interpretation of vanilla GAN is retained and that it satisfies \eqref{eqn:condnonfnsforGAN} so that the optimal discriminator guesses uniformly at random when $P_r=P_{G_\theta}$. Note that $\ell^\prime(1,\hat{y})=\tilde{\ell}^\prime(l^{-1}(\hat{y}))(l^{-1})^\prime(\hat{y})\leq 0$ since the margin-based loss $\tilde{\ell}$ is decreasing and the link function $l$ (and hence its inverse) is increasing. So, $\ell(1,\hat{y})$ is decreasing. Observe that the loss function $\ell(1,\hat{y})=\tilde{\ell}(l^{-1}(\hat{y}))$ may not be convex in $y$ even though the margin-based loss function $\tilde{\ell}(\cdot)$ is convex. However, we show that the symmetric CPE loss associated with \eqref{eqn:margintocpe} indeed satisfies \eqref{eqn:condnonfnsforGAN}.  
\begin{align}
    -\ell(1,t)-\ell(0,t)&=-\ell(1,t)-\ell(1,1-t)\\
    &=-\tilde{\ell}(l^{-1}(t))-\tilde{\ell}(l^{-1}(1-t))\\
    &\leq -2\tilde{\ell}\left(\frac{1}{2}l^{-1}(t)+\frac{1}{2}l^{-1}(1-t)\right)\label{eqn:convexofmargin}\\
    &=-2\tilde{\ell}(0)\label{eqn:linkeqn}\\
    &=-2\tilde{\ell}\left(l^{-1}\left(\frac{1}{2}\right)\right)\label{eqn:linkeqn2}\\
    &=-\ell\left(1,\frac{1}{2}\right)-\ell\left(0,\frac{1}{2}\right),
\end{align}
where \eqref{eqn:convexofmargin} follows since the margin-based loss $\tilde{\ell}(\cdot)$ is convex, and \eqref{eqn:linkeqn} and \eqref{eqn:linkeqn2} follow from \eqref{eqn:regularityonlink1} and \eqref{eqn:regularityonlink2}, respectively.

\section{Proof of Theorem~\ref{thm:alpha-GAN}}\label{proofofthm1}
For a fixed generator, $G_\theta$, we first solve the optimization problem
\begin{align}
   \sup_{\omega\in\Omega}\int_\mathcal{X}\frac{\alpha}{\alpha-1}\left(p_r(x)D_\omega(x)^{\frac{\alpha-1}{\alpha}}+p_{G_\theta}(x)(1-D_\omega(x))^{\frac{\alpha-1}{\alpha}}\right)dx.
\end{align}
Consider the function
\begin{align}
    g(y)=\frac{\alpha}{\alpha-1}\left(ay^{\frac{\alpha-1}{\alpha}}+b(1-y)^{\frac{\alpha-1}{\alpha}}\right),
\end{align}
for $a,b>0$ and $y\in[0,1]$. To show that the optimal discriminator is given by the expression in \eqref{eqn:optimaldoisc}, it suffices to show that $g(y)$ achieves its maximum in $[0,1]$ at $y^*=\frac{a^\alpha}{a^\alpha+b^\alpha}$. Notice that for $\alpha>1$, $y^{\frac{\alpha-1}{\alpha}}$ is a concave function of $y$, meaning the function $g$ is concave. For $0<\alpha<1$, $y^{\frac{\alpha-1}{\alpha}}$ is a convex function of $y$, but since $\frac{\alpha}{\alpha-1}$ is negative, the overall function $g$ is again concave. Consider the derivative 
    $g^\prime(y^*)=0$,
which gives us
\begin{align}
    y^*=\frac{a^\alpha}{a^\alpha+b^\alpha}.
\end{align}
This gives \eqref{eqn:optimaldoisc}. With this, the optimization problem in \eqref{eqn:minimaxalphaGAN} can be written as $\inf_{\theta\in\Theta}C(G_\theta)$,
where
\begin{align}
   C(G_\theta)&=\frac{\alpha}{\alpha-1}
   \left[\int_\mathcal{X}\left(p_r(x)D_{\omega^*}(x)^{\frac{\alpha-1}{\alpha}}+p_{G_\theta}(x)(1-D_{\omega^*}(x))^{\frac{\alpha-1}{\alpha}}\right)dx-2\right]\\
   &=\frac{\alpha}{\alpha-1}\Bigg[\int_\mathcal{X}\Bigg(p_r(x)\left( \frac{p_r(x)^\alpha}{p_r(x)^\alpha+p_{G_\theta}(x)^\alpha}\right)^{\frac{\alpha-1}{\alpha}}+p_{G_\theta}(x)\left( \frac{p_r(x)^\alpha}{p_r(x)^\alpha+p_{G_\theta}(x)^\alpha}\right)^{\frac{\alpha-1}{\alpha}}\Bigg)dx-2\Bigg]\\
    &=\frac{\alpha}{\alpha-1}\left(\int_{\mathcal{X}}\left(p_r(x)^\alpha+p_{G_\theta}(x)^\alpha\right)^{\frac{1}{\alpha}}dx-2\right)\\
    &=D_{f_\alpha}(P_r||P_{G_\theta})+\frac{\alpha}{\alpha-1}\left(2^{\frac{1}{\alpha}}-2\right),
\end{align}
where for the convex function $f_\alpha$ in \eqref{eqn:falpha},
\begin{align}
    D_{f_\alpha}(P_r||P_{G_\theta})=\int_\mathcal{X} p_{G_\theta}(x)f_\alpha\left(\frac{p_r(x)}{p_{G_\theta}(x)}\right) dx=\frac{\alpha}{\alpha-1}\left(\int_{\mathcal{X}}\left(p_r(x)^\alpha+p_{G_\theta}(x)^\alpha\right)^{\frac{1}{\alpha}}dx-2^{\frac{1}{\alpha}}\right).
\end{align}
This gives us \eqref{eqn:inf-obj-alpha}. Since $D_{f_\alpha}(P_r||P_{G_\theta})\geq 0$ with equality if and only if $P_r=P_{G_\theta}$, we have $C(G_\theta)\geq \frac{\alpha}{\alpha-1}\left(2^{\frac{1}{\alpha}}-2\right)$ with equality if and only if $P_r=P_{G_\theta}$.
\balance
\section{Proof of Theorem~\ref{thm:fgans}}\label{proofoftheorem2}
First, using L'H\^{o}pital's rule we can verify that, for $a,b>0$,
\begin{align}
\lim_{\alpha\rightarrow 1}\frac{\alpha}{\alpha-1}\left(\left(a^\alpha+b^\alpha\right)^{\frac{1}{\alpha}}-2^{\frac{1}{\alpha}-1}(a+b)\right)
=a\log{\left(\frac{a}{\frac{a+b}{2}}\right)}+b\log{\left(\frac{b}{\frac{a+b}{2}}\right)}.
\end{align}
Using this, we have
\begin{align}
D_{f_1}(P_r||P_{G_\theta})&\coloneqq\lim_{\alpha\rightarrow 1}D_{f_\alpha}(P_r||P_{G_\theta})\\
&=\lim_{\alpha\rightarrow 1}\frac{\alpha}{\alpha-1}\left(\int_\mathcal{X}\left(p_r(x)^\alpha+p_{G_\theta}(x)^\alpha\right)^{\frac{1}{\alpha}}dx-2^{\frac{1}{\alpha}}\right)\\
&=\lim_{\alpha\rightarrow 1}\Bigg[\frac{\alpha}{\alpha-1}\int_\mathcal{X}\big(\big(p_r(x)^\alpha+p_{G_\theta}(x)^\alpha\big)^{\frac{1}{\alpha}}-2^{\frac{1}{\alpha}-1}(p_r(x)+p_{G_\theta}(x))\big)dx\Bigg]\label{eqn:lim-integrand}\\
&=\int_{\mathcal{X}}p_r(x)\log{\frac{p_r(x)}{\left(\frac{p_r(x)+p_{G_\theta}(x)}{2}\right)}}dx+\int_{\mathcal{X}}p_{G_\theta}(x)\log{\frac{p_{G_\theta}(x)}{\left(\frac{p_r(x)+p_{G_\theta}(x)}{2}\right)}}dx\label{eqn:lim-int-interchange}\\
&=:2D_{\text{JS}}(P_r||P_{G_\theta})\label{eqn:JSD},
\end{align}
where {\eqref{eqn:lim-int-interchange} follows by interchanging the limit and the integral by invoking the dominated convergence theorem because of the boundedness of $f_\alpha$~\cite[Theorem~8]{LieseV06}} and  $D_{\text{JS}}(\cdot||\cdot)$ in \eqref{eqn:JSD} is the Jensen-Shannon divergence.
Now, as $\alpha\rightarrow 1$, \eqref{eqn:inf-obj-alpha} equals $\inf_{\theta\in\Theta}2D_{\text{JS}}(P_r||P_{G_\theta})-\log{4}$ recovering the vanilla GAN.

Substituting $\alpha=\frac{1}{2}$ in \eqref{eqn:alpha-divergence}, we get
\begin{align}
    D_{f_{\frac{1}{2}}}(P_r||P_{G_\theta})&=-\int_\mathcal{X}\left(\sqrt{p_r(x)}+\sqrt{p_{G_\theta}(x)}\right)^2dx+4\\
    &=\int_{\mathcal{X}}\left(\sqrt{p_r(x)}-\sqrt{p_{G_\theta}(x)}\right)^2dx\\
    &=:2D_{\text{H}^2}(P_r||P_{G_\theta}),
\end{align}
where $D_{\text{H}^2}(P_r||P_{G_\theta})$ is the squared Hellinger distance. For $\alpha=\frac{1}{2}$, \eqref{eqn:inf-obj-alpha} gives $2\inf_{\theta\in\Theta}D_{\text{H}^2}(P_r||P_{G_\theta})-2$ recovering Hellinger GAN (up to a constant). 

Noticing that, for $a,b>0$, $\lim_{\alpha\rightarrow \infty}\left(a^\alpha+b^\alpha\right)^{\frac{1}{\alpha}}=\max\{a,b\}$ and defining $\mathcal{A}:=\{x\in\mathcal{X}:p_r(x)\geq p_{G_\theta}(x)\}$, we have
\begin{align}
D_{f_\infty}(P_r||P_{G_\theta})&\coloneqq \lim_{\alpha\rightarrow\infty}D_{f_\alpha}(P_r||P_{G_\theta})\\
&=\lim_{\alpha\rightarrow\infty}\frac{\alpha}{\alpha-1}\left(\int_\mathcal{X}\left(p_r(x)^\alpha+p_{G_\theta}(x)^\alpha\right)^{\frac{1}{\alpha}}dx-2^{\frac{1}{\alpha}}\right)\label{integrand1}\\
&=\int_\mathcal{X}\max\{p_r(x),p_{G_\theta}(x)\}\ dx-1\label{eqn:lim-interchange}\\
&=\int_{\mathcal{X}}\max\{p_r(x)-p_{G_\theta}(x),0\}\ dx\\
&=\int_{\mathcal{A}}(p_r(x)-p_{G_\theta}(x))\ dx\\
&=\int_{\mathcal{A}}\frac{p_r(x)-p_{G_\theta}(x)}{2}\ dx+\int_{\mathcal{A}^c}\frac{p_{G_\theta}(x)-p_r(x)}{2}\ dx\\
&=\frac{1}{2}\int_{\mathcal{X}}\left|p_r(x)-p_{G_\theta}(x)\right|\ dx\\
&=:D_{\text{TV}}(P_r||P_{G_\theta})\label{eqn:TVD},
\end{align}
where {\eqref{eqn:lim-interchange} follows by interchanging the limit and the integral by invoking the dominated convergence theorem because of the boundedness of $f_\alpha$~\cite[Theorem~10]{LieseV06}} and $D_{\text{TV}}(P_r||P_{G_\theta})$ in \eqref{eqn:TVD} is the total variation distance between $P_r$ and $P_{G_\theta}$. Thus, as $\alpha\rightarrow\infty$, \eqref{eqn:inf-obj-alpha} equals $\inf_{\theta\in\Theta}D_{\text{TV}}(P_r||P_{G_\theta})-1$ recovering TV-GAN (modulo a constant).

See Fig.~\ref{fig:plotofdivergence} for an illustration of the behavior of $D_{f_\alpha}$ for different values of $\alpha$.

\begin{figure}[t]
\centering
\footnotesize
{\includegraphics[width=0.4\linewidth]{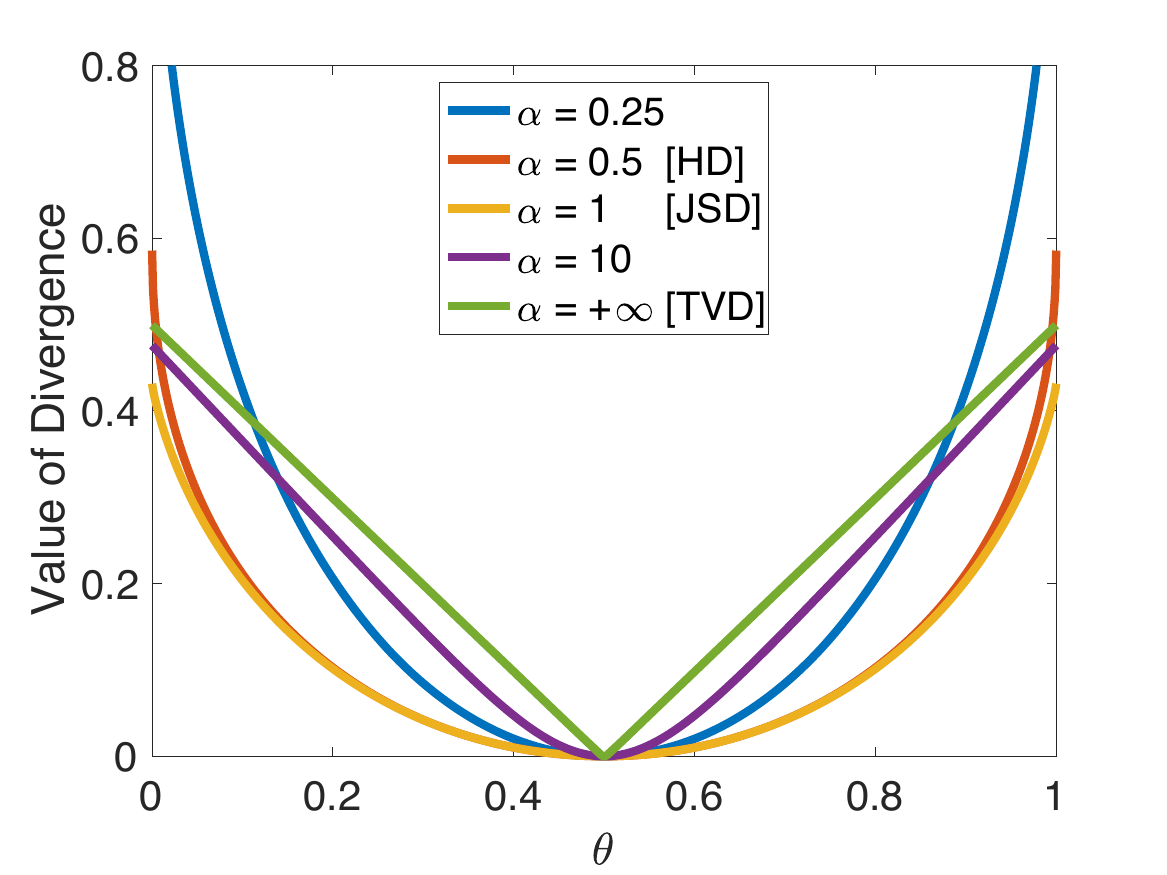}} \\

\caption{A plot of $D_{f_{\alpha}}$ in~\eqref{eqn:alpha-divergence} for several values of $\alpha$ where $p \sim \text{Ber}(1/2)$ and $q \sim \text{Ber}(\theta)$. Note that HD, JSD, and TVD, are abbreviations for Hellinger, Jensen-Shannon, and Total Variation divergences, respectively. As $\alpha \rightarrow 0$, the curvature of the divergence increases, placing increasingly more weight on $\theta \neq 1/2$. 
Conversely, for $\alpha \rightarrow \infty$, $D_{f_{\alpha}}$ quickly resembles $D_{f_{\infty}}$, hence a saturation effect of $D_{f_{\alpha}}$.}
\label{fig:plotofdivergence}
\end{figure}

\section{Proof of Theorem \ref{thm:obj-equiv-gen}}
\label{appendix:obj-equiv-gen}
We first derive the Fenchel conjugate $\Tilde{f}^*_\alpha$ of $\Tilde{f}_\alpha$ as follows:
\begin{align}
        \Tilde{f}_\alpha^* (t) = \underset{u}{\sup} \, \left(ut -  \Tilde{f}_\alpha(u)\right) = \frac{\alpha}{\alpha-1} \, \underset{u}{\sup} \, \left(1+\left(1+\frac{\alpha-1}{\alpha} t\right)u - (1+u^\alpha)^{\frac{1}{\alpha}} \right).
\end{align}
The optimum $u_*$ is obtained by setting the derivative of $ut -  \Tilde{f}_\alpha(u)$ to zero, yielding
\begin{align}
1+\frac{\alpha-1}{\alpha} t = u^{\alpha-1}_*(1+u^\alpha_*)^{\frac{1}{\alpha}-1} = \Big(\frac{u_*^\alpha}{1+u_*^\alpha} \Big)^\frac{\alpha-1}{\alpha},
\label{eq:deriv}
\end{align}
i.e.,
\begin{align}
u_*=u_*(t) = \Big(\frac{s(t)}{1-s(t)}\Big)^{\frac{1}{\alpha}}
\label{eq:ustar}
\end{align}
with
\begin{align}
    s(t)=\left(1+\frac{\alpha-1}{\alpha} t\right)^{\frac{\alpha}{\alpha-1}}.
\label{eq:s}
\end{align}
The verification that $u_*$ is a global maximizer over $u\ge0$ follows from
\begin{align*}
    (ut -  \Tilde{f}_\alpha(u))^{\prime\prime} = -\Big(\frac{u^\alpha}{1+u^\alpha}\Big)^{\alpha-1}(1+u^\alpha)^{-2}\alpha u^{\alpha-1}<0
\end{align*}
for all $u > 0$.
The relations \eqref{eq:deriv} and \eqref{eq:ustar} then lead to
\begin{align}
        \Tilde{f}_\alpha^* (t) = u_*(t)t -  \Tilde{f}_\alpha(u_*(t)) &= \frac{\alpha}{\alpha-1}\left(1+\left(1+\frac{\alpha-1}{\alpha} t\right)u_*(t) - (1+u_*(t)^\alpha)^{\frac{1}{\alpha}} \right) \nonumber \\ 
        &= \frac{\alpha}{\alpha-1}\left(1- (1+u_*(t)^\alpha)^{\frac{1}{\alpha}-1} \right) \nonumber \\
        & = \frac{\alpha}{\alpha-1}\left(1- (1-s(t))^{\frac{\alpha-1}{\alpha}} \right),
        \label{eq:fconj}
\end{align}
where $s$ is given by \eqref{eq:s}. The domain $\text{dom}(\Tilde{f}_\alpha^*)$ consists of values $t$ such that $1+\frac{\alpha-1}{\alpha} t \ge 0$ and $s(t) \le 1$, i.e., $t \in [-\frac{\alpha}{\alpha-1},0]$ for $\alpha>1$ and $t \le 0$ for $\alpha\in(0,1)$. Also note that 
\begin{equation*}
    \Tilde{f}_1^*(t)=\lim_{\alpha \to 1} \Tilde{f}_\alpha^*(t)=\lim_{\alpha \to 1} \frac{\alpha}{\alpha-1}\left(1- (1-s(t))^{\frac{\alpha-1}{\alpha}} \right)= -\log(1-e^t)
\end{equation*} for $t\le0$, where $s$ is again given by \eqref{eq:s}.

In the following we consider $\alpha \ne 1$ with results also valid for $\alpha =1$ by continuity. Let $v \in \overline{\mathbb R}$ and consider
\begin{equation}
    d = s(g_{f_\alpha}(v)) = \left(1+\frac{\alpha-1}{\alpha} g_{f_\alpha}(v)\right)^\frac{\alpha}{\alpha-1} 
    \label{eq:d-gen}.
\end{equation}
We first show that $d \in [0,1]$ and then show that \eqref{eq:thm1-gen} is satisfied. 

If $\alpha >1$, then $g_{f_\alpha}(v) \in [-\frac{\alpha}{\alpha-1},0]=\text{dom}(\Tilde{f}^*_\alpha)$. Therefore, $d \in [0,1]$. If $\alpha \in (0,1)$, then $g_{f_\alpha}(v) \in [-\infty,0]=\text{dom}(\Tilde{f}^*_\alpha)$. Therefore, $\left(1+\frac{\alpha-1}{\alpha} g_{f_\alpha}(v)\right) \in [1,\infty]$, and hence $d \in [0,1]$.

Using \eqref{eq:d-gen},
\begin{equation*}
    \ell_\alpha(1,d) = \frac{\alpha}{\alpha-1}\left(1-d^{\frac{\alpha-1}{\alpha}} \right) = \frac{\alpha}{\alpha-1}\left(1-s(g_{f_\alpha}(v))^{\frac{\alpha-1}{\alpha}} \right) =  -g_{f_\alpha}(v),
\end{equation*}
and
\begin{equation*}
    \ell_\alpha(0,d) =\frac{\alpha}{\alpha-1}\left(1-\left(1-d\right)^{\frac{\alpha-1}{\alpha}} \right) = \frac{\alpha}{\alpha-1}\left(1-\left(1-s(g_{f_\alpha}(v))\right)^{\frac{\alpha-1}{\alpha}} \right) =\Tilde{f}^*_\alpha(g_{f_\alpha}(v)).
\end{equation*}
Conversely, let $d \in [0,1]$ and consider
\begin{equation}
    v = g_{f_\alpha}^{-1}\left(-\ell_\alpha(1,d)\right) =g_{f_\alpha}^{-1}\left(\frac{\alpha}{\alpha-1}(d^\frac{\alpha-1}{\alpha} - 1)\right).
    \label{eq:v-gen}
\end{equation}
We first show that $v \in \overline{\mathbb R}$ and then show that \eqref{eq:thm1-gen} is satisfied. 

If $\alpha >1$, then $-\ell_\alpha(1,d) \in [-\frac{\alpha}{\alpha-1},0]=\text{dom}(\Tilde{f}^*_\alpha)$. Therefore, $v \in \overline{\mathbb R}$. If $\alpha \in (0,1)
$, then $d^\frac{\alpha-1}{\alpha} \in [0,\infty]$ and $-\ell_\alpha(1,d) \in [-\infty,0]=\text{dom}(\Tilde{f}^*_\alpha)$. Hence, $v \in \overline{\mathbb R}$.

Using \eqref{eq:v-gen},
\begin{equation*}
    g_{f_\alpha}(v) = -\ell_\alpha(1,d),
\end{equation*}
and
\begin{equation*}
s(g_{f_\alpha}(v)) = \left(1+\frac{\alpha-1}{\alpha} g_{f_\alpha}(v) \right)^\frac{\alpha}{\alpha-1}=\left(1+\frac{\alpha-1}{\alpha} \left(\frac{\alpha}{\alpha-1}(d^\frac{\alpha-1}{\alpha} -1) \right) \right)^\frac{\alpha}{\alpha-1} = d,
\end{equation*}
so that
\begin{equation*}
\Tilde{f}^*_\alpha(g_{f_\alpha}(v)) = \frac{\alpha}{\alpha-1}\left(1-\left(1-s(g_{f_\alpha}(v))\right)^{\frac{\alpha-1}{\alpha}} \right) = \frac{\alpha}{\alpha-1}\left(1-\left(1-d\right)^{\frac{\alpha-1}{\alpha}} \right) = \ell_\alpha(0,d).
\end{equation*}

\section{Proof of Corollary \ref{corollary:equivalence-falphaGAN-alphaGAN}}
\label{appendix:equivalence-falphaGAN-alphaGAN}

For $Q_\omega \in A$ define $D_\omega \in B$ such that $d=D_\omega(x)$ is obtained from \eqref{eq:d-gen} with $v = Q_\omega(x)$ for all $x \in \mathcal{X}$. By Theorem \ref{thm:obj-equiv-gen}, $g(Q_\omega)=h(D_\omega)$. Conversely, for $D_\omega \in B$ define $Q_\omega \in A$ such that $v=Q_\omega(x)$ is obtained from \eqref{eq:v-gen} with $d=D_\omega(x)$ for all $x \in \mathcal{X}$. Again by Theorem \ref{thm:obj-equiv-gen}, $h(D_\omega)=g(V_\omega)$.

To show that $k$ is bijective, we first show that $s:\text{dom}(\Tilde{f}_\alpha^*)\to [-\infty,1]$ defined in \eqref{eq:s} is bijective. Let the function $s^{-1}:[-\infty,1]\to\text{dom}(\Tilde{f}_\alpha^*)$ be defined by $s^{-1}(u)=\frac{\alpha}{\alpha-1}(u^\frac{\alpha-1}{\alpha}-1)$. Let $t \in \text{dom}(\Tilde{f}_\alpha^*)$. Then
\[s^{-1}(s(t)) = \frac{\alpha}{\alpha-1}\left[\left(\left(1+\frac{\alpha-1}{\alpha} t\right)^\frac{\alpha}{\alpha-1} \right)^\frac{\alpha-1}{\alpha} -1\right] = t.\]
Now, let $u \in [-\infty,1]$. Then
\[s(s^{-1}(u)) = \left(1+\frac{\alpha-1}{\alpha} \left(\frac{\alpha}{\alpha-1}\left(u^\frac{\alpha-1}{\alpha}-1\right) \right) \right)^\frac{\alpha}{\alpha-1} = u.\]
Therefore, $s^{-1}$ is the inverse of $s$, and hence $s$ is bijective. As the composition of two bijective functions, $k$ is also bijective.

\section{Proof of Theorem~\ref{theorem:equivalence-fGAN-CPEGAN}}\label{proofofthm-equv}
As noted in the proof of Theorem~\ref{thm:correspondence}, given a symmetric $f$-divergence, it follows from \cite[Theorem~1(b) and Corollary~3]{NguyenWJ09} that there exists a CPE (partial) loss ${\ell}$ such that 
\begin{align}\label{thm-equvprof:1}
f(u)=\sup_{t\in[0,1]}-u\ell(t)-\ell(1-t).
\end{align}
We assume that the loss $l$ is strictly convex as mentioned in the theorem statement. Note that 
\begin{align}\label{eqn:f-0f-fGAN}
    f(u)=\sup_{v\in\text{dom}f^*} uv-f^*(v).
\end{align}
Noticing that the inner optimization problems in the CPE loss GAN  and $f$-GAN formulations reduce to pointwise optimizations \eqref{thm-equvprof:1} and \eqref{eqn:f-0f-fGAN}, respectively, it suffices to show that the variational forms of $f$ in \eqref{thm-equvprof:1} and \eqref{eqn:f-0f-fGAN} are equivalent. To this end, we show that \eqref{thm-equvprof:1} is equivalent to the optimization problem
\begin{align}\label{eqn:opt-mattShannon}
    f(u)=\sup_{v\in\mathbb{R}_+}uf^\prime(v)-[vf^\prime(v)-f(v)]
\end{align}
which is known to be equivalent to \eqref{eqn:f-0f-fGAN}~\cite{shannon2020properties}. Let $k:\mathbb{R}_+\rightarrow [0,1]$ denote the bijective mapping from $u\in\mathbb{R}_+$ to the optimizer in \eqref{thm-equvprof:1}. {So, $k(u)$ satisfies
\begin{align}\label{eqn:proofofthm-equv1}
-u\ell^\prime(k(u))+\ell^\prime(1-k(u))=0.
\end{align}
Note that it follows from implicit function theorem that $k(u)$ is also differentiable.
}
Fix a $v\in\mathbb{R}_+$. With this, we have 
\begin{align}\label{thm-equvprof:4}
f(v)=-v\ell(k(v))-\ell(1-k(v)).
\end{align}
{On differentiating both sides of \eqref{thm-equvprof:4} with respect to $v$, we get
\begin{align}
    f^\prime(v)&=-\ell(k(v))+k^\prime(v)(v\ell^\prime(-k(v))+\ell^\prime(1-k(v)))\label{eqn:proofofthm-equv2}\\
    &=-\ell(k(v))\label{eqn:proofofthm-equv3},
\end{align}
where \eqref{eqn:proofofthm-equv3} follows from \eqref{eqn:proofofthm-equv1} by replacing $u$ with $v$.
}
Consider
\begin{align}
    vf^\prime(v)-f(v)&=-v\ell(k(v))+v\ell(k(v))+\ell(1-k(v))\label{thm-equvprof:3}\\
    &=\ell(1-k(v)),
\end{align}
where \eqref{thm-equvprof:3} follows from \eqref{eqn:proofofthm-equv3} and \eqref{thm-equvprof:4}. Thus, with the change of variable $t=k(v)$, the objective function in \eqref{eqn:opt-mattShannon} is equal to that of \eqref{thm-equvprof:1}. Since the function $k$ is invertible, for a fixed $t\in[0,1]$, we can also show that the change of variable $v=k^{-1}(t)$ in the objective function of \eqref{thm-equvprof:1} gives the objective function of \eqref{eqn:opt-mattShannon}.  
 
 \section{Proof of Theorem~\ref{thm:equivalenceinconvergence}}\label{proofoftheorem4}
Without loss of generality we take the functions $f_1$ and $f_2$ to be non-negative using the fact that $D_f(\cdot\|\cdot)=D_{f^\prime}(\cdot\|\cdot)$ whenever $f^\prime(x)=f(x)+c(x-1)$, for some $c\in\mathbb{R}$ (see \cite[Theorem~2]{LieseV06}). Note that it suffices to show that any symmetric $f$-divergence $D_f(\cdot\|\cdot)$ is equivalent to $D_{\text{TV}}(\cdot\|\cdot)$, i.e., $D_f(P_n||P)\rightarrow 0$ as $n\rightarrow \infty$ if and only if $D_{\text{TV}}(P_n||P)\rightarrow 0$ as $n\rightarrow \infty$. To this end, we employ a property of any symmetric $f$-divergence which gives lower and upper bounds on it in terms of the total variation distance, $D_{\text{TV}}$. In particular, Feldman and \"{O}sterreicher ~\cite[Theorem~2]{FeldmanO89} proved that for any symmetric $f$-divergence $D_f$, probability distributions $P$ and $Q$, we have
\begin{align}\label{eqn:boundsonArimoto}
    \gamma_f(D_{\text{TV}}(P||Q))\leq D_{f}(P||Q)\leq \gamma_f(1)D_{\text{TV}}(P||Q),
\end{align} 
where the function $\gamma_\alpha:[0,1]\rightarrow [0,\infty)$ defined by $\gamma_f(x)=(1+x)f\left(\frac{1-x}{1+x}\right)$ is convex, strictly increasing and continuous on $[0,1]$ such that $\gamma_f(0)=0$ and $\gamma_f(1)=2f(0)$. 

We first prove the `only if' part, i.e., $D_{f}(P_n||P)\rightarrow 0$ as $n\rightarrow \infty$ implies $D_{\text{TV}}(P_n||P)\rightarrow 0$ as $n\rightarrow \infty$. Suppose $D_{f}(P_n||P)\rightarrow 0$. From the lower bound in \eqref{eqn:boundsonArimoto}, it follows that $\gamma_f(D_{\text{TV}}(P_n||P))\leq D_{f}(P_n||P)$, for each $n\in\mathbbm{N}$. This implies that $\gamma_f(D_{\text{TV}}(P_n||P))\rightarrow 0$ as $n\rightarrow \infty$. We show below that $\gamma_f$ is invertible and $\gamma_f^{-1}$ is continuous. Then it would follow that $\gamma_f^{-1}\gamma_f(D_{\text{TV}}(P_n||P))=D_{\text{TV}}(P_n||P)\rightarrow \gamma_f^{-1}(0)=0$ as $n\rightarrow \infty$ proving that Arimoto divergence is stronger than the total variation distance. It remains to show that $\gamma_f$ is invertible and $\gamma_f^{-1}$ is continuous. Invertibility follows directly from the fact that $\gamma_f$ is strictly increasing function. For the continuity of $\gamma_\alpha^{-1}$, it suffices to show that $\gamma_f(C)$ is closed for a closed set $C\subseteq [0,1]$. The closed set $C$ is compact since a closed subset of a compact set ($[0,1]$ in this case) is also compact. Now since $\gamma_f$ is continuous, $\gamma_f(C)$ is compact because a continuous function of a compact set is compact. By Heine-Borel theorem, this gives that $\gamma_f(C)$ is closed (and bounded) as desired.

We prove the `if part' now, i.e., $D_{\text{TV}}(P_n||P)\rightarrow 0$ as $n\rightarrow \infty$ implies $D_{f}(P_n||P)\rightarrow 0$. It follows from the upper bound in $\eqref{eqn:boundsonArimoto}$ that $D_{f}(P_n||P)\leq D_{\text{TV}}(P_n||P)$, for each $n\in\mathbbm{N}$. This implies that $D_{f}(P_n||P)\rightarrow 0$ as $n\rightarrow \infty$ which completes the proof. 

\section{Equivalence of the Jensen-Shannon Divergence and the Total Variation Distance}\label{appendix:simpler-equivalence-JSD-TVD}
We first show that the total variation distance is stronger than the Jensen-Shannon divergence, i.e., $D_{\text{TV}}(P_n\|P)\rightarrow 0$ as $n\rightarrow \infty$ implies $D_{\text{JS}}(P_n\|P)\rightarrow 0$ as $n\rightarrow \infty$. Suppose $D_{\text{TV}}(P_n||P)\rightarrow 0$ as $n\rightarrow \infty$. Using the fact that the total variation distance upper bounds the Jensen-Shannon divergence~\cite[Theorem 3]{Lin91}, we have $D_{\text{JS}}(P_n||P)\leq (\log_\mathrm{e}{2}) D_{\text{TV}}(P_n||P)$, for each $n\in\mathbbm{N}$. This implies that $D_{\text{JS}}(P_n||P)\rightarrow 0$ as $n\rightarrow \infty$ since $D_{\text{TV}}(P_n||P)\rightarrow 0$ as $n\rightarrow \infty$. The proof for the other direction, i.e., the Jensen-Shannon divergence is stronger than the total variation distance, is exactly along the same lines as that of \cite[Theorem~2(1)]{ArjovskyCB17} using triangle and Pinsker's inequalities. 

\section{Proof of Theorem~\ref{thm:generalizationofarora}}\label{appendix:generalizationofarora}
The proof is along similar lines as that of \cite[Theorem~3.1]{AroraGLMZ17}. Below we argue that, with high probability, for every discriminator $D_\omega$,
\begin{align}
  \left\lvert\mathbb{E}_{X\sim P_r}[\phi\left({D_\omega(X)}\right)]-\mathbb{E}_{X\sim P_{G_\theta}}[\phi\left(D_\omega(X)\right)]\right\rvert\leq \frac{\epsilon}{2},\label{eq8fromarora}\\ 
    \left\lvert\mathbb{E}_{X\sim P_r}[\psi\left({D_\omega(X)}\right)]-\mathbb{E}_{X\sim P_{G_\theta}}[\psi\left(D_\omega(X)\right)]\right\rvert\leq \frac{\epsilon}{2}\label{neweqnarora}.
\end{align}
Assuming $\omega^*$ to be an optimizer attaining $\tilde{d}_{\mathcal{F}}(P_r,P_{G_\theta})$, it would then follow that
\begin{align}
    \tilde{d}_\mathcal{F}(\hat{P}_r,\hat{P}_{G_\theta})
    &=\sup_{\omega\in\Omega}\left\vert\mathbb{E}_{X\sim \hat{P}_r}[\phi\left({D_\omega(X)}\right)]+\mathbb{E}_{X\sim \hat{P}_{G_\theta}}[\psi\left(D_\omega(X)\right)]\right\rvert\\
    &\geq \left\lvert\mathbb{E}_{X\sim \hat{P}_r}[\phi\left({D_{\omega^*}(X)}\right)]+\mathbb{E}_{X\sim \hat{P}_{G_\theta}}[\psi\left(D_{\omega^*}(X)\right)]\right\rvert\\
    &\geq \left\lvert\mathbb{E}_{X\sim {P}_r}[\phi\left({D_{\omega^*}(X)}\right)]+\mathbb{E}_{X\sim {P}_{G_\theta}}[\psi\left(D_{\omega^*}(X)\right)]\right\rvert\nonumber\\
    &\hspace{12pt}- \left\lvert\mathbb{E}_{X\sim P_r}[\phi\left({D_\omega(X)}\right)]-\mathbb{E}_{X\sim \hat{P}_r}[\phi\left(D_\omega(X)\right)]\right\rvert\nonumber\\
    &\hspace{12pt}-\left\lvert\mathbb{E}_{X\sim P_{G_\theta}}[\psi\left({D_\omega(X)}\right)]-\mathbb{E}_{X\sim \hat{P}_{G_\theta}}[\psi\left(D_\omega(X)\right)]\right\rvert\label{eqn:thm5eqn1}\\
    &\geq \tilde{d}_{\mathcal{F}}(P_r,P_G)-\epsilon\label{eqn:thm5eqn2},
\end{align}
where \eqref{eqn:thm5eqn1} follows from the triangle inequality, \eqref{eqn:thm5eqn2} follows from \eqref{eq8fromarora} and \eqref{neweqnarora}. Similarly, we can prove the other direction, i.e., $\tilde{d}_\mathcal{F}(\hat{P}_r,\hat{P}_G)\leq \tilde{d}_\mathcal{F}({P}_r,{P}_G)+\epsilon$, which implies \eqref{eqn:thm5}.

It remains to argue for the concentration bounds \eqref{eq8fromarora} and \eqref{neweqnarora}. Recall that the concentration bound in  \eqref{eq8fromarora} was proved in \cite[Proof of Theorem~3.1]{AroraGLMZ17} by considering a $\frac{\epsilon}{8LL_\phi}$-net in $\Omega$ and leveraging the Lipschitzianity of the discriminator class $\mathcal{F}$  and the function $\phi$. Using the exact same analysis, the concentration bound in \eqref{neweqnarora} can be proved separately by considering a $\frac{\epsilon}{8LL_\psi}$-net. For both the bounds to hold simultaneously, it suffices to consider a $\frac{\epsilon}{8L\max\{L_\phi,L_\psi\}}$-net along the same lines as the last part of \cite[Proof of Theorem~3.1]{AroraGLMZ17}, thus completing the proof.

\section{Proof of Theorem~\ref{thm:estimationerror-upperbound}}\label{proofoftheorem3}
We upper bound the estimation error in terms of the Rademacher complexities of appropriately defined \emph{compositional} classes building upon the proof techniques of \cite[Theorem~1]{JiZL21}. We then bound these Rademacher complexities using a contraction lemma~\cite[Lemma~26.9]{shalev2014understanding}. Details are in order.

We first review the notion of Rademacher complexity.
\begin{definition}[Rademacher complexity]
Let $\mathcal{G}_\Omega:=\{g_\omega: \mathcal{X} \to \mathbb{R} \mid \omega\in\Omega\}$ and $S=\{X_1.\dots,X_n\}$ be a set of random samples in $\mathcal{X}$ drawn independent and identically distributed (i.i.d.) from a distribution $P_X$. Then, the Rademacher complexity of $\mathcal{G}_\Omega$ is defined as
\begin{align}
    \mathcal{R}_S(\mathcal{G}_\Omega)=\mathbb{E}_{X,\epsilon}\sup_{\omega\in\Omega}\left\lvert\frac{1}{n}\sum_{i=1}^n\epsilon_ig_\omega(x_i)\right\rvert
\end{align}
where $\epsilon_1,\dots,\epsilon_n$ are independent random variables uniformly distributed on $\{-1,+1\}$. 
\end{definition}
We write our discriminator model in \eqref{eqn:disc-model} in the form
\begin{align}\label{eqn:thm3proof1}
    D_\omega(x)=\sigma(f_\omega(x)),
\end{align}
where $f_\omega$ is exactly the same discriminator model defined in \cite[Equation~(26)]{JiZL21}. Now by following the similar steps as in \cite[Equations~(16)-(18)]{JiZL21} by replacing $f_\omega(\cdot)$ in the first and second expectation terms in the definition of $d_{\mathcal{F}_{nn}}(\cdot,\cdot)$ by $\phi(D_\omega(\cdot))$ and $-\psi(D_\omega(\cdot))$, respectively, we get
\begin{align}
    d^{(\ell)}_{\mathcal{F}_{nn}}(P_r,\hat{P}_{G_{\hat{\theta}^*}})-\inf_{\theta\in\Theta} d^{(\ell)}_{\mathcal{F}_{nn}}(P_r,P_{G_{\theta}})
    &\leq 2\sup_{\omega}\left\lvert\mathbb{E}_{X\sim P_r}\phi(D_\omega(X))-\frac{1}{n}\sum_{i=1}^n\phi(D_\omega(X_i))\right\rvert\nonumber\\
    &\hspace{12pt}+2\sup_{\omega,\theta}\left\lvert\mathbb{E}_{Z\sim P_Z}\psi(D_\omega(g_\theta(Z)))-\frac{1}{m}\sum_{j=1}^m\psi(D_\omega(g_\theta(Z_j)))\right\rvert\label{eqn:thm3proof2}
\end{align}
Let us denote the supremums in the first and second terms in \eqref{eqn:thm3proof2} by $F^{(\phi)}(X_1,\dots,X_n)$ and $G^{(\psi)}(Z_1,\dots,Z_m)$, respectively. We next bound $G^{(\psi)}(Z_1,\dots,Z_m)$. Note that $\psi(\sigma(\cdot))$ is $\frac{L_\psi}{4}$-Lipschitz since it is a composition of two Lipschitz functions $\psi(\cdot)$ and $\sigma(\cdot)$ which are $L_\psi$- and $\frac{1}{4}$-Lipschitz respectively. For any $z_1,\dots,z_j,\dots,z_m,z_j^\prime$, using $\sup_r|h_1(r)|-\sup_r|h_2(r)|\leq \sup_r |h_1(r)-h_2(r)|$, we have 
\begin{align}
   G^{(\psi)}(z_1,\dots,z_j,\dots,z_m)-G^{(\psi)}(z_1,\dots,z_j^\prime,\dots,z_m)
   &\leq \sup_{\omega,\theta}\frac{1}{m}\left\lvert\psi(D_\omega(g_\theta(z_j)))-\psi(D_\omega(g_\theta(z_j^\prime)))\right\rvert\\
   &\leq \sup_{\omega,\theta}\frac{1}{m}\left\lvert\psi(\sigma(f_\omega(g_\theta(z_j))))-\psi(\sigma(f_\omega(g_\theta(z_j^\prime))))\right\rvert\label{eqn:thm3proof3}\\
      &\leq \frac{L_\psi}{4}\sup_{\omega,\theta}\frac{1}{m}\left\lvert \sigma(f_\omega(g_\theta(z_j)))-\sigma(f_\omega(g_\theta(z_j^\prime)))\right\rvert\label{eqn:thm3proof4}\\
            &\leq \frac{L_\psi}{4}\frac{2}{m}\left(M_k\prod_{i=1}^{k-1}(M_iR_i)\right)\left(N_l\prod_{j=1}^{l-1}(N_jS_j)\right)B_z\label{eqn:thm3proof6}\\
            &=\frac{L_\psi Q_z}{2m}\label{eqn:thm3proof7},
\end{align}
where \eqref{eqn:thm3proof3} follows from \eqref{eqn:thm3proof1}, \eqref{eqn:thm3proof4} follows because $\psi(\sigma(\cdot))$ is $\frac{L_\psi}{4}$-Lipschitz, \eqref{eqn:thm3proof6} follows by using the Cauchy-Schwarz inequality and the fact that $||Ax||_2\leq ||A||_F|||x||_2$ (as observed in \cite{JiZL21}), and \eqref{eqn:thm3proof7} follows by defining
\begin{align}\label{eqn:Q_zparameter}
    Q_z\coloneqq\left(M_k\prod_{i=1}^{k-1}(M_iR_i)\right)\left(N_l\prod_{j=1}^{l-1}(N_jS_j)\right)B_z.
\end{align}
Using \eqref{eqn:thm3proof7}, the McDiarmid's inequality~\cite[Lemma~26.4]{shalev2014understanding} implies that, with probability at least $1-\delta$,
\begin{align}
   G^{(\psi)}(Z_1,\dots,Z_j,\dots,Z_m)\leq\mathbb{E}_ZG^{(\psi)}(Z_1,\dots,Z_j,\dots,Z_m)+\frac{L_\psi Q_z}{2}\sqrt{\log{\frac{1}{\delta}}/(2m)}.\label{eqn:thm3proof8}
\end{align}
Following the standard steps similar to \cite[Equation~(20)]{JiZL21}, the expectation term in \eqref{eqn:thm3proof8} can be upper bounded as
\begin{align}
  \mathbb{E}_ZG^{(\psi)}(Z_1,\dots,Z_j,\dots,Z_m)
  \leq 2\mathbb{E}_{Z,\epsilon}\sup_{\omega,\theta}\left\lvert\frac{1}{m}\sum_{j=1}^m\epsilon_j\psi(D_\omega(g_\theta(Z_j)))\right\lvert =:2\mathcal{R}_{S_z}(\mathcal{H}^{(\psi)}_{\Omega\times\Theta}).
\end{align}
So, we have, with probability at least $1-\delta$,
\begin{align}
    G^{(\psi)}(Z_1,\dots,Z_j,\dots,Z_m)
    \leq 2\mathcal{R}_{S_z}(\mathcal{H}^{(\psi)}_{\Omega\times\Theta})+\sqrt{\log{\frac{1}{\delta}}}\frac{L_\psi Q_z}{2\sqrt{2m}}.\label{eqn:thm3proof9}
\end{align}
Using a similar approach, we have, with probability at least $1-\delta$,
\begin{align}\label{eqn:thm3proof10}
    F^{(\phi)}(X_1,\dots,X_n)\leq 2\mathcal{R}_{S_x}(\mathcal{F}_\Omega^{(\phi)})+\sqrt{\log{\frac{1}{\delta}}}\frac{L_\phi Q_x}{2\sqrt{2n}},
\end{align}
where 
\begin{align}
    \mathcal{R}_{S_x}(\mathcal{F}_\Omega^{(\phi)}):=\mathbb{E}_{X,\epsilon}\sup_{\omega}\left\lvert\frac{1}{n}\sum_{i=1}^n\epsilon_i\phi(D_\omega(X_i))\right\lvert. 
\end{align}
Combining \eqref{eqn:thm3proof2}, \eqref{eqn:thm3proof9}, and \eqref{eqn:thm3proof10} using a union bound, we get, with probability at least $1-2\delta$,
\begin{align}
    d^{(\ell)}_{\mathcal{F}_{nn}}(P_r,{P}_{G_{\hat{\theta}^*}})-\inf_{\theta\in\Theta} d^{(\ell)}_{\mathcal{F}_{nn}}(P_r,P_{G_{\theta}})
    \leq 4\mathcal{R}_{S_x}(\mathcal{F}_\Omega^{(\phi)})+4\mathcal{R}_{S_z}(\mathcal{H}^{(\psi)}_{\Omega\times\Theta})+\sqrt{\log{\frac{1}{\delta}}}\left(\frac{L_\phi Q_x}{\sqrt{2n}}+\frac{L_\psi Q_z}{\sqrt{2m}}\right)\label{eqn:thm3proof11}. 
\end{align}
Now we bound the Rademacher complexities in the RHS of \eqref{eqn:thm3proof11}. We present the contraction lemma on Rademacher complexity required to obtain these bounds. For $A\subset\mathbb{R}^n$, let $\mathcal{R}(A):=\mathbb{E}_\epsilon\left[\sup_{a\in A}\left\lvert\frac{1}{n}\sum_{i=1}^n\epsilon_ia_i\right\rvert\right]$.
\begin{lemma}[Lemma~26.9, \cite{shalev2014understanding}]\label{lemma:contraction}
For each $i\in\{1,\dots,n\}$, let $\gamma_i:\mathbb{R}\rightarrow\mathbb{R}$ be a $\rho$-Lipschitz function. Then, for $A\subset\mathbb{R}^n$,
\begin{align}
    \mathcal{R}(\gamma\circ A)\leq \rho \mathcal{R}(A),
\end{align}
where $\gamma\circ A:=\{(\gamma_1(a_1),\dots,\gamma_n(a_n)):a\in A\}$.
\end{lemma}
Note that $\phi(\sigma(\cdot))$ is $\frac{L_\phi}{4}$-Lipschitz since it is a composition of two Lipschitz functions $\phi(\cdot)$ and $\sigma(\cdot)$ which are $L_\phi$- and $\frac{1}{4}$-Lipschitz respectively. Consider
 \begin{align}
       \mathcal{R}_{S_x}(\mathcal{F}_\Omega^{(\phi)})
       &= \mathbb{E}_X\left[\mathcal{R}\left(\{\left(\phi(D_\omega(X_1)),\dots,\phi(D_\omega(X_n))\right):\omega\in\Omega\}\right)\right]\\
       &= \mathbb{E}_X\left[\mathcal{R}\left(\{\left(\phi(\sigma(f_\omega(X_1))),\dots,\phi(\sigma(f_\omega(X_n)))\right):\omega\in\Omega\}\right)\right]\label{eqn:thm3proof12}\\
       &\leq \frac{L_\phi}{4}\mathbb{E}_X\left[\mathcal{R}\left(\{\left(f_\omega(X_1),\dots,(f_\omega(X_n)\right):\omega\in\Omega\}\right)\right]\label{eqn:thm3proof13}\\
       &\leq \frac{L_\phi Q_x\sqrt{3k}}{4\sqrt{n}}\label{eqn:thm3proof14}
    \end{align}
    where \eqref{eqn:thm3proof12} follows from \eqref{eqn:thm3proof1}, \eqref{eqn:thm3proof13} follows from Lemma~\ref{lemma:contraction} by substituting $\gamma(\cdot)=\phi(\sigma(\cdot))$, and \eqref{eqn:thm3proof14} follows from \cite[Proof of Corollary~1]{JiZL21}. Using a similar approach,  we obtain
    \begin{align}
        \mathcal{R}_{S_z}(\mathcal{H}^{(\psi)}_{\Omega\times\Theta})\leq \frac{L_\psi Q_z\sqrt{3(k+l-1)}}{4\sqrt{m}}\label{eqn:thm3proof15}.
    \end{align}
Substituting \eqref{eqn:thm3proof14} and \eqref{eqn:thm3proof15} into \eqref{eqn:thm3proof11} gives \eqref{eq:estimationboundrhs2}.     
\subsection{Specialization to $\alpha$-GAN}
Let $\phi_\alpha(p)=\psi_\alpha(1-p)=\frac{\alpha}{\alpha-1}\left(1-p^{\frac{\alpha-1}{\alpha}}\right)$. It is shown in \cite[Lemma~6]{sypherd2022journal} that $\phi_\alpha(\sigma(\cdot))$ is  $C_h(\alpha)$-Lipschitz in $[-h,h]$, for $h>0$, with $C_h(\alpha)$ as given in $\eqref{eq:clipalpha}$. Now using the Cauchy-Schwarz inequality and the fact that $||Ax||_2\leq ||A||_F||x||_2$, it follows that 
\begin{align}
    |f_\omega(\cdot)|\leq Q_x,\\
    |f_\omega(g_\theta(\cdot))|\leq Q_z,
\end{align}
where $Q_x:=M_k\prod_{i=1}^{k-1}(M_iR_i)B_x$ and with $Q_z$ as in \eqref{eqn:Q_zparameter}. So, we have $f_\omega(\cdot)\in[-Q_x,Q_x]$ and $f_\omega(g_\theta(\cdot))\in[-Q_z,Q_z]$. Thus, we have that $\psi_\alpha(\sigma(\cdot))$ and $\phi_\alpha(\sigma(\cdot))$ are $C_{Q_z}(\alpha)$- and $C_{Q_x}(\alpha)$-Lipschitz, respectively. Now specializing the steps \eqref{eqn:thm3proof4} and \eqref{eqn:thm3proof13} with these Lipschitz constants, we get the following bound with the substitutions $\frac{L_\phi}{4} \leftarrow C_{Q_x}(\alpha)$ and $\frac{L_\psi}{4} \leftarrow 4C_{Q_z}(\alpha)$ in \eqref{eq:estimationboundrhs2}:
\begin{align}
    d^{(\ell_\alpha)}_{\mathcal{F}_{nn}}(P_r,\hat{P}_{G_{\hat{\theta}^*}})-\inf_{\theta\in\Theta} d^{(\ell_\alpha)}_{\mathcal{F}_{nn}}(P_r,P_{G_{\theta}})
    &\leq \frac{4C_{Q_x}(\alpha) Q_x\sqrt{3k}}{\sqrt{n}}+\frac{4C_{Q_z}(\alpha) Q_z\sqrt{3(k+l-1)}}{\sqrt{m}}\nonumber\\
    &\hspace{12pt}+2\sqrt{2\log{\frac{1}{\delta}}}\left(\frac{C_{Q_x}(\alpha)Q_x}{\sqrt{n}}+\frac{C_{Q_z}(\alpha)Q_z}{\sqrt{m}}\right).
\end{align}

\section{Proof of Theorem \ref{thm:est-error-lower-bound-alpha-infinity}}
\label{appendix:est-error-lower-bound-alpha-infinity}
Let $\phi(\cdot)=-\ell_{\alpha}(1,\cdot)$ and consider the following modified version of $d^{\ell_\alpha}_{\mathcal{F}_{nn}}(\cdot,\cdot)$ (defined in \cite[eq. (13)]{kurri-2022-convergence}):
\begin{align*}
d^{\ell_\alpha}_{\mathcal{F}_{nn}}(P,Q) = \sup_{\omega \in \Omega} \Big (\mathbb{E}_{X\sim P}[\phi(D_\omega(X))]+\mathbb{E}_{X\sim Q}[\phi(1-D_\omega(X))] \Big) -2\phi(1/2),
\label{eq:dfnn-alpha-loss-modified}
\end{align*}
where
\[D_\omega(x)= \sigma\left(\mathbf{w}_k^\mathsf{T}r_{k-1}(\mathbf{W}_{d-1}r_{k-2}(\dots r_1(\mathbf{W}_1(x)))\right)\coloneqq \sigma\left(f_\omega(x)\right).\]
Taking $\alpha\to\infty$, we obtain
\begin{align} d^{\ell_\infty}_{\mathcal{F}_{nn}}(P,Q) =\sup_{\omega \in \Omega} \Big (\mathbb{E}_{X\sim P}[D_\omega(X)]-\mathbb{E}_{X\sim Q}[D_\omega(X)] \Big).
\label{eq:dfnn-infinity--alpha-loss-modified}
\end{align}
We first prove that $d^{\ell_\infty}_{\mathcal{F}_{nn}}$ is a semi-metric.
\\\noindent \textbf{Claim 1:} For any distribution pair $(P,Q)$, $d^{\ell_\infty}_{\mathcal{F}_{nn}}(P,Q)\ge0$.
\begin{proof}\let\qed\relax
Consider a discriminator which always outputs 1/2, i.e., $D_\omega(x)=1/2$ for all $x$. Note that such a neural network discriminator exists, as setting $\mathbf{w}_k=0$ results in $D_\omega(x)=\sigma(0)=0$. For this discriminator, the objective function in \eqref{eq:dfnn-infinity--alpha-loss-modified} evaluates to $1/2-1/2=0$. Since $d^{\ell_\infty}_{\mathcal{F}_{nn}}$ is a supremum over all discriminators, we have $d^{\ell_\infty}_{\mathcal{F}_{nn}}(P,Q)\ge0$.
\end{proof}
\noindent \textbf{Claim 2:} For any distribution pair $(P,Q)$, $d^{\ell_\infty}_{\mathcal{F}_{nn}}(P,Q)=d^{\ell_\infty}_{\mathcal{F}_{nn}}(Q,P)$.
\begin{proof}\let\qed\relax
\begin{align*}
d^{\ell_\infty}_{\mathcal{F}_{nn}}(P,Q)
&=\sup_{\omega \in \Omega} \Big (\mathbb{E}_{X\sim P}[D_\omega(X)]-\mathbb{E}_{X\sim Q}[D_\omega(X)] \Big)  \\
& =\sup_{\mathbf{W}_1,\dots,\mathbf{w}_k} \Big (\mathbb{E}_{X\sim P}[D_\omega(X)]-\mathbb{E}_{X\sim Q}[D_\omega(X)] \Big) \\
& \overset{(i)}{=}\sup_{\mathbf{W}_1,\dots,-\mathbf{w}_k} \Big (\mathbb{E}_{X\sim P}[\sigma\left(-f_\omega(x)\right)]-\mathbb{E}_{X\sim Q}[\sigma\left(-f_\omega(x)\right)] \Big) \\
&\overset{(ii)}{=}\sup_{\mathbf{W}_1,\dots,\mathbf{w}_k} \Big (\mathbb{E}_{X\sim P}[1-\sigma\left(f_\omega(x)\right)]-\mathbb{E}_{X\sim Q}[1-\sigma\left(f_\omega(x)\right)] \Big)
 \\
&=\sup_{\mathbf{W}_1,\dots,\mathbf{w}_k} \Big (\mathbb{E}_{X\sim Q}[\sigma\left(f_\omega(x)\right)]-\mathbb{E}_{X\sim P}[\sigma\left(f_\omega(x)\right)] \Big) \\
& = d^{\ell_\infty}_{\mathcal{F}_{nn}}(Q,P),
\end{align*}
where $(i)$ follows from replacing $\mathbf{w}_k$ with $-\mathbf{w}_k$ and $(ii)$ follows from the sigmoid property $\sigma(-x)=1-\sigma(x)$ for all $x$.
\end{proof}
\noindent \textbf{Claim 3:} For any distribution $P$, $d^{\ell_\infty}_{\mathcal{F}_{nn}}(P,P)=0$.
\begin{proof}\let\qed\relax
\begin{align*}
&d^{\ell_\infty}_{\mathcal{F}_{nn}}(P,P)  =\sup_{\omega \in \Omega} \Big (\mathbb{E}_{X\sim P}[D_\omega(X)]-\mathbb{E}_{X\sim P}[D_\omega(X)] \Big)=0.
\end{align*}
\end{proof}
\noindent \textbf{Claim 4:} For any distributions $P,Q,R$, $d^{\ell_\infty}_{\mathcal{F}_{nn}}(P,Q)\le d^{\ell_\infty}_{\mathcal{F}_{nn}}(P,R)+d^{\ell_\infty}_{\mathcal{F}_{nn}}(R,Q)$.
\begin{proof}\let\qed\relax
\begin{align*}
d^{\ell_\infty}_{\mathcal{F}_{nn}}(P,Q)
&=\sup_{\omega \in \Omega} \Big (\mathbb{E}_{X\sim P}[D_\omega(X)]-\mathbb{E}_{X\sim Q}[D_\omega(X)] \Big)  \\
& =\sup_{\omega \in \Omega} \Big (\mathbb{E}_{X\sim P}[D_\omega(X)] - \mathbb{E}_{X\sim R}[D_\omega(X)] + \mathbb{E}_{X\sim R}[D_\omega(X)]-\mathbb{E}_{X\sim Q}[D_\omega(X)] \Big)  \\
& \le\sup_{\omega \in \Omega} \Big (\mathbb{E}_{X\sim P}[D_\omega(X)] - \mathbb{E}_{X\sim R}[D_\omega(X)] \Big) +  \sup_{\omega \in \Omega} \Big (\mathbb{E}_{X\sim R}[D_\omega(X)]-\mathbb{E}_{X\sim Q}[D_\omega(X)] \Big)  \\
& = d^{\ell_\infty}_{\mathcal{F}_{nn}}(P,R)+d^{\ell_\infty}_{\mathcal{F}_{nn}}(R,Q).
\end{align*}
\end{proof}

Thus, $d^{\ell_\infty}_{\mathcal{F}_{nn}}$ is a semi-metric. The remaining part of the proof of the lower bound follows along the same lines as that of \cite[Theorem 2]{JiZL21} by an application of Fano's inequality \cite[Theorem 2.5]{tsybakov2008nonparametric} (that requires the involved divergence measure to be a semi-metric), replacing $d_{\mathcal{F}_{nn}}$ with $d^{\ell_\infty}_{\mathcal{F}_{nn}}$ and noting that the additional sigmoid activation function after the last layer in the discriminator satisfies the monotonicity assumption so that
$C(\mathcal{P}(\mathcal{X}))>0$ (for $C(\mathcal{P}(\mathcal{X}))$ defined in \eqref{eq:est-error-lower-bound-constant}).

\section{Proof of Theorem \ref{thm:alpha_D,alpha_G-GAN-saturating}}
\label{appendix:alpha_D,alpha_G-GAN-saturating}
The proof to obtain \eqref{eqn:optimaldisc-gen-alpha-GAN} is the same as that for Theorem \ref{thm:alpha-GAN}, where $\alpha=\alpha_D$. The generator's optimization problem in \eqref{eqn:gen_obj} with the optimal discriminator in \eqref{eqn:optimaldisc-gen-alpha-GAN} can be written as $\inf_{\theta\in\Theta}V_{\alpha_G}(\theta,\omega^*)$,
where
\begin{align*}
   V_{\alpha_G}(\theta,\omega^*)&=\frac{\alpha_G}{\alpha_G-1}\left[\int_\mathcal{X}\left(p_r(x)D_{\omega^*}(x)^{\frac{\alpha_G-1}{\alpha_G}}+p_{G_\theta}(x)(1-D_{\omega^*}(x))^{\frac{\alpha_G-1}{\alpha_G}}\right)dx-2\right]\\
   &=\frac{\alpha_G}{\alpha_G-1}\Bigg[\int_\mathcal{X}\Bigg(p_r(x)\left( \frac{p_r(x)^{\alpha_D}}{p_r(x)^{\alpha_D}+p_{G_\theta}(x)^{\alpha_D}}\right)^{\frac{\alpha_G-1}{\alpha_G}} + p_{G_\theta}(x)\left( \frac{p_{G_\theta}(x)^{\alpha_D}}{p_r(x)^{\alpha_D}+p_{G_\theta}(x)^{\alpha_D}}\right)^{\frac{\alpha_G-1}{\alpha_G}}\Bigg)dx-2\Bigg]\\
    &=\frac{\alpha_G}{\alpha_G-1}\left[\int_{\mathcal{X}}p_{G_\theta}(x)\left(\frac{(p_r(x)/p_{G_\theta}(x))^{\alpha_D(1-1/\alpha_G)+1}+1}{((p_r(x)/p_{G_\theta}(x))^{\alpha_D}+1)^{1-1/\alpha_G}}\right)dx-2\right] \\
    &  =   \int_\mathcal{X} p_{G_\theta}(x)f_{\alpha_D,\alpha_G}\left(\frac{p_r(x)}{p_{G_\theta}(x)}\right) dx + \frac{\alpha_G}{\alpha_G-1}\left(2^{\frac{1}{\alpha_G}}-2\right ),
\end{align*}
where $f_{\alpha_D,\alpha_G}$ is as defined in \eqref{eqn:f-alpha_d,alpha_g}. Observe that if $f_{\alpha_D,\alpha_G}$ is strictly convex, the first term in the last equality above equals an $f$-divergence which is minimized if and only if $P_r = P_{G_\theta}$. {We note that continuous extensions of $D_{f_{\alpha_D,\alpha_G}}(P_r||P_{G_\theta})$ for $\alpha_D,\alpha_G \in \{1,\infty\}$ exist and can be computed by interchanging the limit and integral following the dominated convergence theorem. In particular, as $(\alpha_D,\alpha_G) \to (1,1)$, $D_{f_{\alpha_D,\alpha_G}}(P_r || P_{G_\theta})$ recovers $D_{\text{JS}}(P_r || P_{G_\theta})$, and as $(\alpha_D,\alpha_G) \to (\infty,\infty)$, $D_{f_{\alpha_D,\alpha_G}}(P_r || P_{G_\theta})$ recovers $D_{\text{TV}}(P_r || P_{G_\theta})$. We also note that since the  $\alpha$-loss functions for both D and G have continuous extensions at 1 and $\infty$, we can obtain the same simplifications noted above by using the optimal discriminator strategies for the limiting points and the corresponding divergences for the generator's objective.}

Define the regions $R_1$ and $R_2$ as follows:
\begin{align*}
    R_1 \coloneqq \Big\{(\alpha_D,\alpha_G) \in (0,\infty]^2 \bigm\vert \alpha_D \le 1,\alpha_G > \frac{\alpha_D}{\alpha_D+1}\Big\}
\end{align*}
and 
\begin{align*}
    R_2 \coloneqq \Big\{(\alpha_D,\alpha_G) \in (0,\infty]^2 \bigm\vert \alpha_D > 1,\frac{\alpha_D}{2}< \alpha_G \le \alpha_D\Big\}.
\end{align*}
 In order to prove that $f_{\alpha_D,\alpha_G}$ is strictly convex for $(\alpha_D,\alpha_G)\in R_1\cup R_2$, we take its second derivative, which yields
\begin{align}
    f^{\prime\prime}_{\alpha_D,\alpha_G}(u)
    = A_{\alpha_D,\alpha_G}(u) \bigg[(\alpha_G+\alpha_D\alpha_G-\alpha_D)\left(u+u^{\alpha_D+\frac{\alpha_D}{\alpha_G}}\right) +(\alpha_G-\alpha_D\alpha_G)\left(u^\frac{\alpha_D}{\alpha_G}+u^{\alpha_D+1}\right)\bigg],
    \label{eq:sec_deriv_sat_sym}
\end{align}
where 
\begin{align}
    &A_{\alpha_D,\alpha_G}(u)=\frac{\alpha_D}{\alpha_G}u^{\alpha_D-\frac{\alpha_D}{\alpha_G}-2}(1+u^{\alpha_D})^{\frac{1}{\alpha_G}-3}.
    \label{eq:f-sec-deriv-mult-const}
\end{align}
Note that $A_{\alpha_D,\alpha_G}(u)> 0$ for all $u > 0$ and $\alpha_D,\alpha_G\in(0,\infty]$. Therefore, in order to ensure $f^{\prime\prime}_{\alpha_D,\alpha_G}(u)>0$ for all $u>0$ it is sufficient to have
\begin{align}
    \alpha_G+\alpha_D\alpha_G-\alpha_D > \alpha_G(\alpha_D-1)B_{\alpha_D,\alpha_G}(u),
    \label{eq:sat-main-convexity-condition}
\end{align}
where
\begin{align}
    B_{\alpha_D,\alpha_G}(u) = \frac{u^\frac{\alpha_D}{\alpha_G}+u^{\alpha_D+1}}{u+u^{\alpha_D+\frac{\alpha_D}{\alpha_G}}}
    \label{eq:B-function}
\end{align} for $u > 0$. Since $B_{\alpha_D,\alpha_G}(u) > 0$ for all $u  > 0$, the sign of the RHS of \eqref{eq:sat-main-convexity-condition} is determined by whether $\alpha_D \le 1$ or $\alpha_D > 1$. We look further into these two cases in the following:

\noindent\textbf{Case 1:} $\alpha_D \le 1$. Then $\alpha_G(\alpha_D-1)B_{\alpha_D,\alpha_G}(u) \le 0$ for all $u > 0$ and $(\alpha_D,\alpha_G)\in(0,\infty]^2$. Therefore, we need
\begin{align}
    \alpha_G(1+\alpha_D)-\alpha_D > 0 \Leftrightarrow \alpha_G > \frac{\alpha_D}{\alpha_D+1}.
\end{align}
\noindent\textbf{Case 2:} $\alpha_D > 1$. Then $\alpha_G(\alpha_D-1)B_{\alpha_D,\alpha_G}(u) > 0$ for all $u > 0$ and $(\alpha_D,\alpha_G)\in(0,\infty]^2$. In order to obtain conditions on $\alpha_D$ and $\alpha_G$, we determine the monotonicity of $B_{\alpha_D,\alpha_G}$ by finding its first derivative as follows:
\begin{align*}
    B^\prime_{\alpha_D,\alpha_G}(u) = \frac{(\alpha_G-\alpha_D)(u^{2\alpha_D}-1)+\alpha_D\alpha_G\Big(u^{\alpha_D-\frac{\alpha_D}{\alpha_G}+1}-u^{\alpha_D+\frac{\alpha_D}{\alpha_G}-1}\Big)}{\alpha_G u^{-\frac{\alpha_D}{\alpha_G}}\Big(u+u^{\alpha_D+\frac{\alpha_D}{\alpha_G}}\Big)^2}.
\end{align*}
Since the denominator of $B^\prime_{\alpha_D,\alpha_G}$ is positive for all $u>0$ and $(\alpha_D,\alpha_G)\in (0,\infty]^2$, we just need to check the sign of the numerator.
\\\textbf{Case 2a:} $\alpha_D>\alpha_G$. For $u \in (0,1)$, 
\[u^{2\alpha_D}-1 < 0 \quad \text{and} 
 \quad u^{\alpha_D-\frac{\alpha_D}{\alpha_G}+1}-u^{\alpha_D+\frac{\alpha_D}{\alpha_G}-1} >0,\]
so $B^\prime_{\alpha_D,\alpha_G}(u) > 0$. For $u > 1$, 
\[u^{2\alpha_D}-1 > 0 \quad \text{and} 
 \quad u^{\alpha_D-\frac{\alpha_D}{\alpha_G}+1}-u^{\alpha_D+\frac{\alpha_D}{\alpha_G}-1} < 0,\]
 so $B^\prime_{\alpha_D,\alpha_G}(u) < 0$. For $u=1$, $B^\prime_{\alpha_D,\alpha_G}(u) = 0$.  Hence, $B^\prime_{\alpha_D,\alpha_G}$ is strictly increasing for $u\in (0,1)$ and strictly decreasing for $u \ge 1$. Therefore, $B_{\alpha_D,\alpha_G}$ attains a maximum value of 1 at $u=1$. This means $B_{\alpha_D,\alpha_G}$ is bounded, i.e. $B_{\alpha_D,\alpha_G}\in (0,1]$ for all $u>0$. Thus, in order for \eqref{eq:sat-main-convexity-condition} to hold, it suffices to ensure that
\begin{align}
    \alpha_G+\alpha_D\alpha_G-\alpha_D > \alpha_G(\alpha_D-1) \Leftrightarrow \alpha_G > \frac{\alpha_G}{2}.
\end{align}
\textbf{Case 2b:} $\alpha_D<\alpha_G$. For $u \in (0,1)$, $u^{2\alpha_D}-1 < 0$ and $u^{\alpha_D-\frac{\alpha_D}{\alpha_G}+1}-u^{\alpha_D+\frac{\alpha_D}{\alpha_G}-1} < 0$, so $B^\prime_{\alpha_D,\alpha_G}(u) < 0$. For $u > 1$, $u^{2\alpha_D}-1 > 0$ and $u^{\alpha_D-\frac{\alpha_D}{\alpha_G}+1}-u^{\alpha_D+\frac{\alpha_D}{\alpha_G}-1} > 0$, so $B^\prime_{\alpha_D,\alpha_G}(u) > 0$. Hence, $B^\prime_{\alpha_D,\alpha_G}$ is strictly decreasing for $u\in (0,1)$ and strictly increasing for $u \ge 1$. Therefore, $B_{\alpha_D,\alpha_G}$ attains a minimum value of 1 at $u=1$. This means that $B_{\alpha_D,\alpha_G}$ is not bounded above, so it is not possible to satisfy \eqref{eq:sat-main-convexity-condition} without restricting the domain of $B_{\alpha_D,\alpha_G}$.

 Thus, for $(\alpha_D,\alpha_G)\in R_1 \cup R_2$,
\[
   V_{\alpha_G}(\theta,\omega^*)
    =D_{f_{\alpha_D,\alpha_G}}(P_r||P_{G_\theta})+\frac{\alpha_G}{\alpha_G-1}\left(2^{\frac{1}{\alpha_G}}-2\right).
\]
This yields \eqref{eqn:gen-alpha_d,alpha_g-obj}. Figure \ref{fig:convexity_regions}(a) illustrates the feasible $(\alpha_D,\alpha_G)$-region $R_1\cup R_2$. Note that $D_{f_{\alpha_D,\alpha_G}}(P||Q)$ is symmetric since
\begin{align*}
    D_{f_{\alpha_D,\alpha_G}}(Q||P)
    &= \int_\mathcal{X} p(x)f_{\alpha_D,\alpha_G}\left(\frac{q(x)}{p(x)}\right) dx \\
    &=\frac{\alpha_G}{\alpha_G-1} \left[\int_{\mathcal{X}}p(x)\left(\frac{(p(x)/q(x))^{-{\alpha_D\left(1-\frac{1}{\alpha_G}\right)}-1}+1}{((p(x)/q(x))^{-\alpha_D}+1)^{1-\frac{1}{\alpha_G}}}\right)dx-2^\frac{1}{\alpha_G}\right] \\
    &=\frac{\alpha_G}{\alpha_G-1} \left[\int_{\mathcal{X}}p(x)\left(\frac{q(x)/p(x)+(p(x)/q(x))^{\alpha_D\left(1-\frac{1}{\alpha_G}\right)}}{(1+(p(x)/q(x))^{\alpha_D})^{1-\frac{1}{\alpha_G}}}\right)dx-2^\frac{1}{\alpha_G}\right]\\
    &=\frac{\alpha_G}{\alpha_G-1}\left[\int_{\mathcal{X}}q(x)\left(\frac{1+(p(x)/q(x))^{\alpha_D\left(1-\frac{1}{\alpha_G}\right)}}{(1+(p(x)/q(x))^{\alpha_D})^{1-\frac{1}{\alpha_G}}}\right)dx-2^\frac{1}{\alpha_G}\right]\\
    & = D_{f_{\alpha_D,\alpha_G}}(P||Q).
\end{align*}
Since $f_{\alpha_D,\alpha_G}$ is strictly convex and $f_{\alpha_D,\alpha_G}(1)=0$, $D_{f_{\alpha_D,\alpha_G}}(P_r||P_{G_\theta})\geq 0$ with equality if and only if $P_r=P_{G_\theta}$. Thus, we have $V_{\alpha_G}(\theta,\omega^*)\geq \frac{\alpha_G}{\alpha_G-1}\left(2^{\frac{1}{\alpha_G}}-2\right)$ with equality if and only if $P_r=P_{G_\theta}$. 

\begin{figure}[t]
\centering
\footnotesize
\setlength{\tabcolsep}{20pt}
\begin{tabular}{@{}cc@{}}
  \includegraphics[page=1,width=0.3\linewidth]{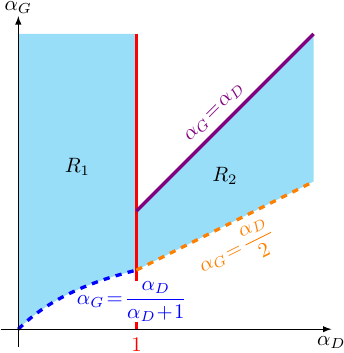}  & \raisebox{1pt}{\includegraphics[page=2,width=0.3\linewidth]{./Figures/plots_sat_and_nonsat.pdf}} \\
   (a)  & (b)
\end{tabular}
\caption{(a) Plot of regions $R_1 = \{(\alpha_D,\alpha_G) \in (0,\infty]^2 \bigm\vert \alpha_D \le 1,\alpha_G > \frac{\alpha_D}{\alpha_D+1}\}$ and $R_2 = \{(\alpha_D,\alpha_G) \in (0,\infty]^2 \bigm\vert \alpha_D > 1,\frac{\alpha_D}{2}< \alpha_G \le \alpha_D\}$ for which $f_{\alpha_D,\alpha_G}$ is strictly convex. (b) Plot of region $R_\text{NS}=\{(\alpha_D,\alpha_G)\in (0,\infty]^2 \mid \alpha_D+\alpha_G > \alpha_D\alpha_G\}$ for which $f^\text{NS}_{\alpha_D,\alpha_G}$ is strictly convex.}
\label{fig:convexity_regions}
\end{figure}

\section{Proof of Theorem \ref{thm:alpha_D,alpha_G-GAN-nonsaturating}}
\label{appendix:alpha_D,alpha_G-GAN-nonsaturating}
The generator's optimization problem in \eqref{eqn:gen_obj} with the optimal discriminator in \eqref{eqn:optimaldisc-gen-alpha-GAN} can be written as $\inf_{\theta\in\Theta}V^\text{NS}_{\alpha_G}(\theta,\omega^*)$,
where
\begin{align*}
   V^\text{NS}_{\alpha_G}(\theta,\omega^*)
   &=\frac{\alpha_G}{\alpha_G-1}\left[1-\int_\mathcal{X}\left(p_{G_\theta}(x)D_{\omega^*}(x)^{\frac{\alpha_G-1}{\alpha_G}}\right)dx\right]\\
   &=\frac{\alpha_G}{\alpha_G-1}\Bigg[1-\int_\mathcal{X}p_{G_\theta}(x)\left( \frac{p_r(x)^{\alpha_D}}{p_r(x)^{\alpha_D}+p_{G_\theta}(x)^{\alpha_D}}\right)^{\frac{\alpha_G-1}{\alpha_G}}dx\Bigg]\\
    &=\frac{\alpha_G}{\alpha_G-1}\Bigg[1-\int_\mathcal{X}p_{G_\theta}(x) \frac{(p_r(x)/p_{G_\theta}(x))^{\alpha_D(1-1/\alpha_G)}}{((p_r(x)/p_{G_\theta}(x))^{\alpha_D}+1)^{1-1/\alpha_G}}dx\Bigg] \\
    & = \int_\mathcal{X} p_{G_\theta}(x)f^\text{NS}_{\alpha_D,\alpha_G}\left(\frac{p_r(x)}{p_{G_\theta}(x)}\right) dx + \frac{\alpha_G}{\alpha_G-1}\left(1-2^{\frac{1}{\alpha_G}-1}\right),
\end{align*}
where $f^\text{NS}_{\alpha_D,\alpha_G}$ is as defined in \eqref{eqn:f-alpha_d,alpha_g-ns}. {Continuous extensions of $D_{f^\text{NS}_{\alpha_D,\alpha_G}}(P_r||P_{G_\theta})$ for $\alpha_D,\alpha_G \in \{1,\infty\}$ exist and can be computed by interchanging the limit and integral following the dominated convergence theorem.} In order to prove that $f^\text{NS}_{\alpha_D,\alpha_G}$ is strictly convex for $(\alpha_D,\alpha_G)\in R_\text{NS}= \{(\alpha_D,\alpha_G) \in (0,\infty]^2 \mid \alpha_D > \alpha_G(\alpha_D-1)\}$, we take its second derivative, which yields
\begin{align}
    f^{\prime\prime}_{\alpha_D,\alpha_G}(u) = A_{\alpha_D,\alpha_G}(u) \bigg[(\alpha_G-\alpha_D\alpha_G+\alpha_D) +\alpha_G(1+\alpha_D)u^{\alpha_D}\bigg],
    \label{eq:sec_deriv_nonsat}
\end{align}
where $A_{\alpha_D,\alpha_G}$ is defined as in \eqref{eq:f-sec-deriv-mult-const}.
Since $A_{\alpha_D,\alpha_G}(u)> 0$ for all $u > 0$ and $(\alpha_D,\alpha_G)\in(0,\infty]^2$, to ensure $f^{\prime\prime}_{\alpha_D,\alpha_G}(u)>0$ for all $u>0$ it suffices to have
\[\frac{\alpha_G-\alpha_D\alpha_G+\alpha_D}{\alpha_G(1+\alpha_D)} > -u^{\alpha_D} \]
for all $u > 0$. This is equivalent to
\[\frac{\alpha_G-\alpha_D\alpha_G+\alpha_D}{\alpha_G(1+\alpha_D)} > 0, \]
which results in the condition
\[ \alpha_D > \alpha_G(\alpha_D-1)\]
for $(\alpha_D,\alpha_G)\in(0,\infty]^2$. Thus, for $(\alpha_D,\alpha_G)\in R_\text{NS}$,
\[
   V^\text{NS}_{\alpha_G}(\theta,\omega^*)
    =D_{f^\text{NS}_{\alpha_D,\alpha_G}}(P_r||P_{G_\theta})+\frac{\alpha_G}{\alpha_G-1}\left(1-2^{\frac{1}{\alpha_G}-1}\right).
\]
This yields \eqref{eqn:gen-alpha_d,alpha_g-obj-ns}. Figure \ref{fig:convexity_regions}(b) illustrates the feasible $(\alpha_D,\alpha_G)$-region $R_{\text{NS}}$. Note that $D_{f^\text{NS}_{\alpha_D,\alpha_G}}(P||Q)$ is not symmetric since $D_{f^\text{NS}_{\alpha_D,\alpha_G}}(P||Q) \ne D_{f^\text{NS}_{\alpha_D,\alpha_G}}(Q||P)$.
Since $f^\text{NS}_{\alpha_D,\alpha_G}$ is strictly convex and $f^\text{NS}_{\alpha_D,\alpha_G}(1)=0$, $D_{f^\text{NS}_{\alpha_D,\alpha_G}}(P_r||P_{G_\theta})\geq 0$ with equality if and only if $P_r=P_{G_\theta}$. Thus, we have $V^\text{NS}_{\alpha_G}(\theta,\omega^*)\geq \frac{\alpha_G}{\alpha_G-1}\left(1-2^{\frac{1}{\alpha_G}-1}\right)$ with equality if and only if $P_r=P_{G_\theta}$. 

\section{Proof of Theorem \ref{thm:sat-gradient}}
\label{appendix:sat-gradient-proof}
\noindent\textbf{Saturating $(\alpha_D,\alpha_G)$-GANs:}

For the optimal discriminator $D_{\omega^{*}}$ defined in \eqref{eqn:optimaldisc-gen-alpha-GAN}, we first
derive $\partial D_{\omega^{*}} / \partial x$ using the quotient rule as
\begin{align}
\frac{\partial D_{\omega^{*}}}{\partial x} & =
\left(p_{r}(x)^{\alpha_{D}} + p_{G_{\theta}}(x)^{\alpha_{D}}\right)^{-2} \bigg[
(p_{r}(x)^{\alpha_{D}} + p_{G_{\theta}}(x)^{\alpha_{D}})\Big(\alpha_{D}p_{r}(x)^{\alpha_{D}-1}\frac{\partial p_{r}}{\partial x}\Big) \nonumber \\ & \qquad \qquad \qquad \qquad - p_{r}(x)^{\alpha_{D}}\Big(\alpha_{D}p_{r}(x)^{\alpha_{D}-1}\frac{\partial p_{r}}{\partial x} + \alpha_{D}p_{G_{\theta}}(x)^{\alpha_{D}-1}\frac{\partial p_{G_{\theta}}}{\partial x}\Big) \bigg] \\
& = \left(p_{r}(x)^{\alpha_{D}} + p_{G_{\theta}}(x)^{\alpha_{D}}\right)^{-2} \bigg[
\alpha_{D}p_{r}(x)^{2\alpha_{D}-1}\frac{\partial p_{r}}{\partial x}
+ \alpha_{D}p_{r}(x)^{\alpha_{D}-1}p_{G_{\theta}}(x)^{\alpha_{D}}\frac{\partial p_{r}}{\partial x} \nonumber \\
& \qquad \qquad \qquad \qquad - \alpha_{D}p_{r}(x)^{2\alpha_{D}-1}\frac{\partial p_{r}}{\partial x}
- \alpha_{D}p_{r}(x)^{\alpha_{D}}p_{G_{\theta}}(x)^{\alpha_{D}-1}\frac{\partial p_{G_{\theta}}}{\partial x} 
\bigg] \\
& = \left(p_{r}(x)^{\alpha_{D}} + p_{G_{\theta}}(x)^{\alpha_{D}}\right)^{-2}\left(\alpha_{D}p_{r}(x)^{\alpha_{D}-1}p_{G_{\theta}}(x)^{\alpha_{D}}\frac{\partial p_{r}}{\partial x} - \alpha_{D}p_{r}(x)^{\alpha_{D}}p_{G_{\theta}}(x)^{\alpha_{D}-1}\frac{\partial p_{G_{\theta}}}{\partial x}  \right) \\
& = \alpha_{D}p_{r}(x)^{\alpha_{D}}p_{G_{\theta}}(x)^{\alpha_{D}} \left(p_{r}(x)^{\alpha_{D}} + p_{G_{\theta}}(x)^{\alpha_{D}}\right)^{-2} \left(\frac{1}{p_{r}(x)}\frac{\partial p_{r}}{\partial x} - \frac{1}{p_{G_{\theta}}(x)}\frac{\partial p_{G_{\theta}}}{\partial x}\right) \\
& = \alpha_{D} D_{\omega^{*}}(x)\left(1 - D_{\omega^{*}}(x)\right) \left(\frac{1}{p_{r}(x)}\frac{\partial p_{r}}{\partial x} - \frac{1}{p_{G_{\theta}}(x)}\frac{\partial p_{G_{\theta}}}{\partial x}\right). \label{eq:dd-dx}
\end{align}
Next, we set $\mu = D_{\omega^{*}}(x)$ and derive $-\partial \ell_{\alpha_{G}}\left(0, \mu\right)/\partial \mu$ as follows
\begin{align}
    -\frac{\partial \ell_{\alpha_{G}}(0,\mu)}{\partial \mu} & = \frac{\partial}{\partial \mu} \left[ -\frac{\alpha_{G}}{\alpha_{G}-1}\left(1 - \left(1 - \mu\right)^{1 - 1/\alpha_{G}}\right)\right] \\
    & = -(1 - \mu)^{-1/\alpha_{G}}. \label{eq:dl-dd-s}
\end{align}
Lastly, to find the gradient $-\partial \ell_{\alpha_{G}}\left(0, D_{\omega^{*}}(x)\right)/\partial x$, we apply the chain rule and substitute \eqref{eq:dd-dx} and \eqref{eq:dl-dd-s}:
\begin{align}
    -\frac{\partial \ell_{\alpha_{G}}\left(0, D_{\omega^{*}}(x)\right)}{\partial x} & = -\frac{\partial \ell_{\alpha_{G}}\left(0, D_{\omega^{*}}(x)\right)}{\partial D_{\omega^{*}}} \times \frac{\partial D_{\omega^{*}}}{\partial x} \\
    & = C_{x,\alpha_{D},\alpha_{G}} \left(\frac{1}{p_{G_{\theta}}(x)}\frac{\partial p_{G_{\theta}}}{\partial x} - \frac{1}{p_{r}(x)}\frac{\partial p_{r}}{\partial x}\right),
\end{align}
where
\[C_{x,\alpha_{D},\alpha_{G}} = \alpha_{D}D_{\omega^{*}}(x)\left(1 - D_{\omega^{*}}(x)\right)^{1 - 1/\alpha_{G}},\]
or equivalently,
\[C_{x,\alpha_{D}, \alpha_{G}} = \alpha_{D} P^{(\alpha_D)}_{Y|X}(1|x) \left(1 - P^{(\alpha_D)}_{Y|X}(1|x) \right)^{1 - 1/\alpha_{G}}.\]
Since the scalar $C_{x,\alpha_{D}, \alpha_{G}}$ is positive and the only term reliant on $\alpha_{D}$ and $\alpha_{G}$ for a fixed $P_{G_\theta}$, we conclude that the direction of $-\partial \ell_{\alpha_{G}}\left(0, D_{\omega^{*}}(x)\right)/\partial x$ is independent of these parameters. See Fig.~\ref{fig:p-real-vs-c}(a) for a plot of $C_{x,\alpha_{D}, \alpha_{G}}$ as a function of $P_{Y|X}(1|x)$ for five $(\alpha_{D}, \alpha_{G})$ combinations.

\noindent\textbf{Non-saturating $(\alpha_D,\alpha_G)$-GANs:}

The proof for NS $(\alpha_D,\alpha_G)$-GANs follows similarly to that for saturating $(\alpha_D,\alpha_G)$-GANs. First, we set $\mu = D_{\omega^{*}}(x)$ and derive $\partial \ell_{\alpha_{G}}(1, \mu) /\partial \mu$ as follows:
\begin{align}
    \frac{\partial \ell_{\alpha_{G}}(1, \mu)}{\partial \mu} & = \frac{\partial}{\partial \mu} \left[\frac{\alpha_{G}}{\alpha_{G} - 1}\left(1 - \mu^{1 - 1/\alpha_{G}}\right)\right] \\
    & = -\mu^{-1/\alpha_{G}}. \label{eq:dl-dd-ns}
\end{align}
Then we derive the gradient $\partial \ell_{\alpha_{G}}\left(1, D_{\omega^{*}}(x)\right)/\partial x$ using the chain rule and substituting \eqref{eq:dl-dd-s} and \eqref{eq:dl-dd-ns}:
\begin{align}
    \frac{\partial \ell_{\alpha_{G}}\left(1, D_{\omega^{*}}(x)\right)}{\partial x} & = \frac{\partial \ell_{\alpha_{G}}\left(1, D_{\omega^{*}}(x)\right)}{\partial D_{\omega^{*}}} \times \frac{\partial D_{\omega^{*}}}{\partial x} \\
    & = C^{\text{NS}}_{x,\alpha_{D},\alpha_{G}} \left(\frac{1}{p_{G_{\theta}}(x)}\frac{\partial p_{G_{\theta}}}{\partial x} - \frac{1}{p_{r}(x)}\frac{\partial p_{r}}{\partial x}\right),
\end{align}
where \[C^{\text{NS}}_{x,\alpha_{D},\alpha_{G}} = \alpha_{D}\left(1 - D_{\omega^{*}}(x)\right)D_{\omega^{*}}(x)^{1 - 1/\alpha_{G}},\]
or equivalently,
\[C^{\text{NS}}_{x,\alpha_{D}, \alpha_{G}} = \alpha_{D} \left(1 - P^{(\alpha_D)}_{Y|X}(1|x)\right) P^{(\alpha_D)}_{Y|X}(1|x)^{1 - 1/\alpha_{G}}.\]
Since the scalar $C^{\text{NS}}_{x,\alpha_{D}, \alpha_{G}}$ is positive and the only term reliant on $\alpha_{D}$ and $\alpha_{G}$ for a fixed $P_{G_\theta}$, we conclude that the direction of $\partial \ell_{\alpha_{G}}\left(1, D_{\omega^{*}}(x)\right)/\partial x$ is independent of these parameters. See Fig.~\ref{fig:p-real-vs-c}(b) for a plot of $C^{\text{NS}}_{x,\alpha_{D}, \alpha_{G}}$ as a function of $P_{Y|X}(1|x)$ for five $(\alpha_{D}, \alpha_{G})$ combinations.

\begin{figure}[t]
\centering
\footnotesize
\setlength{\tabcolsep}{20pt}
\begin{tabular}{@{}cc@{}}
    \includegraphics[width=0.4\linewidth]{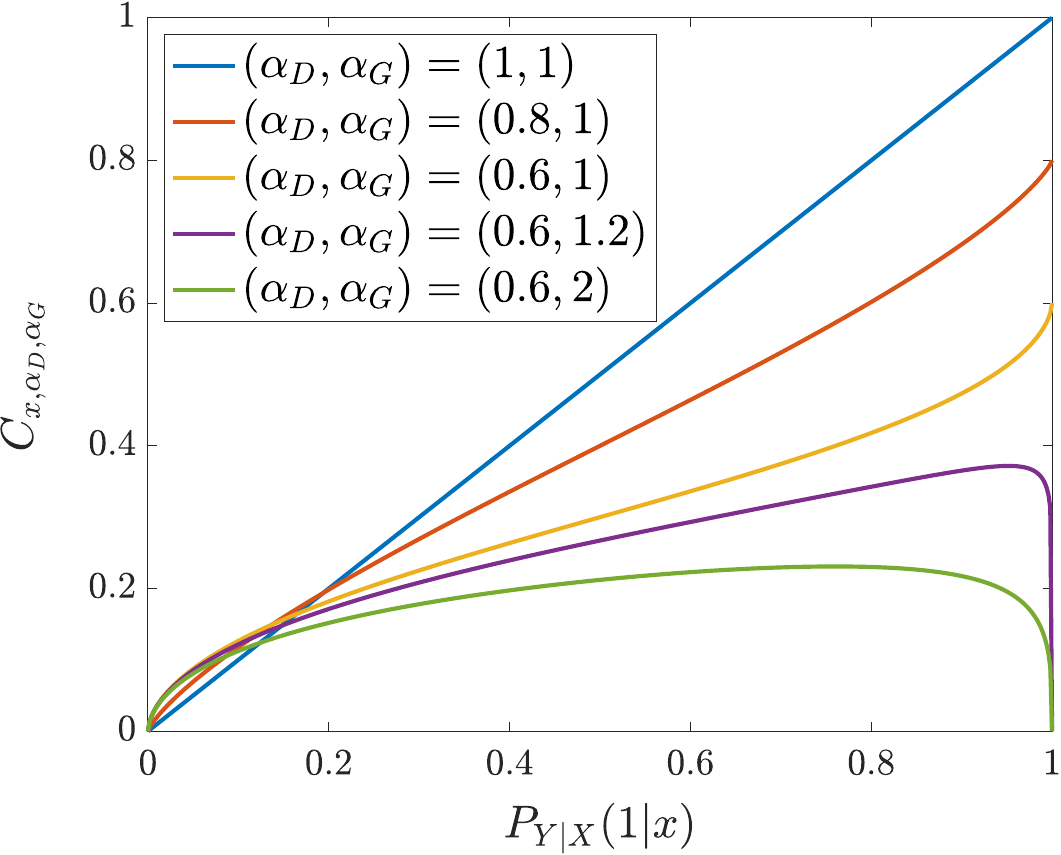}
    & {\includegraphics[width=0.4\linewidth]{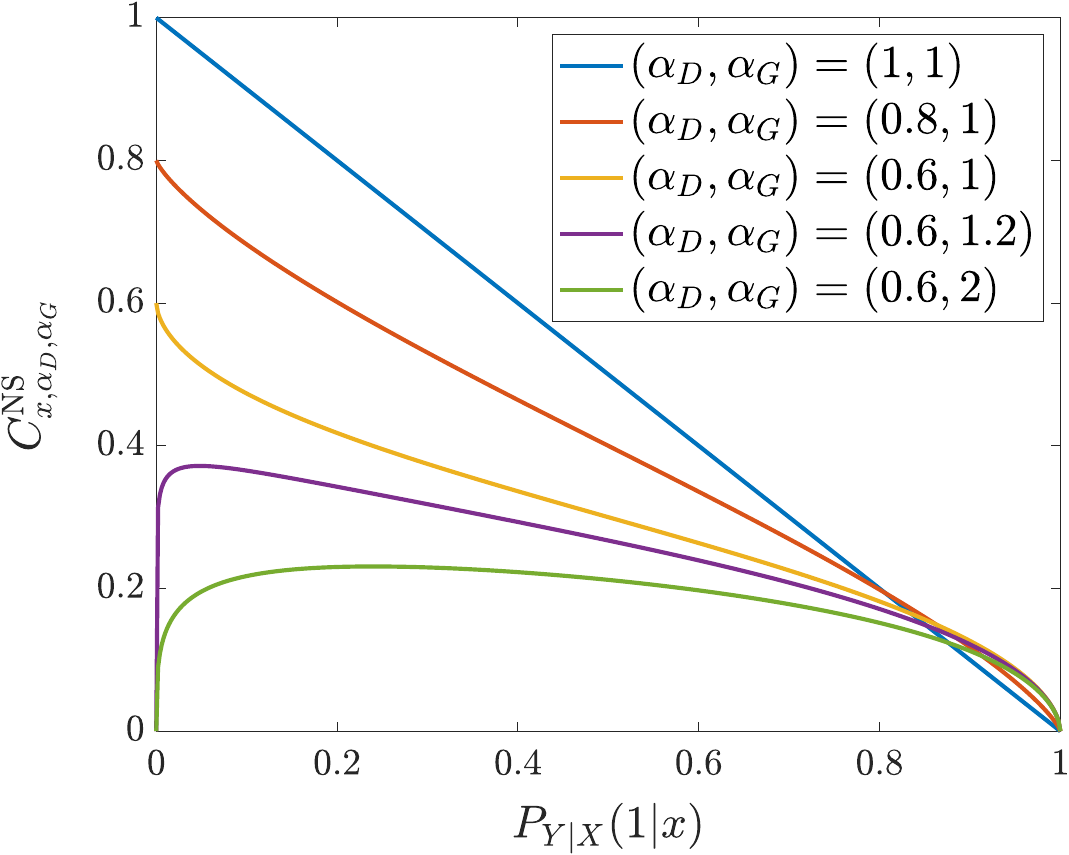}} \\
   (a)  & (b)
\end{tabular}
\caption
{(a) Plot of the gradient scalar $C_{x,\alpha_{D},\alpha_{G}}$ defined in \eqref{eq:c-sat} over the true posterior $P_{Y|X}(1|x)$ for five different saturating $(\alpha_{D},\alpha_{G})$-GANs. (b) Plot of the gradient scalar $C^{\text{NS}}_{x,\alpha_{D},\alpha_{G}}$ defined in \eqref{eq:c-nonsat} over the true posterior $P_{Y|X}(1|x)$ for five different NS $(\alpha_{D},\alpha_{G})$-GANs.}
\label{fig:p-real-vs-c}
\end{figure}

\section{Proof of Proposition \ref{prop:dual-objective-CPE-loss-GAN-strategies}}
\label{appendix:dual-objective-CPE-loss-GAN-strategies}
As noted in the proof of Theorems~\ref{thm:correspondence} and \ref{theorem:equivalence-fGAN-CPEGAN}, since $\ell_D$ is a symmetric CPE loss, the discriminator's optimization problem in \eqref{eqn:cpe-disc_obj} reduces to solving the pointwise optimization
\begin{align}
\sup_{t\in[0,1]}-u\ell_D(1,t)-\ell_D(1,1-t), \quad u \ge 0,
\end{align}
For a fixed $u \ge 0$, consider the function
\begin{align}
    g(t)=-u\ell_D(1,t)-\ell_D(1,1-t), \quad t \in [0,1].
\end{align}
To show that the optimal discriminator $D_{\omega^*}$ satisfies the implicit equation in \eqref{eqn:optimaldoisc}, we find when $g$ is maximized over $[0,1]$. Since $\ell_D(1,\cdot)$ is strictly convex and therefore $g$ is strictly concave, there exists a unique global maximum attained either at $t^*\in\{0,1\}$ or at $t^*\in (0,1)$, in which case it occurs when $g^\prime(t^*) = 0$, which yields \eqref{eq:opt-disc-cpe-implicit-eq} as follows:
\[u\ell_D^\prime(1,1-t^*)-\ell_D^\prime(t^*) = 0 \implies \ell_D^\prime(t^*) = u\ell_D^\prime(1-t^*). \]
The generator's optimization problem in \eqref{eqn:cpe-gen_obj} with the optimal discriminator $D_{\omega^*}$ satisfying \eqref{eq:opt-disc-cpe-implicit-eq} can be written as $\inf_{\theta\in\Theta}V_{\ell_G}(\theta,\omega^*)$,
where
\begin{align}
   V_{\ell_G}(\theta,\omega^*)&=\int_\mathcal{X}\left(-p_r(x)\ell_G(1,D_{\omega^*}(x))-p_{G_\theta}(x)\ell_G(1,1-D_{\omega^*}(x))\right)dx\\
    &  =   \int_\mathcal{X} p_{G_\theta}(x)\left(-\frac{p_r(x)}{p_{G_\theta}(x)}\ell_G(1,D_{\omega^*}(x)) -\ell_G(1,1-D_{\omega^*}(x))\right) dx. \label{eq:gen-obj-dual-cpe-f-divergence}
\end{align}
Let $A\left(\frac{p_r(x)}{p_{G_\theta}(x)}\right) \coloneqq D_{\omega^*}(x)$ for $x \in \mathcal{X}$ and
\begin{align}
    f(u) = -u\ell_G(1,A(u)) -\ell_G(1,1-A(u)) + 2\ell_G(1,1/2), \quad u \ge 0.
    \label{eq:f_dual_cpe}
\end{align}
Note that the additional term 2$\ell_G(1,1/2)$ in \ref{eq:f_dual_cpe} is required to satisfy $f(1)=0$, where the $1/2$ comes from the fact that $D_{\omega^*}(x)=1/2$ for any $x\in\mathcal{X}$ such that $p_r(x)=p_{G_\theta}(x)$. Then $f$ is convex by assumption, so \eqref{eq:gen-obj-dual-cpe-f-divergence} becomes $D_f(P_r||P_{G_\theta}) -2\ell_G(1,1/2)$, where $D_f(P_r||P_{G_\theta})$ is the $f$-divergence with $f$ as given in \eqref{eq:f_dual_cpe}.

\section{Proof of Theorem \ref{thm:estimationerror-upperbound-double-objective}}
\label{appendix:estimationerror-upperbound-alpha_d,alpha_g-GAN}
By adding and subtracting relevant terms, we obtain
\begin{subequations}
\begin{align} 
d_{\omega^*(P_r,P_{G_{\hat{\theta}^*}})}(P_r,{P}_{G_{\hat{\theta}^*}})&-\inf_{\theta\in\Theta} d_{\omega^*(P_r,P_{G_{\theta}})}(P_r,P_{G_{\theta}}) \nonumber \\
& = d_{\omega^*(P_r,P_{G_{\hat{\theta}^*}})}(P_r,{P}_{G_{\hat{\theta}^*}}) - d_{\omega^*(P_r,P_{G_{\hat{\theta}^*}})}(\hat{P}_r,{P}_{G_{\hat{\theta}^*}}) \label{eq:estimation-err-expanded-term1} \\
&\quad + \inf_{\theta\in\Theta} d_{\omega^*(P_r,P_{G_{\theta}})}(\hat{P}_r,P_{G_{\theta}}) - \inf_{\theta\in\Theta} d_{\omega^*(P_r,P_{G_{\theta}})}(P_r,P_{G_{\theta}}) \label{eq:estimation-err-expanded-term2} \\
& \quad + d_{\omega^*(P_r,P_{G_{\hat{\theta}^*}})}(\hat{P}_r,{P}_{G_{\hat{\theta}^*}}) - \inf_{\theta\in\Theta} d_{\omega^*(P_r,P_{G_{\theta}})}(\hat{P}_r,P_{G_{\theta}}). \label{eq:estimation-err-expanded-term3}
\end{align}
\label{eq:estimation-err-expanded}
\end{subequations}

We upper-bound \eqref{eq:estimation-err-expanded} in the following three steps. Let $\phi(\cdot) = -\ell_{G}(1,\cdot)$ and $\psi(\cdot) = -\ell_{G}(0,\cdot)$.

We first upper-bound \eqref{eq:estimation-err-expanded-term1}. Let $\omega^*(\hat{\theta}^*)=\omega^*(P_r,P_{G_{\hat{\theta}^*}})$. Using \eqref{eq:est-err-gen-obj} yields
\begin{align}
    d_{\omega^*(P_r,P_{G_{\hat{\theta}^*}})}(P_r,{P}_{G_{\hat{\theta}^*}}) &- d_{\omega^*(P_r,P_{G_{\hat{\theta}^*}})}(\hat{P}_r,{P}_{G_{\hat{\theta}^*}}) \nonumber \\
    & = \mathbb{E}_{X\sim P_r}[\phi(D_{\omega^*(\hat{\theta}^*)}(X))] +\mathbb{E}_{X\sim P_{G_{\hat{\theta}^*}}}[\psi(D_{\omega^*(\hat{\theta}^*)}(X))] \nonumber \\
    & \quad - \left(\mathbb{E}_{X\sim \hat{P}_r}[\phi(D_{\omega^*(\hat{\theta}^*)}(X))] + \mathbb{E}_{X\sim P_{G_{\hat{\theta}^*}}}[\psi(D_{\omega^*(\hat{\theta}^*)}(X))] \right) \nonumber \\
    & \le \left| \mathbb{E}_{X\sim P_r}[\phi(D_{\omega^*(\hat{\theta}^*)}(X))] - \mathbb{E}_{X\sim \hat{P}_r}[\phi(D_{\omega^*(\hat{\theta}^*)}(X))] \right| \nonumber \\
    & \le \sup_{\omega \in \Omega} \left| \mathbb{E}_{X\sim P_r}[\phi(D_{\omega}(X))] - \mathbb{E}_{X\sim \hat{P}_r}[\phi(D_{\omega}(X))] \right|.
    \label{eq:est-err-term1-bound}
\end{align}

Next, we upper-bound \eqref{eq:estimation-err-expanded-term2}. Let $\theta^* = \arg\min_{\theta \in \Theta} d_{\omega^*(P_r,P_{G_{{\theta}}})}(P_r,{P}_{G_{{\theta}}})$ and $\omega^*({\theta}^*)=\omega^*(P_r,P_{G_{{\theta}^*}})$. Then
\begin{align}
     \inf_{\theta\in\Theta} d_{\omega^*(P_r,P_{G_{{\theta}}})}(\hat{P}_r,P_{G_{\theta}}) &- \inf_{\theta\in\Theta} d_{\omega^*(P_r,P_{G_{{\theta}}})}(P_r,P_{G_{\theta}}) \nonumber \\
    & \le d_{\omega^*(\theta^*)}(\hat{P}_r,P_{G_{\theta^*}}) - d_{\omega^*(\theta^*)}(P_r,P_{G_{\theta^*}}) \nonumber \\
    & = \mathbb{E}_{X\sim \hat{P}_r}[\phi(D_{\omega^*({\theta}^*)}(X))] +\mathbb{E}_{X\sim P_{G_{{\theta}^*}}}[\psi(D_{\omega^*({\theta}^*)}(X))] \nonumber \\
    & \quad - \left(\mathbb{E}_{X\sim {P}_r}[\phi(D_{\omega^*({\theta}^*)}(X))] + \mathbb{E}_{X\sim P_{G_{{\theta}^*}}}[\psi(D_{\omega^*({\theta}^*)}(X))] \right) \nonumber \\
    & = \mathbb{E}_{X\sim \hat{P}_r}[\phi(D_{\omega^*({\theta}^*)}(X))] - \mathbb{E}_{X\sim {P}_r}[\phi(D_{\omega^*({\theta}^*)}(X))] \nonumber \\
    & \le \sup_{\omega \in \Omega} \left| \mathbb{E}_{X\sim P_r}[\phi(D_{\omega}(X))] - \mathbb{E}_{X\sim \hat{P}_r}[\phi(D_{\omega}(X))] \right|.
    \label{eq:est-err-term2-bound}
\end{align}

Lastly, we upper-bound \eqref{eq:estimation-err-expanded-term3}. Let $\Tilde{\theta} = \arg\min_{\theta \in \Theta} d_{\omega^*(P_r,P_{G_{\theta}})}(\hat{P}_r,{P}_{G_{{\theta}}})$ and $\omega^*(\Tilde{\theta})=\omega^*(P_r,P_{G_{\Tilde{\theta}}})$. Then
\begin{align}
    d_{\omega^*(P_r,P_{G_{\hat{\theta}^*}})}(\hat{P}_r,{P}_{G_{\hat{\theta}^*}}) &- \inf_{\theta\in\Theta} d_{\omega^*(P_r,P_{G_{\theta}})}(\hat{P}_r,P_{G_{\theta}}) \nonumber \\
    & = d_{\omega^*(\hat{\theta}^*)}(\hat{P}_r,{P}_{G_{\hat{\theta}^*}}) - d_{\omega^*({\Tilde{\theta}})}(\hat{P}_r,\hat{P}_{G_{\Tilde{\theta}}}) + d_{\omega^*({\Tilde{\theta}})}(\hat{P}_r,\hat{P}_{G_{\Tilde{\theta}}}) - d_{\omega^*(\Tilde{\theta})}(\hat{P}_r,P_{G_{\Tilde{\theta}}}) \nonumber \\
    & \le d_{\omega^*(\hat{\theta}^*)}(\hat{P}_r,{P}_{G_{\hat{\theta}^*}}) - d_{\omega^*({\hat{\theta}^*})}(\hat{P}_r,\hat{P}_{G_{\hat{\theta}^*}}) + d_{\omega^*({\Tilde{\theta}})}(\hat{P}_r,\hat{P}_{G_{\Tilde{\theta}}}) - d_{\omega^*(\Tilde{\theta})}(\hat{P}_r,P_{G_{\Tilde{\theta}}}) \nonumber \\
    & = \mathbb{E}_{X\sim \hat{P}_r}[\phi(D_{\omega^*(\hat{\theta}^*)}(X))] +\mathbb{E}_{X\sim P_{G_{\hat{\theta}^*}}}[\psi(D_{\omega^*(\hat{\theta}^*)}(X))] \nonumber \\
    & \quad - \left(\mathbb{E}_{X\sim \hat{P}_r}[\phi(D_{\omega^*(\hat{\theta}^*)}(X))] + \mathbb{E}_{X\sim \hat{P}_{G_{\hat{\theta}^*}}}[\psi(D_{\omega^*(\hat{\theta}^*)}(X))] \right) \nonumber \\
    & \quad + \mathbb{E}_{X\sim \hat{P}_r}[\phi(D_{\omega^*(\Tilde{\theta})}(X))] +\mathbb{E}_{X\sim \hat{P}_{G_{\Tilde{\theta}}}}[\psi(D_{\omega^*(\Tilde{\theta})}(X))] \nonumber \\
    & \quad - \left(\mathbb{E}_{X\sim \hat{P}_r}[\phi(D_{\omega^*(\Tilde{\theta})}(X))] + \mathbb{E}_{X\sim {P}_{G_{\Tilde{\theta}}}}[\psi(D_{\omega^*(\Tilde{\theta})}(X))] \right) \nonumber \\
    & = \mathbb{E}_{X\sim P_{G_{\hat{\theta}^*}}}[\psi(D_{\omega^*(\hat{\theta}^*)}(X))] - \mathbb{E}_{X\sim \hat{P}_{G_{\hat{\theta}^*}}}[\psi(D_{\omega^*(\hat{\theta}^*)}(X))] \nonumber \\
    & \quad + \mathbb{E}_{X\sim \hat{P}_{G_{\Tilde{\theta}}}}[\psi(D_{\omega^*(\Tilde{\theta})}(X))] - \mathbb{E}_{X\sim {P}_{G_{\Tilde{\theta}}}}[\psi(D_{\omega^*(\Tilde{\theta})}(X))] \nonumber \\
    & \le 2\sup_{\omega \in \Omega,\theta \in \Theta} \left| \mathbb{E}_{X\sim P_{G_{{\theta}}}}[\psi(D_{\omega}(X))] - \mathbb{E}_{X\sim \hat{P}_{G_{{\theta}}}}[\psi(D_{\omega}(X))] \right|.
    \label{eq:est-err-term3-bound}
\end{align}
Combining \eqref{eq:est-err-term1-bound}-\eqref{eq:est-err-term3-bound}, we obtain the following bound for \eqref{eq:estimation-err-expanded}:
\begin{align}
    d_{\omega^*(P_r,P_{G_{\hat{\theta}^*}})}(P_r,{P}_{G_{\hat{\theta}^*}})&-\inf_{\theta\in\Theta} d_{\omega^*(P_r,P_{G_{\theta}})}(P_r,P_{G_{\theta}}) \nonumber \\
    & \le 2\sup_{\omega \in \Omega} \Big| \mathbb{E}_{X\sim P_r}[\phi(D_{\omega}(X))] - \mathbb{E}_{X\sim \hat{P}_r}[\phi(D_{\omega}(X))] \Big| \nonumber \\
    & \quad + 2\sup_{\omega \in \Omega,\theta \in \Theta} \Big| \mathbb{E}_{X\sim P_{G_{{\theta}}}}[\psi(D_{\omega}(X))] - \mathbb{E}_{X\sim \hat{P}_{G_{{\theta}}}}[\psi(D_{\omega}(X))] \Big| \nonumber \\
    & = 2\sup_{\omega \in \Omega} \Big| \mathbb{E}_{X\sim P_r}[\phi(D_{\omega}(X))] - \frac{1}{n}\sum_{i=1}^n \phi(D_{\omega}(X_i)) \Big| \nonumber \\
    & \quad + 2\sup_{\omega \in \Omega,\theta \in \Theta} \Big| \mathbb{E}_{X\sim P_{G_{{\theta}}}}[\psi(D_{\omega}(X))] - \frac{1}{m}\sum_{j=1}^m \psi(D_{\omega}(X_j)) \Big|.
    \label{eq:est-err-proof-bound}
\end{align}
Note that \eqref{eq:est-err-proof-bound} is exactly the same bound as that in \eqref{eq:estimationboundrhs2}. Hence, the remainder of the proof follows from the proof of Theorem \ref{thm:estimationerror-upperbound}, where $\phi(\cdot) \coloneqq -\ell_{G}(1,\cdot)$ and $\psi(\cdot) \coloneqq -\ell_{G}(0,\cdot)$. The specialization to $(\alpha_D,\alpha_D)$-GANs follows from setting $\ell_D=\ell_{\alpha_D}$ and $\ell_G=\ell_{\alpha_G}$.


\section{Additional Experimental Results}
\label{appendix:experimental-details-results}

\subsection{Brief Overview of LSGAN}
\label{subsec:lsgan-overview}
The Least Squares GAN (LSGAN) is a dual-objective min-max game introduced in \cite{Mao_2017_LSGAN}. The LSGAN objective functions, as the name suggests, involve squared loss functions for D and G which are written as
\begin{align}
    \inf_{\omega\in\Omega} \; \frac{1}{2}\Big(\mathbb{E}_{X\sim P_r}[(D_\omega(X)-b)^2]+\mathbb{E}_{X\sim P_{G_\theta}}[(D_\omega(X)-a)^2]\Big) \nonumber \\
    \inf_{\theta\in\Theta} \; \frac{1}{2}\Big(\mathbb{E}_{X\sim P_r}[(D_\omega(X)-c)^2]+\mathbb{E}_{X\sim P_{G_\theta}}[(D_\omega(X)-c)^2]\Big).
    \label{eq:lsgan-objectives}
\end{align}
For appropriately chosen values of the parameters $a$, $b$, and $c$,
\eqref{eq:lsgan-objectives} reduces to minimizing the Pearson $\chi^2$-divergence between $P_r+P_{G_\theta}$ and $2P_{G_\theta}$. As done in the original paper \cite{Mao_2017_LSGAN}, we use $a=0$, $b=1$ and $c=1$ for our experiments to make fair comparisons. The authors refer to this choice of parameters as the 0-1 binary coding scheme. 

\subsection{2D Gaussian Mixture Ring}
\label{appendix:2d-ring}

In Tables \ref{table:2d-ring-sat-success-rates} and \ref{table:2d-ring-sat-failure-rates}, we
report the success (8/8 mode coverage) and failure (0/8 mode coverage) rates over 200 seeds for a grid of $(\alpha_{D}, \alpha_{G})$ combinations for the  \emph{saturating} setting. Compared to the vanilla GAN performance, we find that tuning $\alpha_{D}$ below 1 leads to a greater success rate and lower failure rate. However, in this saturating loss setting, we find that tuning $\alpha_{G}$ away from 1 has no significant impact on GAN performance.

\begin{table*}[ht]
\centering
\caption{Success rates for 2D-ring with the saturating $(\alpha_{D}, \alpha_{G})$-GAN over 200 seeds, with top 4 combinations emboldened.
}
\label{table:2d-ring-sat-success-rates}
\renewcommand{\arraystretch}{1.2}
\centering
\begin{sc}
\begin{tabular}{cl|llllll|}
\cline{3-8}
\multicolumn{2}{c|}{\multirow{2}{*}{\begin{tabular}[c]{@{}c@{}}\% of success \\ (8/8 modes)\end{tabular}}} & \multicolumn{6}{c|}{$\alpha_{D}$} \\ \cline{3-8} 
\multicolumn{2}{c|}{} & 0.5 & 0.6 & 0.7 & 0.8 & 0.9 & 1.0 \\ \hline
\multicolumn{1}{|c|}{\multirow{4}{*}{$\alpha_{G}$}} & 0.9 & \multicolumn{1}{l|}{73} & \multicolumn{1}{l|}{79} & \multicolumn{1}{l|}{69} & \multicolumn{1}{l|}{60} & \multicolumn{1}{l|}{46} & 34 \\ \cline{3-8} 
\multicolumn{1}{|c|}{} & 1.0 & \multicolumn{1}{l|}{\textbf{80}} & \multicolumn{1}{l|}{\textbf{79}} & \multicolumn{1}{l|}{74} & \multicolumn{1}{l|}{68} & \multicolumn{1}{l|}{54} & 47 \\ \cline{3-8} 
\multicolumn{1}{|c|}{} & 1.1 & \multicolumn{1}{l|}{\textbf{79}} & \multicolumn{1}{l|}{77} & \multicolumn{1}{l|}{68} & \multicolumn{1}{l|}{70} & \multicolumn{1}{l|}{59} & 47 \\ \cline{3-8} 
\multicolumn{1}{|c|}{} & 1.2 & \multicolumn{1}{l|}{\textbf{75}} & \multicolumn{1}{l|}{74} & \multicolumn{1}{l|}{71} & \multicolumn{1}{l|}{65} & \multicolumn{1}{l|}{57} & 46 \\ \hline
\end{tabular}
\end{sc}
\end{table*}
\begin{table*}[ht]
\centering
\caption{Failure rates for 2D-ring with the saturating $(\alpha_{D}, \alpha_{G})$-GAN over 200 seeds, with top 3 combinations emboldened.
}
\label{table:2d-ring-sat-failure-rates}
\renewcommand{\arraystretch}{1.2}
\centering
\begin{sc}
\begin{tabular}{cl|llllll|}
\cline{3-8}
\multicolumn{2}{c|}{\multirow{2}{*}{\begin{tabular}[c]{@{}c@{}}\% of failure \\ (0/8 modes)\end{tabular}}} & \multicolumn{6}{c|}{$\alpha_{D}$} \\ \cline{3-8} 
\multicolumn{2}{c|}{} & 0.5 & 0.6 & 0.7 & 0.8 & 0.9 & 1.0 \\ \hline
\multicolumn{1}{|c|}{\multirow{4}{*}{$\alpha_{G}$}} & 0.9 & \multicolumn{1}{l|}{11} & \multicolumn{1}{l|}{10} & \multicolumn{1}{l|}{12} & \multicolumn{1}{l|}{13} & \multicolumn{1}{l|}{29} & 49 \\ \cline{3-8} 
\multicolumn{1}{|c|}{} & 1.0 & \multicolumn{1}{l|}{\textbf{5}} & \multicolumn{1}{l|}{\textbf{5}} & \multicolumn{1}{l|}{7} & \multicolumn{1}{l|}{8} & \multicolumn{1}{l|}{16} & 30 \\ \cline{3-8} 
\multicolumn{1}{|c|}{} & 1.1 & \multicolumn{1}{l|}{7} & \multicolumn{1}{l|}{9} & \multicolumn{1}{l|}{13} & \multicolumn{1}{l|}{12} & \multicolumn{1}{l|}{13} & 26 \\ \cline{3-8} 
\multicolumn{1}{|c|}{} & 1.2 & \multicolumn{1}{l|}{9} & \multicolumn{1}{l|}{\textbf{5}} & \multicolumn{1}{l|}{9} & \multicolumn{1}{l|}{12} & \multicolumn{1}{l|}{17} & 31 \\ \hline
\end{tabular}
\end{sc}
\end{table*}
\begin{table*}[h!]
\caption{Success rates for 2D-ring with the NS $(\alpha_{D}, \alpha_{G})$-GAN over 200 seeds, with top 5 combinations emboldened.
}
\label{table:2d-ring-ns-success-rates}
\renewcommand{\arraystretch}{1.2}
\centering
\begin{small}
\begin{sc}
\begin{tabular}{cl|cccccccc|}
\cline{3-10}
\multicolumn{2}{c|}{\multirow{2}{*}{\begin{tabular}[c]{@{}c@{}}\% of success \\ (8/8 modes)\end{tabular}}} & \multicolumn{8}{c|}{$\alpha_{D}$} \\ \cline{3-10} 
\multicolumn{2}{c|}{} & \multicolumn{1}{l}{0.5} & \multicolumn{1}{l}{0.6} & \multicolumn{1}{l}{0.7} & \multicolumn{1}{l}{0.8} & \multicolumn{1}{l}{0.9} & \multicolumn{1}{l}{1.0} & \multicolumn{1}{l}{1.1} & \multicolumn{1}{l|}{1.2} \\ \hline
\multicolumn{1}{|c|}{\multirow{6}{*}{$\alpha_{G}$}} & 0.8 & \multicolumn{1}{c|}{35} & \multicolumn{1}{c|}{24} & \multicolumn{1}{c|}{19} & \multicolumn{1}{c|}{19} & \multicolumn{1}{c|}{14} & \multicolumn{1}{c|}{16} & \multicolumn{1}{c|}{18} & 10 \\ \cline{3-10} 
\multicolumn{1}{|c|}{} & 0.9 & \multicolumn{1}{c|}{\textbf{39}} & \multicolumn{1}{c|}{37} & \multicolumn{1}{c|}{19} & \multicolumn{1}{c|}{22} & \multicolumn{1}{c|}{16} & \multicolumn{1}{c|}{20} & \multicolumn{1}{c|}{19} & 21 \\ \cline{3-10} 
\multicolumn{1}{|c|}{} & 1.0 & \multicolumn{1}{c|}{34} & \multicolumn{1}{c|}{35} & \multicolumn{1}{c|}{29} & \multicolumn{1}{c|}{28} & \multicolumn{1}{c|}{26} & \multicolumn{1}{c|}{22} & \multicolumn{1}{c|}{20} & 32 \\ \cline{3-10} 
\multicolumn{1}{|c|}{} & 1.1 & \multicolumn{1}{c|}{\textbf{40}} & \multicolumn{1}{c|}{36} & \multicolumn{1}{c|}{31} & \multicolumn{1}{c|}{22} & \multicolumn{1}{c|}{24} & \multicolumn{1}{c|}{15} & \multicolumn{1}{c|}{23} & 25 \\ \cline{3-10} 
\multicolumn{1}{|c|}{} & 1.2 & \multicolumn{1}{c|}{\textbf{45}} & \multicolumn{1}{c|}{38} & \multicolumn{1}{c|}{34} & \multicolumn{1}{c|}{25} & \multicolumn{1}{c|}{26} & \multicolumn{1}{c|}{28} & \multicolumn{1}{c|}{20} & 22 \\ \cline{3-10} 
\multicolumn{1}{|c|}{} & 1.3 & \multicolumn{1}{c|}{\textbf{44}} & \multicolumn{1}{c|}{\textbf{39}} & \multicolumn{1}{c|}{26} & \multicolumn{1}{c|}{28} & \multicolumn{1}{c|}{28} & \multicolumn{1}{c|}{25} & \multicolumn{1}{c|}{31} & 29 \\ \hline
\end{tabular}
\end{sc}
\end{small}
\end{table*}

In Table \ref{table:2d-ring-ns-success-rates}, we detail the success rates for the NS setting. We note that for this dataset, no failures, and therefore, no vanishing/exploding gradients, occurred in the NS setting. In particular, we find that the $(0.5,1.2)$-GAN doubles the success rate of the vanilla $(1,1)$-GAN, which is more susceptible to mode collapse as illustrated in Figure \ref{fig:NS}. We also find that LSGAN achieves a success rate of 32.5\%, which is greater than vanilla GAN but less than the best-performing $(\alpha_{D}, \alpha_{G})$-GAN.

\subsection{Celeb-A \& LSUN Classroom}

The discriminator and generator architectures used for the Celeb-A and LSUN Classroom datasets are described in Tables \ref{tab:arch-celeba} and \ref{tab:arch-lsun} respectively. Each architecture consists of four CNN layers, with parameters such as kernel size (i.e., size of the filter, denoted as ``Kernel"), stride (the amount by which the filter moves), and the activation functions applied to the layer outputs. Zero padding is also assumed. In both tables, ``BN" represents batch normalization, a technique that normalizes the inputs to each layer using a batch of samples during model training. Batch normalization is commonly employed in deep learning to prevent cumulative floating point errors and overflows, and to ensure that all features remain within a similar range. This technique serves as a computational tool to address vanishing and/or exploding gradients.

\begin{table*}[]
\caption{Discriminator and generator architectures for Celeb-A.\\ The final sigmoid activation layer is removed for the LSGAN discriminator.}
\label{tab:arch-celeba}
\renewcommand{\arraystretch}{1.3}
\centering
\footnotesize
\begin{sc}
\resizebox{\textwidth}{!}{
\begin{tabular}{|cccccc||cccccc|}
\hline
\multicolumn{6}{|c||}{Discriminator} & \multicolumn{6}{c|}{Generator} \\ \hline
\multicolumn{1}{|c|}{Layer} & \multicolumn{1}{c|}{Output size} & \multicolumn{1}{c|}{Kernel} & \multicolumn{1}{c|}{Stride} & \multicolumn{1}{c|}{BN} & Activation & \multicolumn{1}{c|}{Layer} & \multicolumn{1}{c|}{Output size} & \multicolumn{1}{c|}{Kernel} & \multicolumn{1}{c|}{Stride} & \multicolumn{1}{c|}{BN} & Activation \\ \hline
\multicolumn{1}{|c|}{Input} & \multicolumn{1}{c|}{$3 \times 64 \times 64$} & \multicolumn{1}{c|}{} & \multicolumn{1}{c|}{} & \multicolumn{1}{c|}{} & Leaky ReLU & \multicolumn{1}{c|}{Input} & \multicolumn{1}{c|}{$100 \times 1 \times 1$} & \multicolumn{1}{c|}{} & \multicolumn{1}{c|}{} & \multicolumn{1}{c|}{} & ReLU \\
\multicolumn{1}{|c|}{Convolution} & \multicolumn{1}{c|}{$64 \times 32 \times 32$} & \multicolumn{1}{c|}{$4 \times 4$} & \multicolumn{1}{c|}{2} & \multicolumn{1}{c|}{Yes} & Leaky ReLU & \multicolumn{1}{c|}{ConvTranspose} & \multicolumn{1}{c|}{$512 \times 4 \times 4$} & \multicolumn{1}{c|}{$4 \times 4$} & \multicolumn{1}{c|}{2} & \multicolumn{1}{c|}{Yes} & ReLU \\
\multicolumn{1}{|c|}{Convolution} & \multicolumn{1}{c|}{$128 \times 16 \times 16$} & \multicolumn{1}{c|}{$4 \times 4$} & \multicolumn{1}{c|}{2} & \multicolumn{1}{c|}{Yes} & Leaky ReLU & \multicolumn{1}{c|}{ConvTranspose} & \multicolumn{1}{c|}{$256 \times 8 \times 8$} & \multicolumn{1}{c|}{$4 \times 4$} & \multicolumn{1}{c|}{2} & \multicolumn{1}{c|}{Yes} & ReLU \\
\multicolumn{1}{|c|}{Convolution} & \multicolumn{1}{c|}{$256 \times 8 \times 8$} & \multicolumn{1}{c|}{$4 \times 4$} & \multicolumn{1}{c|}{2} & \multicolumn{1}{c|}{Yes} & Leaky ReLU & \multicolumn{1}{c|}{ConvTranspose} & \multicolumn{1}{c|}{$128 \times 16 \times 16$} & \multicolumn{1}{c|}{$4 \times 4$} & \multicolumn{1}{c|}{2} & \multicolumn{1}{c|}{Yes} & ReLU \\
\multicolumn{1}{|c|}{Convolution} & \multicolumn{1}{c|}{$512 \times 4 \times 4$} & \multicolumn{1}{c|}{$4 \times 4$} & \multicolumn{1}{c|}{2} & \multicolumn{1}{c|}{Yes} & Leaky ReLU & \multicolumn{1}{c|}{ConvTranspose} & \multicolumn{1}{c|}{$64 \times 32 \times 32$} & \multicolumn{1}{c|}{$4 \times 4$} & \multicolumn{1}{c|}{2} & \multicolumn{1}{c|}{Yes} & ReLU \\
\multicolumn{1}{|c|}{Convolution} & \multicolumn{1}{c|}{$1 \times 1 \times 1$} & \multicolumn{1}{c|}{$4 \times 4$} & \multicolumn{1}{c|}{2} & \multicolumn{1}{c|}{} & Sigmoid & \multicolumn{1}{c|}{ConvTranspose} & \multicolumn{1}{c|}{$3 \times 64 \times 64$} & \multicolumn{1}{c|}{$4 \times 4$} & \multicolumn{1}{c|}{2} & \multicolumn{1}{c|}{} & Tanh \\ \hline
\end{tabular}}
\end{sc}
\end{table*}

\begin{table*}[]
\caption{Discriminator and generator architectures for LSUN Classroom.\\The final sigmoid activation layer is removed for the LSGAN discriminator.}
\label{tab:arch-lsun}
\renewcommand{\arraystretch}{1.3}
\centering
\footnotesize
\begin{sc}
\resizebox{\textwidth}{!}{%
\begin{tabular}{|cccccc||cccccc|}
\hline
\multicolumn{6}{|c||}{Discriminator} & \multicolumn{6}{c|}{Generator} \\ \hline
\multicolumn{1}{|c|}{Layer} & \multicolumn{1}{c|}{Output size} & \multicolumn{1}{c|}{Kernel} & \multicolumn{1}{c|}{Stride} & \multicolumn{1}{c|}{BN} & Activation & \multicolumn{1}{c|}{Layer} & \multicolumn{1}{c|}{Output size} & \multicolumn{1}{c|}{Kernel} & \multicolumn{1}{c|}{Stride} & \multicolumn{1}{c|}{BN} & Activation \\ \hline
\multicolumn{1}{|c|}{Input} & \multicolumn{1}{c|}{$3 \times 112 \times 112$} & \multicolumn{1}{c|}{} & \multicolumn{1}{c|}{} & \multicolumn{1}{c|}{} & Leaky ReLU & \multicolumn{1}{c|}{Input} & \multicolumn{1}{c|}{$100 \times 1 \times 1$} & \multicolumn{1}{c|}{} & \multicolumn{1}{c|}{} & \multicolumn{1}{c|}{} & ReLU \\
\multicolumn{1}{|c|}{Convolution} & \multicolumn{1}{c|}{$64 \times 56 \times 56$} & \multicolumn{1}{c|}{$4 \times 4$} & \multicolumn{1}{c|}{2} & \multicolumn{1}{c|}{Yes} & Leaky ReLU & \multicolumn{1}{c|}{ConvTranspose} & \multicolumn{1}{c|}{$512 \times 7 \times 7$} & \multicolumn{1}{c|}{$7 \times 7$} & \multicolumn{1}{c|}{2} & \multicolumn{1}{c|}{Yes} & ReLU \\
\multicolumn{1}{|c|}{Convolution} & \multicolumn{1}{c|}{$128 \times 28 \times 28$} & \multicolumn{1}{c|}{$4 \times 4$} & \multicolumn{1}{c|}{2} & \multicolumn{1}{c|}{Yes} & Leaky ReLU & \multicolumn{1}{c|}{ConvTranspose} & \multicolumn{1}{c|}{$256 \times 14 \times 14$} & \multicolumn{1}{c|}{$4 \times 4$} & \multicolumn{1}{c|}{2} & \multicolumn{1}{c|}{Yes} & ReLU \\
\multicolumn{1}{|c|}{Convolution} & \multicolumn{1}{c|}{$256 \times 14 \times 14$} & \multicolumn{1}{c|}{$4 \times 4$} & \multicolumn{1}{c|}{2} & \multicolumn{1}{c|}{Yes} & Leaky ReLU & \multicolumn{1}{c|}{ConvTranspose} & \multicolumn{1}{c|}{$128 \times 28 \times 28$} & \multicolumn{1}{c|}{$4 \times 4$} & \multicolumn{1}{c|}{2} & \multicolumn{1}{c|}{Yes} & ReLU \\
\multicolumn{1}{|c|}{Convolution} & \multicolumn{1}{c|}{$512 \times 7 \times 7$} & \multicolumn{1}{c|}{$4 \times 4$} & \multicolumn{1}{c|}{2} & \multicolumn{1}{c|}{Yes} & Leaky ReLU & \multicolumn{1}{c|}{ConvTranspose} & \multicolumn{1}{c|}{$64 \times 56 \times 56$} & \multicolumn{1}{c|}{$4 \times 4$} & \multicolumn{1}{c|}{2} & \multicolumn{1}{c|}{Yes} & ReLU \\
\multicolumn{1}{|c|}{Convolution} & \multicolumn{1}{c|}{$1 \times 1 \times 1$} & \multicolumn{1}{c|}{$7 \times 7$} & \multicolumn{1}{c|}{2} & \multicolumn{1}{c|}{} & Sigmoid & \multicolumn{1}{c|}{ConvTranspose} & \multicolumn{1}{c|}{$3 \times 112 \times 112$} & \multicolumn{1}{c|}{$4 \times 4$} & \multicolumn{1}{c|}{2} & \multicolumn{1}{c|}{} & Tanh \\ \hline
\end{tabular}}
\end{sc}
\end{table*}

In Table \ref{table:celeba-lsun-stability}, we collate the FID results for both datasets as a function of the learning rates. This table captures the percentage (out of 50 seeds) of FID scores below a desired threshold, which is 80 for the CELEB-A dataset and 800 for the LSUN Classroom dataset. 

We first focus on the CELEB-A dataset: Table \ref{table:celeba-lsun-stability} demonstrates that for a learning rate of $1 \times 10^{-4}$, all GANs (vanilla, different $(\alpha_D,\alpha_G)$-GANs, and LSGANs) achieve an FID score below 80 at least 93\% of the time. However, the instability of vanilla GAN is also evident in Table \ref{table:celeba-lsun-stability}, where for a slightly higher learning rate of $6 \times 10^{-4}$, the $(1,1)$-GAN achieves an FID score below 80 only 60\% of the time whereas at least one $(\alpha_D,\alpha_G=1)$-GAN consistently performs better than 76\% over all chosen learning rates. We observe that tuning $\alpha_{D}$ below 1 contributes to stabilizing the FID scores over the 50 seeds while maintaining relatively low scores on average. This stability is emphasized in Table \ref{table:celeba-lsun-stability}, in particular for the $(0.7,1)$-GAN, as it achieves an FID score below 80 at least 80\% of the time for 7 out of the 10 the learning rates.

Table \ref{table:celeba-lsun-stability} also illustrates similar results for the LSUN Classroom dataset. 
However, increasing it to $2 \times 10^{-4}$ leads to instability in the vanilla $(1,1)$-GAN across 50 seeds. 

\begin{table*}[h]
\caption{Percentage out of 50 seeds of FID scores below 80 (Celeb-A) or 800 (LSUN Classroom) for each combination of $(\alpha_{D}, \alpha_{G})$-GAN and learning rate, trained for 100 epochs. Best results for each dataset and learning rate are \textbf{emboldened}.}
\label{table:celeba-lsun-stability}
\renewcommand{\arraystretch}{1.3}
\centering
\begin{small}
\begin{sc}
\begin{tabular}{|c||ccccccccccccc|}
\hline
\multirow{2}{*}{\begin{tabular}[c]{@{}c@{}}GAN\end{tabular}} & \multicolumn{8}{c||}{Celeb-A} & \multicolumn{5}{c|}{LSUN Classroom} \\ \cline{2-14} 
 & \multicolumn{13}{c|}{Learning rate ($\times 10^{-4}$)} \\ \hline
$(\alpha_{D}, \alpha_{G})$ & \multicolumn{1}{c|}{1} & \multicolumn{1}{c|}{2} & \multicolumn{1}{c|}{5} & \multicolumn{1}{c|}{6} & \multicolumn{1}{c|}{7} & \multicolumn{1}{c|}{8} & \multicolumn{1}{c|}{9} & \multicolumn{1}{c||}{10} & \multicolumn{1}{c|}{1} & \multicolumn{1}{c|}{2} & \multicolumn{1}{c|}{3} & \multicolumn{1}{c|}{4} & 5 \\ \hline
$(1,1)$ & \multicolumn{1}{c|}{\textbf{100}} & \multicolumn{1}{c|}{93.2} & \multicolumn{1}{c|}{82.6} & \multicolumn{1}{c|}{59.5} & \multicolumn{1}{c|}{58.5} & \multicolumn{1}{c|}{39.0} & \multicolumn{1}{c|}{53.7} & \multicolumn{1}{c||}{54.8} & \multicolumn{1}{c|}{92.0} & \multicolumn{1}{c|}{36.2} & \multicolumn{1}{c|}{12.5} & \multicolumn{1}{c|}{13.0} & 12.2 \\ \hline
$(0.9,1)$ & \multicolumn{1}{c|}{\textbf{100}} & \multicolumn{1}{c|}{95.2} & \multicolumn{1}{c|}{78.3} & \multicolumn{1}{c|}{72.3} & \multicolumn{1}{c|}{81.4} & \multicolumn{1}{c|}{66.7} & \multicolumn{1}{c|}{74.4} & \multicolumn{1}{c||}{46.5} & \multicolumn{1}{c|}{76.0} & \multicolumn{1}{c|}{53.1} & \multicolumn{1}{c|}{22.2} & \multicolumn{1}{c|}{17.0} & 22.2 \\ \hline
$(0.8,1)$ & \multicolumn{1}{c|}{97.8} & \multicolumn{1}{c|}{\textbf{97.6}} & \multicolumn{1}{c|}{\textbf{88.9}} & \multicolumn{1}{c|}{82.2} & \multicolumn{1}{c|}{81.4} & \multicolumn{1}{c|}{72.1} & \multicolumn{1}{c|}{68.4} & \multicolumn{1}{c||}{75.6} & \multicolumn{1}{c|}{88.5} & \multicolumn{1}{c|}{60.8} & \multicolumn{1}{c|}{36.2} & \multicolumn{1}{c|}{27.9} & 29.2 \\ \hline
$(0.7,1)$ & \multicolumn{1}{c|}{\textbf{100}} & \multicolumn{1}{c|}{90.7} & \multicolumn{1}{c|}{\textbf{88.9}} & \multicolumn{1}{c|}{\textbf{91.5}} & \multicolumn{1}{c|}{\textbf{86.4}} & \multicolumn{1}{c|}{\textbf{81.2}} & \multicolumn{1}{c|}{67.6} & \multicolumn{1}{c||}{\textbf{80.0}} & \multicolumn{1}{c|}{90.2} & \multicolumn{1}{c|}{80.4} & \multicolumn{1}{c|}{78.4} & \multicolumn{1}{c|}{67.4} & 55.1 \\ \hline
$(0.6,1)$ & \multicolumn{1}{c|}{97.8} & \multicolumn{1}{c|}{93.0} & \multicolumn{1}{c|}{88.4} & \multicolumn{1}{c|}{76.6} & \multicolumn{1}{c|}{84.6} & \multicolumn{1}{c|}{75.6} & \multicolumn{1}{c|}{\textbf{76.9}} & \multicolumn{1}{c||}{69.2} & \multicolumn{1}{c|}{\textbf{95.7}} & \multicolumn{1}{c|}{\textbf{90.4}} & \multicolumn{1}{c|}{\textbf{85.1}} & \multicolumn{1}{c|}{\textbf{78.3}} & \textbf{66.0} \\ \hline
\end{tabular}
\end{sc}
\end{small}
\vskip -0.1in
\end{table*}

\begin{figure*}[t]
    \centering
    \footnotesize
\setlength{\tabcolsep}{1pt}
\begin{tabular}{@{}cc@{}}
  \includegraphics[height=3.5cm]{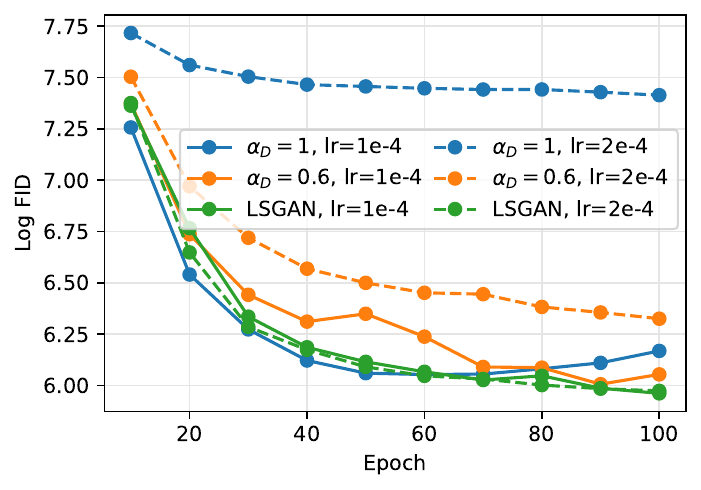}  & \includegraphics[height=4cm]{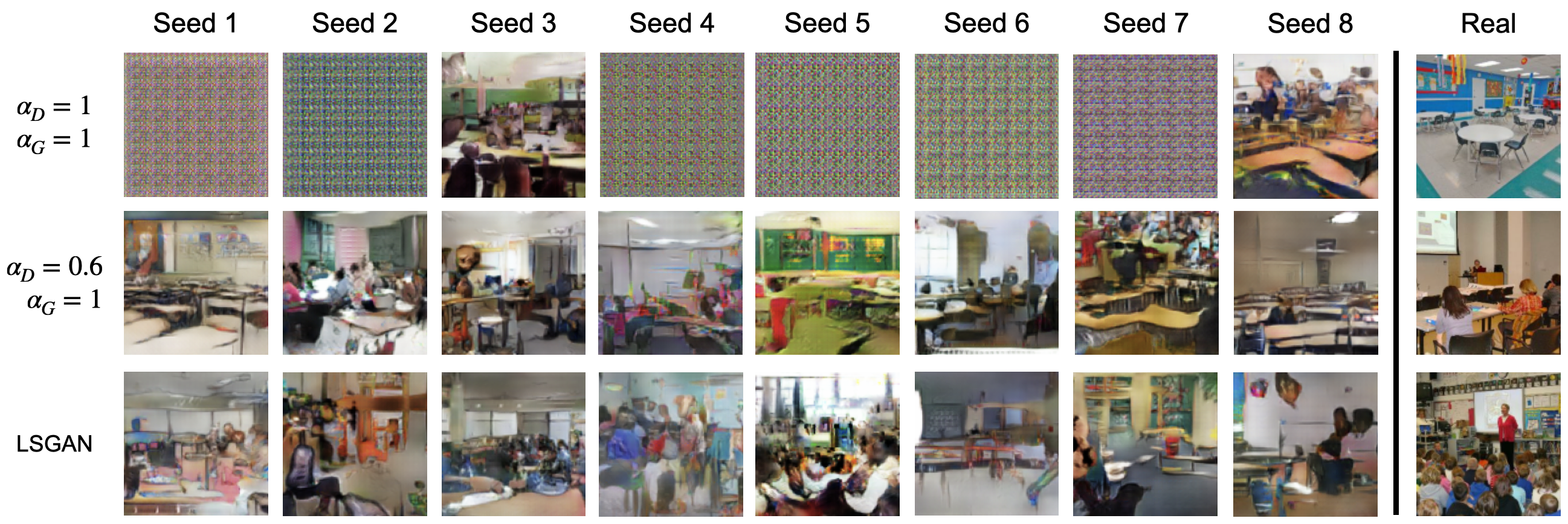} \\
   (a)  & (b)
\end{tabular}
\caption{(a) Log-scale plot of \textbf{LSUN Classroom} FID scores over training epochs in steps of 10 up to 100 total, for three noteworthy GANs-- $(1,1)$-GAN (vanilla), $(0.6,1)$-GAN, and LSGAN-- and for two similar learning rates-- $1 \times 10^{-4}$ and $2 \times 10^{-4}$. Results show that the vanilla GAN performance is very sensitive to learning rate choice as the difference between training with $1 \times 10^{-4}$ and $2 \times 10^{-4}$ is drastic. On the other hand, the other two GANs achieve consistently lower FIDs, with the LSGAN performing the best. (b) Generated LSUN Classroom images from the same three GANs over 8 seeds when trained for 100 epochs with a learning rate of $2 \times 10^{-4}$. These samples show that the vanilla $(1,1)$-GAN training fails for most of seeds while the other two GANs perform fairly well across all seeds, thus exhibiting robustness to random weight initializations.}
\label{fig:lsun_epochs_fids_images}
\end{figure*}

\end{document}